\documentclass[final]{IEEEtran}

\usepackage[numbers, square, comma, sort&compress]{natbib}

\usepackage{graphicx}
\graphicspath{{./figures/}}

\usepackage[labelformat=simple]{subcaption}

\DeclareCaptionLabelSeparator{periodspace}{.\quad}
\usepackage[cmex10]{amsmath}
\usepackage{amssymb, mathrsfs, amsfonts}
\usepackage{mathtools}
\usepackage{icomma}

\usepackage{amsthm}

\theoremstyle{remark}
\newtheorem*{note}{Note}

\theoremstyle{plain}
\newtheorem{definition}{Definition}
\newtheorem{proposition}{Proposition}
\newtheorem{corollary}{Corollary}
\newtheorem{theorem}{Theorem}

\usepackage[colorlinks]{hyperref}

\usepackage[capitalise]{cleveref}
\crefname{section}{Section}{sections}
\crefname{figure}{Fig.}{figs.}
\crefname{equation}{Equation}{equations}
\crefrangelabelformat{equation}{(#3#1#4--#5#2#6)}
\crefname{problem}{Problem}{problems}
\creflabelformat{problem}{(#2#1#3)}
\crefname{algorithm}{Algorithm}{algorithms}
\creflabelformat{algorithm}{(#2#1#3)}
\crefname{assumption}{Assumption}{assumptions}
\creflabelformat{assumption}{(#2#1#3)}
\crefname{step}{Step}{steps}
\creflabelformat{step}{(#2#1#3)}
\crefname{algoline}{Line}{lines}

\DeclareMathOperator*{\argmin}{\arg\!\min}
\newcommand{\vect}[1]{{\mathbf{#1}}}
\newcommand{\mat}[1]{{\mathbf{#1}}}
\newcommand{\set}[1]{{\mathcal{#1}}}
\newcommand{\norm}[2]{\left\| #1 \right\|_{#2}}
\newcommand{\transpose}[1]{#1^\mathrm{T}}

\newcommand{\Real}{{\mathbb R}}
\newcommand{\Natural}{{\mathbb N}}

\usepackage{siunitx}

\usepackage[algo2e,ruled,vlined,linesnumbered]{algorithm2e}
\SetKwComment{Comment}{//\ }{}
\makeatletter
\let\oldnl\nl
\newcommand{\nonl}{\renewcommand{\nl}{\let\nl\oldnl}}
\makeatother

\makeatletter
\newcommand{\removelatexerror}{\let\@latex@error\@gobble}
\makeatother

\usepackage{booktabs}
\usepackage{multirow}
\usepackage{tabu}
\usepackage{rotating}
\usepackage{threeparttable}

\usepackage{dcolumn}
\newcolumntype{d}[1]{D{.}{.}{#1}}
\newcolumntype{.}{D{.}{.}{-1}}
\makeatletter
\newcolumntype{B}[1]{>{\boldmath\DC@{.}{.}{#1}}c<{\DC@end}}
\makeatother

\usepackage{paralist}

\usepackage[usenames,dvipsnames,svgnames,table]{xcolor}
\definecolor{DarkRed}{rgb}{0.75, 0, 0}

\usepackage{tikz}
\usetikzlibrary{arrows}

\title{COMPRESSED NONNEGATIVE MATRIX FACTORIZATION IS FAST AND ACCURATE}

\author{Mariano~Tepper, Guillermo~Sapiro%
\thanks{This work was partially supported by NSF, ONR, NGA, ARO, and NSSEFF.}%
\thanks{The authors are with the Department of Electrical and Computer Engineering, Duke University, NC 27708 USA (e-mail: \{mariano.tepper,guillermo.sapiro\}@duke.edu)}}

\begin{document}

\maketitle

\begin{abstract}
Nonnegative matrix factorization (NMF) has an established reputation as a useful data analysis technique in numerous applications. However, its usage in practical situations is undergoing challenges in recent years. The fundamental factor to this is the increasingly growing size of the datasets available and needed in the information sciences.
To address this, in this work we propose to use structured random compression, that is, random projections that exploit the data structure, for two NMF variants: classical and separable. In separable NMF (SNMF) the left factors are a subset of the columns of the input matrix. We present suitable formulations for each problem, dealing with different representative algorithms within each one. We show that the resulting compressed techniques are faster than their uncompressed variants, vastly reduce memory demands, and do not encompass any significant deterioration in performance.
The proposed structured random projections for SNMF allow to deal with arbitrarily shaped large matrices, beyond the standard limit of tall-and-skinny matrices, granting access to very efficient computations in this general setting.
We accompany the algorithmic presentation with theoretical foundations and numerous and diverse examples, showing the suitability of the proposed approaches.
\end{abstract}


\section{Introduction}

The number and diversity of the fields that make use of data analysis is rapidly increasing, from economics and marketing to medicine and neuroscience. In all of them, data is being collected at an astounding speed: databases are now measured in gigabytes and terabytes, including trillions of point-of-sale transactions, worldwide social networks, and gigapixel images. Organizations need to rapidly turn these terabytes of raw data into significant insights for their users to guide their research, marketing, investment, and/or management strategies.

Matrix factorization is a fundamental data analysis technique. Whereas its usefulness as a theoretical tool is beyond doubt now, its usage in practical situations has undergone a few challenges in recent years. Among other factors contributing to this are new developments in computer hardware architecture and new applications in the
information sciences.

Perhaps the key aspect is that the matrices to analyze are becoming astonishingly big. Classical algorithms are not designed to cope with the amount of information present in these large-scale problems. We may even hypothesize that, if proper tools for these problems were widely available for commercial computer power, such rich datasets would be created at an increasing speed.

In this big data scenario, data communication is one of the main performance bottlenecks for numerical algorithms (here, we mean communication in a broad sense, including for example, network transfers and secondary memory access).
Since the data cannot be easily stored in main memory, performing fewer passes over the original data, even at the cost of more floating-point operations, may result in substantially faster techniques.

Lastly, the architecture of computing units is evolving towards massive parallelism (consider, for example, general purpose GPUs and MapReduce models~\cite{Dean2008}). Numerical algorithms should adapt to these environments and exploit their benefits for boosting their performance.

In recent years, Nonnegative Matrix Factorization (NMF)~\cite{Paatero1994} has been frequently used since it provides a good way for modeling many real-life applications (e.g., recommender systems~\cite{Melville2010} and audio processing~\cite{Fevotte2009}). NMF seeks to represent a nonnegative matrix (i.e., a matrix with nonnegative entries) as the product of two nonnegative matrices. One of the reasons for the method's popularity is that the use of non-subtractive linear combinations renders the factorization, in many cases, easily interpretable. The goal of this work is to develop algorithms, based on structured random projections, for computing NMF for big data matrices.

\subsection{Two flavors of nonnegative matrix factorization}

Given an $m \times n$ nonnegative matrix $\mat{A}$, NMF is formally defined as
\begin{equation}
    \min_{\substack{ \mat{X} \in \Real^{m \times r}, \mat{Y} \in \Real^{r \times n}}} \norm{\mat{A} - \mat{X} \mat{Y}}{F}^2
    \quad \text{s.t.} \quad
    \quad \mat{X}, \mat{Y} \geq 0 ,
    \label[problem]{eq:nmf}
\end{equation}
where $r$ is a parameter that controls the size of factors $\mat{X}$ and $\mat{Y}$ and, hence, the factorization's accuracy.
For simplicity, we use $\mat{B} \geq 0$ to denote a matrix $\mat{B}$ with nonnegative entries.

Despite its appealing advantages, NMF does present some theoretical and practical challenges. In the general case, NMF is known to be NP-Hard~\cite{Vavasis2010} and highly ill-posed~\cite[and references therein]{Gillis2012}.
However, there are matrices that exhibit a particular structure such that NMF can be solved efficiently (i.e., in polynomial time)~\cite{Arora2012}.

\begin{definition}
	A nonnegative matrix $\mat{A}$ is $r$-separable if there exists an index set $\set{K}$ of cardinality $r$ over the columns of $\mat{A}$ and a nonnegative matrix $\mat{Y} \in \Real^{r \times n}$, such that
	\begin{equation}
	\mat{A} = (\mat{A})_{:\set{K}} \mat{Y} ,
	\end{equation}
	where $(\mat{A})_{:\set{K}}$ represents the matrix obtained by horizontally stacking the columns of $\mat{A}$ indexed by $\set{K}$.
	Consequently, a nonnegative matrix $\mat{A}$ is near $r$-separable if it can be represented as
	\begin{equation}
	\mat{A} = (\mat{A})_{:\set{K}} \mat{Y} + \mat{N} ,
	\end{equation}
	where $\mat{N}$ is a noise matrix.
	
	\label[definition]{separableMatrix}
\end{definition}
	
When $\mat{A}$ presents this type of special structure, the NMF problem (now denoted as separable NMF, SNMF) can be simply modeled as
\begin{equation}
	\min_{\substack{\set{K} \subset \{ 1, \dots, n\} \\ \mat{Y} \in \Real^{r \times n}}} \norm{\mat{A} - (\mat{A})_{:\set{K}} \mat{Y}}{F}^2
	\quad
	\text{s.t.}
	\quad
	\begin{gathered}
		\# \set{K} = r,\\
		\mat{Y} \geq 0 ,
	\end{gathered}
	\label[problem]{eq:snmf}
\end{equation}
where the choice of the Frobenius norm corresponds to a Gaussian noise matrix $\mat{N}$.
Having a more constrained structure for the left factor (i.e., $\mat{X} = (\mat{A})_{:\set{K}}$) makes the problem significantly easier to solve, improving the stability and the speed of the involved algorithms.

\subsection{Structured random projections}

In recent years, we have seen an increase in the popularity of randomized algorithms for computing partial matrix decompositions. These partial decompositions assume that most of the action of a matrix occurs in a subspace. The key observation here is that such a subspace can be identified through random sampling. After projecting the input matrix into this subspace (i.e., compressing it), the desired low-rank factorization can be obtained by manipulating deterministically this compressed matrix. In many cases, this approach outperforms its classical competitors in terms of accuracy, speed, and robustness. See~\cite{Halko2009} for a thorough review of these techniques.

\subsection{Contributions and organization}

We propose an algorithmic solution for computing structured random projections of extremely large matrices (i.e., matrices so large that even after compression they do not fit in main memory). This is useful as a general tool for computing many different matrix decompositions (beyond NMF, which is the particular focus of this work). Our approach leads to the implementation of compression algorithms that perform out-of-core computations (i.e., loading information in main memory only as needed).

We propose to use structured random projections for NMF and show that, in practice, their use implies a substantial increase in speed. This performance boost does not come at the price of significant errors with respect to the uncompressed solutions. We show this for representative algorithms of different NMF approaches, namely, multiplicative updates~\cite{Lee2000}, active set method for nonnegative least squares~\cite{Kim2008}, and ADMM~\cite{Xu2012}.

We present a general SNMF algorithm based on structured random projections, reaching to similar conclusions as in the general NMF case. While there are in the literature very efficient SNMF algorithms for tall-and-skinny matrices~\cite{Benson2014}, we show that, when the rank of the desired decomposition is lower than the number of columns of the input matrix, the proposed algorithm is substantially faster than its competitors. Interestingly, the use of structured random projections allows to compute SNMF for arbitrarily large matrices, eliminating the tall-and-skinny requirement while preserving efficiency. Our code is available at \url{http://www.marianotepper.com.ar/research/cnmf}.

The remainder of the paper is organized as follows. In \cref{sec:randomProjections} we provide an overview of random projection methods for matrix factorization and provide some theoretical results relevant to this work.
In \cref{sec:nmf,sec:snmf} we propose a set of techniques for using random projections for NMF and SNMF, respectively. Extensive experimental results on diverse problems are presented in \cref{sec:results}, studying the performance of the proposed techniques on both medium and large-scale problems. Finally, we provide some concluding remarks in \cref{sec:conclusions}.

\section{On randomization and matrix decompositions}
\label{sec:randomProjections}

In this section we begin by describing the random projection algorithm used throughout this work. We also present theory that provides some guarantees for the use of random projections in matrix decomposition (in this work we use interchangeably projection or compression). Finally, we discuss the performance limits of the algorithm when dealing with big data and introduce a way to overcome such limitations.

In \cref{eq:nmf,eq:snmf}, the rank of the desired matrix factorization is prespecified. In the following, we will thus assume that we are given a matrix $\mat{A}$, a target rank $r$, and an oversampling parameter $r_\textsc{ov}$ (its role will become clear next).

We define a Gaussian random matrix $\mat{\Omega}$ as a matrix whose entries are drawn independently from a standard Gaussian distribution, i.e., each entry $(\mat{\Omega})_{ij}$ is a realization of an independent and identically distributed random variable with distribution $\mathcal{N} (0, 1)$.

The overall approach to matrix factorization presented in~\cite{Halko2009} consists of the following three steps:
\begin{enumerate}
	\item Compute an approximate basis for the range of the input matrix $\mat{A}$: we construct a matrix $\mat{Q}$, with $r + r_\textsc{ov}$ orthonormal columns (i.e., $\transpose{\mat{Q}} \mat{Q} = \mat{I}$, where $\mat{I}$ is the $(r + r_\textsc{ov}) \times (r + r_\textsc{ov})$ identity matrix), for which
	\begin{equation}
		\norm{ \mat{A} - \mat{Q} \transpose{\mat{Q}} \mat{A} }{2} \approx \min_{\operatorname{rank}(\mat{Z}) \leq r} \norm{ \mat{A} - \mat{Z} }{2} = \sigma_{r+1} ,
		\label{eq:compression_approximation}
	\end{equation}
	where $\sigma_{j}$ denotes the $j$-th largest singular value of $\mat{A}$.
	In other words, $\mat{Q} \transpose{\mat{Q}} \mat{A}$ is a good rank-$r$ approximation of $\mat{A}$.
	\label[step]{compression_computation}
	\item Compute a factorization of $\transpose{\mat{Q}} \mat{A}$.
	\label[step]{compression_factorization}
	\item Multiply the leftmost factor of the decomposition by $\mat{Q}$, all other factors remain unchanged.
\end{enumerate}
Throughout this paper, we will use the algorithm in \cref{algo:compression} for performing \cref{compression_computation}. For more details about this algorithm, we refer the reader to~\cite{Halko2009}.
Since the algorithm exploits the structure in $\mat{A}$, trying to find a subspace were the majority of its action happens, we will refer to this technique as \emph{structured random compression}.

In the following, we present some results from~\cite{Halko2009} that demonstrate the nice theoretical characteristics of the compression matrix $\mat{Q}$, obtained with the algorithm in \cref{algo:compression}.
Let $\mathbb{E}$ denote the expectation with respect to the random matrix.

\begin{figure}
	\removelatexerror
	\begin{algorithm2e}[H]	
		\SetInd{0.5em}{0.5em}
		
		\begin{small}
			
			\SetKwInOut{Input}{input}\SetKwInOut{Output}{output}
			
			\Input{a matrix $\mat{A} \in \Real^{m \times n}$, a target rank $r  \in \Natural^+$, an oversampling parameter $r_\textsc{ov} \in \Natural^+$ ($r + r_\textsc{ov} \leq m$), an exponent $w \in \Natural$.}
			\Output{a compression matrix $\mat{Q} \in \Real^{m \times (r + r_\textsc{ov})}$ for $\mat{A}$.}
			
			Draw a Gaussian random matrix
			$\mat{\Omega}_L \in \Real^{n \times (r + r_\textsc{ov})}$\;

			Form the matrix product\\
			\nonl \qquad $\mat{B} = \left( \mat{A} \transpose{\mat{A}} \right)^w \mat{A} \mat{\Omega}$\;
			\nllabel{algo:compression_powerit}
			
			Let $\mat{Q}$ be an orthogonal basis for $\mat{B}$, obtained using the QR decomposition\; 
			\nllabel{algo:compression_qr}
			
		\end{small}
		
	\end{algorithm2e}
	
	\caption{Structured random compression algorithm.}
	\label{algo:compression}
\end{figure}

\begin{theorem}[\cite{Halko2009}]
	\label[theorem]{theo:compression_probaBounds}
	Given a matrix $\mat{A} \in \Real^{m \times n}$, a target rank $r  \in \Natural^+$, and an oversampling parameter $r_\textsc{ov} \in \Natural^+$ ($r + r_\textsc{ov} \leq m$), execute the algorithm in \cref{algo:compression} with $w=0$ (no power iterations). We obtain a matrix $\mat{Q} \in \Real^{m \times (r + r_\textsc{ov})}$. Let $\mat{P} = \mat{Q} \transpose{\mat{Q}}$. Then,
	\begin{equation}
		\mathbb{E} \norm{ \mat{A} - \mat{P} \mat{A}}{F} \leq
		\left( 1 + \tfrac{r}{r_\textsc{ov} - 1} \right)^{1/2}
		\left( \sum_{j>r} \sigma_{j}^2 \right)^{1/2} ,
	\end{equation}
	\begin{equation}
		\mathbb{E} \norm{ \mat{A} - \mat{P} \mat{A}}{2} \leq
		\left[ 1 + \tfrac{4 \sqrt{r+r_\textsc{ov}}}{r_\textsc{ov} - 1} \sqrt{\min\{m, n\}} \right] \ \sigma_{r+1} .
	\end{equation}
\end{theorem}

Note that $\left( \sum_{j>r} \sigma_{j}^2 \right)^{1/2}$ and $\sigma_{r+1}$ are the smallest possible errors, see \cref{eq:compression_approximation}.

\begin{theorem}[\cite{Halko2009}]
	Frame the same hypotheses of \cref{theo:compression_probaBounds}.
	Assume $r_\textsc{ov} \geq 4$. Then, $\forall u, t \geq 1$,
	\begin{align}
		\norm{ \mat{A} - \mat{P} \mat{A}}{F} \leq &
		\left( 1 + t \sqrt{12 r / r_\textsc{ov}} \right)^{1/2}
		\left( \sum_{j>r} \sigma_{j}^2 \right)^{1/2}
		+ \nonumber\\
		& ut \tfrac{e \sqrt{r + r_\textsc{ov}}}{r_\textsc{ov} + 1} \sigma_{r+1} ,
	\end{align}
	with failure probability at most $5t^{-r_\textsc{ov}} + 2e^{-u^2 / 2}$.
	We also have
	\begin{align}
		\norm{ \mat{A} - \mat{P} \mat{A}}{2} \leq\ &
		3 \sqrt{r + r_\textsc{ov}}
		\left( \sum_{j>r} \sigma_{j}^2 \right)^{1/2}
		+ \nonumber\\
		& \left( 1 + t \sqrt{8 (r + r_\textsc{ov}) r_\textsc{ov} \log r_\textsc{ov}} \right) \sigma_{r+1} ,
	\end{align}	
	with failure probability at most $6 (r_\textsc{ov})^{-r_\textsc{ov}}$.
\end{theorem}

Beyond proving that the achieved error is very close to the optimal error, the above theorems provide a theoretical justification for the oversampling parameter $r_\textsc{ov}$. It grants more freedom in the choice of $\mat{Q}$, crucial in the effectiveness of \cref{compression_factorization}~\cite{Halko2009}. This freedom allows  the probability of failure to decrease exponentially fast as $r_\textsc{ov}$ grows.

\begin{theorem}[\cite{Halko2009}]
	\label[theorem]{theo:compressionPower_probaBounds}
	Given a matrix $\mat{A} \in \Real^{m \times n}$, a target rank $r  \in \Natural^+$, an oversampling parameter $r_\textsc{ov} \in \Natural^+$ ($r + r_\textsc{ov} \leq m$), and an exponent $w \in \Natural$, execute the algorithm in \cref{algo:compression}. We obtain a matrix $\mat{Q} \in \Real^{m \times (r + r_\textsc{ov})}$.  
	Let $\mat{P} = \mat{Q} \transpose{\mat{Q}}$.
	Then,
	\begin{align}
		\mathbb{E} \norm{ \mat{A} - \mat{P} \mat{A}}{2} \leq c^{1/(2w+1)} \ \sigma_{r+1} ,
	\end{align}
	where
	\begin{equation}
		c = 1 + \sqrt{\tfrac{r}{(r_\textsc{ov} + 1)}}  + \tfrac{e \sqrt{r + r_\textsc{ov}}}{r_\textsc{ov}} \sqrt{\min \{ m, n \} - k} .
	\end{equation}
\end{theorem}
As we increase the exponent $w$, the power scheme drives the extra factor in the error to one exponentially fast.
As noted in~\cite{Halko2009}, finding an analogous bound for the Frobenius norm is still an open problem.

Throughout this work we use a Gaussian test matrix $\mat{\Omega}$. Other alternative test matrices can be used in its place, such as the subsampled randomized Hadamard and Fourier transforms~\cite{Halko2009,Tropp2010}. The product $\mat{A} \mat{\Omega}$ can be significantly faster when using a test matrix obtained with these transforms, giving an automatic speedup. From this perspective, all the experimental results in this paper present a \emph{worst case} scenario with respect to running times.

\begin{note}
	An alternative to structured random compression would be to just left-multiply $\mat{A}$ by a Gaussian random matrix $\mat{\Omega}$. Let us define the compression matrix $\mat{Q}_\mat{\Omega} \in \Real^{m \times s}$ as
	\begin{equation}
	\mat{Q}_\mat{\Omega} = s^{-1/2} \ \mat{\Omega} ,
	\label{eq:gaussianCompression}
	\end{equation}
	where $\mat{\Omega}$ is a Gaussian random matrix
	Then, instead of computing measures with the data matrix $\mat{A}$ on the $m$-dimensional space, the much smaller matrix $\transpose{\mat{Q}_\mat{\Omega}} \mat{A}$ can be used to compute approximations in the $s$-dimensional space. It is well studied that Gaussian projection preserves the $\ell_2$ norm \cite[e.g.,][and references therein]{Baraniuk2008}. However, our extensive experiments show that structured random compression achieves better performance than Gaussian compression. Intuitively, Gaussian compression is a general data-agnostic tool, whereas structured compression uses information from the matrix (an analogous of training). Theoretical research is needed to fully justify this performance gap.
\end{note}

\subsection{Big data algorithmic solutions}
By design, the product in line~\ref{algo:compression_powerit} of the algorithm in \cref{algo:compression} forms a tall and skinny matrix $\mat{B} \in \Real^{m \times (r + r_\textsc{ov})}$, where $m \gg (r + r_\textsc{ov})$.
We have thus successfully reduced the number of columns in $\mat{B}$ from $n$ to $r + r_\textsc{ov}$. While matrix $\mat{A}$ may not fit in main memory, we can still perform the necessary computations using $\mat{B}$ without significant loss of precision.

An interesting question arises when working with large matrices: what happens if the number of rows $m$ is so large that even $\mat{B}$ does not fit in main memory? Assuming that we need to store $\mat{B}$ in secondary memory (i.e., the hard drive), how do we compute its QR decomposition (line~\ref{algo:compression_qr} of the algorithm in \cref{algo:compression})?

A suitable and efficient algorithm to address the latter question is the direct TSQR (tall-and-skinny QR)~\cite{Benson2014}. For completeness, we give its outline in \cref{sec:tsqr}. The highlight of TSQR is that it is designed for being parallelizable while minimizing the dependencies between parallel computations (i.e., communication costs). Thus, it adheres perfectly to the main mantra of this work.

An interesting byproduct of using TSQR is that there is no need to form the entire matrix $\mat{B}$ in main memory. See \cref{sec:tsqr} for further details. This allows to implement an \emph{out-of-core} version of the compression algorithm, that is, where the involved matrices do not reside in main memory.

Let us note that the use of TSQR for computing random compression is introduced in this paper for the first time, providing a \emph{true} scalable solution for computing many types of matrix decompositions (i.e., beyond NMF) when both the number of rows and columns of the input matrix are large.

\subsection{Matrix decompositions with alternative norms}

The algorithm in \cref{fig:compression} works under the Frobenius and nuclear norms, as detailed in the theorems presented above. These two cases already cover a significant range of matrix decompositions that are commonly used in practice.

However, other norms are becoming increasingly popular in recent years. For example, NMF is widely used in audio processing with the Itakura-Saito distance instead of the Frobenius norm in \cref{eq:nmf}. The entrywise $\ell_1$ norm is also very popular when the input matrix $\mat{A}$ is contaminated with impulsive noise. In these cases, proper structured random projection algorithms need to be used, adapted to the right type of measure for the application at hand.

In particular, we are currently investigating the use of the framework here developed for NMF under an $\ell_1$ norm. In such a case, the fast Cauchy transform appears as a suitable alternative for the task~\cite{Clarkson2013}.

\section{Randomly compressed NMF}
\label{sec:nmf}

The goal of this section is to efficiently solve \cref{eq:nmf} for large input matrices. We do not aim at developing a new NMF algorithm, but rather to illustrate how structured random projections can be used to enhance the speed of existing algorithms and make them usable for big data. As detailed in \cref{sec:results}, this speedup does not come at the price of significantly higher reconstruction errors.

Most NMF algorithms work by iterating the following two steps:
\begin{subequations}
	\begin{itemize}
		\item Find $\mat{X}_{k+1} \in \Real^{m \times r}$, $\mat{X}_{k+1} \geq 0$, such that
		\begin{align}
		\norm{\mat{A} - \mat{X}_{k+1} \mat{Y}_k}{F}^2 &\leq \norm{\mat{A} - \mat{X}_{k} \mat{Y}_k}{F}^2 .
		\intertext{\item Find $\mat{Y}_{k+1} \in \Real^{r \times n}$, $\mat{Y}_{k+1} \geq 0$, such that}
		\norm{\mat{A} - \mat{X}_{k+1} \mat{Y}_{k+1}}{F}^2 &\leq \norm{\mat{A} - \mat{X}_{k+1} \mat{Y}_k}{F}^2 .
		\end{align}
	\end{itemize}
	\label[algorithm]{eq:alternate_nmf}
\end{subequations}
This general formulation encompasses different particular algorithms such as multiplicative updates~\cite{Lee2000} and several variants of alternating nonnegative least squares~\cite{Chu2004,Lin2007,Kim2008}.
The latter consists of a particular case of \cref{eq:alternate_nmf} in which its right-hand sides are minimized to the end. We thus obtain the following algorithm:
\begin{subequations}
	\begin{align}
	\mat{X}_{k+1} &= \argmin_{\substack{ \mat{X} \in \Real^{m \times r}}} \norm{\mat{A} - \mat{X} \mat{Y}_k}{F}^2
	& \text{s.t.} &
	\quad \mat{X} \geq 0 ,\\
	\mat{Y}_{k+1} &= \argmin_{\substack{ \mat{Y} \in \Real^{r \times n}}} \norm{\mat{A} - \mat{X}_{k+1} \mat{Y}}{F}^2
	& \text{s.t.} &
	\quad \mat{Y} \geq 0 .
	\end{align}
	\label[algorithm]{eq:alternateLS_nmf}
\end{subequations}

Let us assume that we apply the algorithm in \cref{algo:compression} to $\mat{A}$ and $\transpose{\mat{A}}$ and obtain two matrices $\mat{L} \in \Real^{m \times (r + r_\textsc{ov})}, \mat{R} \in \Real^{(r + r_\textsc{ov}) \times n}$, respectively. By construction, $\mat{L}$ and $\mat{R}$ have orthonormal columns and rows, respectively.
Also let $\check{\mat{A}} = \mat{A} \transpose{\mat{R}}$, $\hat{\mat{A}} = \transpose{\mat{L}} \mat{A}$.

Using matrices $\mat{L}$ and $\mat{R}$, we propose to approximate \cref{eq:alternate_nmf} with the iterations
\begin{subequations}
	\begin{itemize}
		\item Find $\mat{X}_{k+1} \in \Real^{m \times r}$, $\mat{X}_{k+1} \geq 0$, such that
		\par\hspace*{-\leftmargin}\parbox{\columnwidth}{%
		\begin{equation}
			\norm{\check{\mat{A}} - \mat{X}_{k+1} \mat{Y}_k \transpose{\mat{R}}}{F}^2 \leq \norm{\check{\mat{A}} - \mat{X}_{k} \mat{Y}_k \transpose{\mat{R}}}{F}^2 .
			\label{eq:alternate_nmf_compressed_X}
		\end{equation}}
		\item Find $\mat{Y}_{k+1} \in \Real^{r \times n}$, $\mat{Y}_{k+1} \geq 0$, such that
		\par\hspace*{-\leftmargin}\parbox{\columnwidth}{%
		\begin{equation}
			\norm{\hat{\mat{A}} - \transpose{\mat{L}} \mat{X}_{k+1} \mat{Y}_{k+1}}{F}^2 \leq \norm{\hat{\mat{A}} - \transpose{\mat{L}} \mat{X}_{k+1} \mat{Y}_k}{F}^2 .
			\label{eq:alternate_nmf_compressed_Y}
		\end{equation}}
	\end{itemize}
	\label[algorithm]{eq:alternate_nmf_compressed}
\end{subequations}

Equivalently, using $\mat{L}$ and $\mat{R}$, we propose to approximate \cref{eq:alternateLS_nmf} with the iterations
\begin{subequations}
	\begin{align}
	\mat{X}_{k+1} &= \argmin_{\substack{ \mat{X} \in \Real^{m \times r}}} \norm{\check{\mat{A}} - \mat{X} \mat{Y}_k \transpose{\mat{R}} }{F}^2
	& \text{s.t.} &
	\quad \mat{X} \geq 0 , \label{eq:alternateLS_nmf_compressed_X} \\
	\mat{Y}_{k+1} &= \argmin_{\substack{ \mat{Y} \in \Real^{r \times n}}} \norm{ \hat{\mat{A}} - \transpose{\mat{L}} \mat{X}_{k+1} \mat{Y}}{F}^2
	& \text{s.t.} &
	\quad \mat{Y} \geq 0 . \label{eq:alternateLS_nmf_compressed_Y}
	\end{align}
	\label[algorithm]{eq:alternateLS_nmf_compressed}
\end{subequations}

The algorithm in \cref{algo:alternate_nmf_compressed} contains an overview of the proposed NMF algorithm using structured random compression.
For our experiments regarding the techniques described in \cref{sec:nmf}, as representative examples of \cref{eq:alternate_nmf} and \cref{eq:alternateLS_nmf}, we respectively use the active set method~\cite{Kim2008} and the multiplicative updates in~\cite[Eq. (8)]{Ding2010}.

\begin{figure}
	\removelatexerror
	\begin{algorithm2e}[H]	
		\SetInd{0.5em}{0.5em}
		
		\begin{small}
			
			\SetKwInOut{Input}{input}\SetKwInOut{Output}{output}
			
			\Input{a matrix $\mat{A} \in \Real^{m \times n}$, a target rank $r  \in \Natural^+$, an oversampling parameter $r_\textsc{ov} \in \Natural^+$ ($r + r_\textsc{ov} \leq \min \{ m, n\}$), an exponent $w \in \Natural$.}
			\Output{nonnegative matrices $\mat{X}_{k} \in \Real^{m \times r}, \mat{Y}_{k} \in \Real^{r \times n}$.}
			
			Compute compression matrices $\mat{L} \in \Real^{m \times (r + r_\textsc{ov})}$, $\mat{R} \in \Real^{(r + r_\textsc{ov}) \times n}$\;
			
			$k \gets 1$\;
			Initialize $\mat{Y}_{k}$\;
			
			\Repeat{convergence}{
				$\check{\mat{Y}}_k \gets \mat{Y}_k \transpose{\mat{R}}$\;
				
				Find $\mat{X}_{k+1} \in \Real^{m \times r}$, $\mat{X}_{k+1} \geq 0$, such that
				\nllabel{algo:alternate_nmf_compressed_X}
				\begin{equation*}
					\norm{\check{\mat{A}} - \mat{X}_{k+1} \check{\mat{Y}}_k}{F}^2 \leq \norm{\check{\mat{A}} - \mat{X}_{k} \check{\mat{Y}}_k}{F}^2 ;
				\end{equation*}
				

				$\hat{\mat{X}}_{k+1} \gets \transpose{\mat{L}} \mat{X}_{k+1}$\;
				Find $\mat{Y}_{k+1} \in \Real^{r \times n}$, $\mat{Y}_{k+1} \geq 0$, such that
				\nllabel{algo:alternate_nmf_compressed_Y}
				\begin{equation*}
					\norm{\hat{\mat{A}} - \hat{\mat{X}}_{k+1} \mat{Y}_{k+1}}{F}^2 \leq \norm{\hat{\mat{A}} - \hat{\mat{X}}_{k+1} \mat{Y}_k}{F}^2 ;
				\end{equation*}
				
				\tcp{The optimizations in lines~\ref{algo:alternate_nmf_compressed_X} and~\ref{algo:alternate_nmf_compressed_Y} can be performed using \emph{any} variant of multiplicative updates or \emph{any} nonnegative least squares method.}
				
				$k \gets k + 1$\;
				
			}
		\end{small}
		
	\end{algorithm2e}
	
	\caption{NMF using structured random compression.}
	\label{algo:alternate_nmf_compressed}
\end{figure}

We achieve a significant size reduction of the matrices in \cref{eq:alternate_nmf_compressed,eq:alternateLS_nmf_compressed}. For each of these algorithms, we reduced the number of columns from $n$ to $r + r_\textsc{ov}$ in \cref{eq:alternate_nmf_compressed_X,eq:alternateLS_nmf_compressed_X} and the number of rows from $m$ to $r + r_\textsc{ov}$ in \cref{eq:alternate_nmf_compressed_Y,eq:alternateLS_nmf_compressed_Y}. This makes the system much faster to solve, but more importantly in our context, it greatly reduces the cost of data communication in parallel frameworks. For example, after compression, large matrices might fit in GPU memory.

Alternatively, \cref{eq:nmf} can be equivalently re-formulated as
\begin{equation}
    \min_{\substack{ \mat{X}, \mat{U} \in \Real^{m \times r} \\ \mat{Y}, \mat{V} \in \Real^{r \times n}}} \norm{\mat{A} - \mat{X} \mat{Y}}{F}^2
    \quad \text{s.t.} \quad
    \begin{gathered}
    \mat{U} = \mat{X},\ \mat{V} = \mat{Y}, \\
    \quad \mat{U}, \mat{V} \geq 0 .
    \end{gathered}
    \label[problem]{eq:nmfEquiv}
\end{equation}

Again, using the matrices $\mat{L}$ and $\mat{R}$ defined above, we propose to approximate \cref{eq:nmfEquiv} with
\begin{equation}
    \min_{\substack{ \mat{X}, \mat{U} \in \Real^{m \times r} \\ \mat{Y}, \mat{V} \in \Real^{r \times n}}} \norm{ \mat{L} \transpose{\mat{L}} \left( \mat{A} - \mat{X} \mat{Y} \right) \transpose{\mat{R}} \mat{R} }{F}^2
    \ \ \text{s.t.}
    \
    \begin{gathered}
    \mat{U} = \mat{X}, \\
    \mat{V} = \mat{Y}, \\
    \mat{U}, \mat{V} \geq 0 .
    \end{gathered}
    \label[problem]{eq:nmfEquiv_extended}
\end{equation}
Let $\displaystyle \widetilde{\mat{A}} = \transpose{\mat{L}} \mat{A} \transpose{\mat{R}}$,  $\displaystyle \widetilde{\mat{X}} = \transpose{\mat{L}} \mat{X}$, and $\displaystyle \widetilde{\mat{Y}} = \mat{Y} \transpose{\mat{R}}$. We propose to further approximate \cref{eq:nmfEquiv} with
\begin{equation}
    \min_{\substack{ \mat{U} \in \Real^{m \times r}, \mat{V} \in \Real^{r \times n} \\ \widetilde{\mat{X}} \in \Real^{(r + r_\textsc{ov}) \times r}\\ \widetilde{\mat{Y}} \in \Real^{r \times (r + r_\textsc{ov})}} } \norm{ \widetilde{\mat{A}} - \widetilde{\mat{X}} \widetilde{\mat{Y}} }{F}^2
    \quad \text{s.t.}
    \quad
    \begin{gathered}
    \mat{U} = \mat{L} \widetilde{\mat{X}}, \\
    \mat{V} = \widetilde{\mat{Y}} \mat{R}, \\
    \mat{U}, \mat{V} \geq 0 .
    \end{gathered}
    \label[problem]{eq:nmfEquiv_extended_compressed}
\end{equation}
The alternating direction method of multipliers (ADMM) can be used for solving \cref{eq:nmfEquiv}~\cite{Xu2012}. Thus, a similar technique can solve \cref{eq:nmfEquiv_extended_compressed}. The details of the proposed algorithm are presented in \cref{sec:nmf_admm_compressed}.

The level of compression in \cref{eq:nmfEquiv_extended_compressed} is significantly higher than in \cref{eq:alternate_nmf_compressed,eq:alternateLS_nmf_compressed}. The latter formulations only employ (alternated) single-sided compression, whereas the former uses a (simultaneous) double-sided compression. One may be inclined to think that such an aggressive compression might lead to greater errors; however, in practice, this is not the case. Studying this behavior from a theoretical standpoint might shed light into this interesting characteristic.

\subsection{Limits of NMF for big data}
When matrix $\mat{A}$ gets sufficiently large, solving \cref{eq:nmf} becomes challenging. The compression techniques here presented significantly alleviate the problem for in-core computations and are easily extensible for out-of-core computations. For example, each iteration of the multiplicative updates algorithm can be implemented on a MapReduce framework~\cite{Liu2010}; its structured compressed version can be easily adapted in this framework, greatly reducing communication costs thanks to the use of smaller matrices. Implementing our compressed ADMM algorithm on a MapReduce framework is just as straightforward.

However, when dealing with large volumes of data, the practical problem actually resides in the iterative nature of the algorithms. As an example, consider that the execution time of a single iteration of the multiplicative algorithm on a MapReduce framework is measured in hours for \emph{sparse} matrices with millions of columns and rows~\cite{Liu2010,Liao2014}. As expected, the issue is hugely exacerbated for dense matrices.

\section{Randomly compressed separable NMF}
\label{sec:snmf}

Following \Cref{separableMatrix}, let us now assume that matrix $\mat{A}$ is (near) $r$-separable.
Most state-of-the-art techniques for computing SNMF, see \cref{eq:snmf}, are based on the following two-step approach:
\begin{enumerate}
	\item Extract $r$ columns of $\mat{A}$, indexed by $\set{K}$. The literature usually refers to them as extreme columns.
	\label[step]{snmf_extremeColumns}
	\item Solve
	\begin{equation}
	\mat{Y} = \argmin_{\mat{H} \in \Real^{r \times n}} \norm{\mat{A} - (\mat{A})_{:\set{K}} \, \mat{H}}{F}^2
	\quad
	\text{s.t.}
	\quad
	\mat{H} \geq 0 .
	\label[problem]{eq:snmf_right_factor}
	\end{equation}
	\label[step]{snmf_right_factor}
\end{enumerate}
The literature on SNMF has mainly focused on \cref{snmf_extremeColumns} of the above algorithm. There are several types of algorithms for performing this task~\cite{Araujo2001,Bittorf2012,Kumar2013,Gillis2014a}. As for \cref{snmf_right_factor}, \Cref{eq:snmf_right_factor} involves solving $n$ nonnegative least squares problems separately, i.e.,
\begin{equation}
	(\mat{Y})_{:i} = \argmin_{\vect{h} \in \Real^{m}} \norm{\mat{A}_{:i} - (\mat{A})_{:\set{K}} \, \vect{h} }{F}^2
	\quad
	\text{s.t.}
	\quad
	\vect{h} \geq 0 .
	\label[problem]{eq:snmf_right_factor_col}
\end{equation}
This makes \cref{snmf_right_factor} trivially parallelizable.

Let $\mat{Q} \in \Real^{m \times m}$ be an orthonormal basis for $\mat{A} \in \Real^{m \times n}$
\begin{align}
	\transpose{\mat{Q}} \mat{A} &=
	\begin{bmatrix}
		\mat{R} \\ 0
	\end{bmatrix} ,
	&
	\transpose{\mat{Q}} (\mat{A})_{:\set{K}} &=
	\begin{bmatrix}
		(\mat{R})_{:\set{K}} \\ 0
	\end{bmatrix} ,
\end{align}
where $\mat{R} \in \Real^{n \times n}$.
A key observation here is that the zero rows do not provide information for finding extreme columns of $\mat{A}$~\cite{Benson2014}.
We also trivially have that, for any orthonormal matrix $\mat{Q} \in \Real^{m \times m}$,
\begin{equation}
\norm{\transpose{\mat{Q}} \left(\mat{A} - \mat{X} \mat{Y} \right)}{F} \propto \norm{\mat{A} - \mat{X} \mat{Y}}{F} .
\end{equation}
Then,
\begin{subequations}
\begin{align}
	\mat{Y} &= \argmin_{\mat{H} \geq 0} \norm{\mat{A} - (\mat{A})_{:\set{K}} \mat{H}}{F}^2 \\
	&= \argmin_{\mat{H} \geq 0} \norm{\transpose{\mat{Q}} \left( \mat{A} - (\mat{A})_{:\set{K}} \mat{H} \right)}{F}^2 \\
	&= \argmin_{\mat{H} \geq 0} \norm{\mat{R} - (\mat{R})_{:\set{K}} \mat{H}}{F}^2 .
	\label[problem]{eq:snmf_right_factor_QR}
\end{align}
\end{subequations}
Notice that \cref{eq:snmf_right_factor_QR} has succeeded to reduce the problem size to $n \times n$ from the original $m \times n$ \cref{eq:snmf_right_factor}.
We then obtain the following three-step algorithm~\cite{Benson2014}:
\begin{enumerate}
	\item Compute $\mat{Q}$ using, e.g., a QR decomposition of $\mat{A}$.
	\label[step]{snmf_qr}
	\item Find $r$ extreme columns of $\mat{R} = \transpose{\mat{Q}} \mat{A}$, indexed by $\set{K}$.
	\item Solve
	\begin{equation}
		\mat{Y} = \argmin_{\mat{H} \in \Real^{r \times n}} \norm{\mat{R} - (\mat{R})_{:\set{K}} \, \mat{H}}{F}^2
		\quad
		\text{s.t.}
		\quad
		\mat{H} \geq 0 .
	\end{equation}
	\label[step]{snmf_H}
\end{enumerate}

As the main assumption in NMF and SNMF is that $\mat{A}$ has (or can be approximated by) a low-rank structure, by all practical means we expect that $r \ll \min(m, n)$; otherwise, it would not even make sense to try these type of decompositions.
We claim that little to no information is lost by replacing the full orthonormal basis with a rank-preserving basis that projects the data into a lower-dimensional space.

As the reader might be already suspecting, we propose to obtain such a basis via the use structured random projections.
This involves a small but conceptually important change in the above SNMF algorithm. Replace \cref{snmf_qr} by
\begin{enumerate}
	\item Compute a structured random compression matrix $\mat{Q}$ for $\mat{A}$.
	\label[step]{snmf_structured_random_projection}
\end{enumerate}
The proposed algorithm is depicted in \cref{fig:snmf_compressed}.
Let us now detail the main differences with the QR-based algorithm.

\begin{figure}
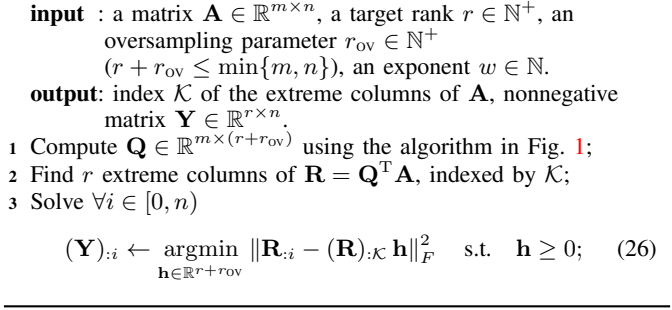

	\removelatexerror
	\begin{algorithm2e}[H]
	\begin{small}
		
		\SetKwInOut{Input}{input}\SetKwInOut{Output}{output}
		
		\Input{a matrix $\mat{A} \in \Real^{m \times n}$, a target rank $r  \in \Natural^+$, an oversampling parameter $r_\textsc{ov} \in \Natural^+$ ($r + r_\textsc{ov} \leq \min \{ m, n\}$), an exponent $w \in \Natural$.}
		\Output{index $\set{K}$ of the extreme columns of $\mat{A}$, nonnegative matrix $\mat{Y} \in \Real^{r \times n}$.}
		
		Compute $\mat{Q} \in \Real^{m \times (r + r_\textsc{ov})}$  using the algorithm in \cref{algo:compression}\;
		Find $r$ extreme columns of $\mat{R} = \transpose{\mat{Q}} \mat{A}$, indexed by $\set{K}$\;
		Solve $\forall i \in [0, n)$
		\begin{equation}
			(\mat{Y})_{:i} \gets \argmin_{\vect{h} \in \Real^{r + r_\textsc{ov}}} \norm{\mat{R}_{:i} - (\mat{R})_{:\set{K}} \, \vect{h} }{F}^2
			\quad
			\text{s.t.}
			\quad
			\vect{h} \geq 0 ;
		\end{equation}
		
	\end{small}
	\end{algorithm2e}
	
	\caption{SNMF using structured random compression.}
	\label{fig:snmf_compressed}
\end{figure}

First, let us note that $\mat{R} = \transpose{\mat{Q}} \mat{A}$ is now an $(r + r_\textsc{ov}) \times n$ matrix instead of an $n \times n$ matrix. This allows to process matrices that have many more columns, as storing $\mat{R}$ has become orders of magnitude easier/cheaper. Also note that each nonnegative least squares problem in \cref{eq:snmf_right_factor_col} has also become orders of magnitude smaller and thus faster to solve. Again, the huge decrease in communication costs for parallel implementations is even more important in our context than the gain in computational speed.

Second, the computation of the basis itself has become much faster. This is easy to understand when we compare the algorithm in \cref{algo:compression}, which only computes the QR decomposition of an $m \times (r + r_\textsc{ov})$ matrix, with the QR decomposition of the full $m \times n$ matrix. Of course, as the ratio $r/n$ decreases, the proposed algorithm becomes faster.

Let us assume for a moment that $n$ is sufficiently small such that we can use the TSQR algorithm directly on the input matrix $\mat{A}$, but not trivially small. As detailed in \cref{sec:tsqr}, the QR decomposition in \cref{eq:tsqr_centralizedQR} in the appendix is the only centralized step in TSQR; the amount of information that needs to be transmitted to carry this step is, again, orders of magnitude smaller when using structured random compressions.

\begin{note}
	The separable NMF model is similar to the model presented in~\cite{Esser2012} (and in~\cite{Elhamifar2012} without non-negativity constraints)
	\begin{equation}
		\min_{\mat{T} \in \Real^{n \times n}}
		\norm{\mat{A} - \mat{A} \mat{T}}{F}^2 + \lambda \norm{\mat{T}}{\text{row-}0}
		\quad
		\text{s.t.}
		\quad
		\mat{T} \geq 0 ,
		\label{eq:sparse_representatives}
	\end{equation}
	where $\norm{\mat{T}}{\text{row-}0}$ denotes the number of non-zero rows. The similarity resides in that selecting a subset of rows from $\mat{T}$ is equivalent to selecting a subset of columns from $\mat{A}$. This problem can be relaxed into a convex problem by replacing the $\ell_0$ pseudo-norm by a (possibly weighted) $\ell_1$ norm.  However, whichever optimization technique we choose for solving this problem, it will involve an iterative algorithm, where an $n \times n$ system is solved in every iteration. In~\cite{Esser2012}, the problem is shrank by clustering the columns of $\mat{A}$ and feeding a new matrix, only containing the cluster centers, into \cref{eq:sparse_representatives}. For this reasons, in our view, the SNMF model, as presented here, presents a cleaner and faster alternative to \cref{eq:sparse_representatives}.
\end{note}	

\section{Experimental results}
\label{sec:results}

We will now present numerous examples supporting the use of structured random projections for NMF and SNMF, both in terms of speed and accuracy.

Before jumping to these problems, in \cref{fig:compression} we show a simulation of the nice properties of the out-of-core compression algorithm presented in \cref{sec:randomProjections}. We performed our tests on $m \times n$ matrices with Gaussian entries, where different values for $m$ were tested, ranging from $10^3$ to $10^6$, and $n=500$ in all cases. This ensures that all matrices fit in main memory, allowing (1) to compress them with the in-core algorithm, and (2) to disregard disk access times, making the comparisons fair. The out-of-core algorithm for structured random compression is slower for matrices with approximately less than $2 \cdot 10^4$ rows; for these small matrices, the overhead of processing the matrix per blocks becomes evident (notice though that both computing times are well under 1 second). For larger matrices, the overhead's impact becomes less significant, and both algorithms exhibit the same overall performance (linear in $m$). In summary, we observe the expected behavior: the greater flexibility of the proposed out-of-core compression algorithm for processing large matrices does not cause performance to degrade with respect to the in-core one.

\begin{figure}[t]
	\centering
	\includegraphics[width=\columnwidth]{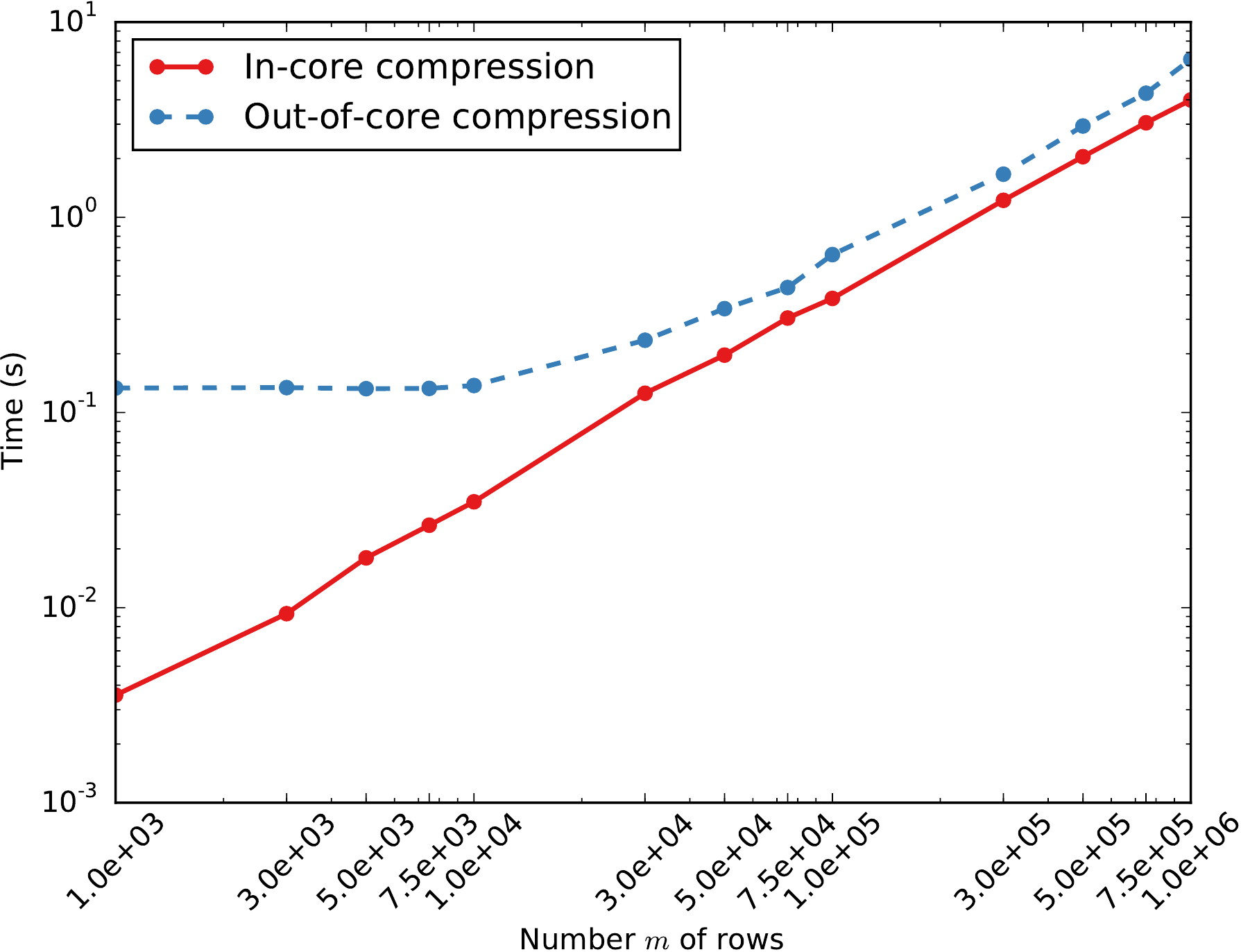}
	
	\caption{\textbf{Performance of out-of-core compression}. We tested different values for $m$, while fixing $n=500$. The out-of-core algorithm for structured random compression presented in \cref{sec:randomProjections} is slower for matrices with approximately less than $2 \cdot 10^4$ rows. For larger matrices, it exhibits the same complexity as the in-core one (linear in $m$). Out-of-core computations do not come at the price of a significantly slower compression algorithm, though permit to work with significantly larger matrices.}
	\label{fig:compression}
\end{figure}

\subsection{NMF}
\label{sec:results_nmf}

For our experiments regarding the techniques presented in \cref{sec:nmf}, as representative examples of \cref{eq:alternate_nmf} and \cref{eq:alternateLS_nmf}, we respectively use the active set method~\cite{Kim2008} and the multiplicative updates in~\cite[Eq. (8)]{Ding2010}. For these two algorithms, we compared with a vanilla version and a variant using Gaussian projection, as presented in~\cite{Wang2010} (also see \cref{sec:randomProjections}). We also implemented the ADMM algorithm in~\cite{Xu2012} and the proposed ADMM algorithm with structured random compression. All the methods were implemented in Matlab. In all tests, we set $w = 4$ and $r_\textsc{ov} = 10$ in the compression algorithm in \cref{algo:compression}; we further adjust the value of $r_\textsc{ov}$ so that $r + r_\textsc{ov} = \min(\max(20, r + r_\textsc{ov}), n)$.

We begin by showing in \cref{fig:syntheticNMF} simulations results of the different NMF variants on synthetic examples. The first interesting observation from these examples is that, although the computation of the compression matrix is more costly for structured than for Gaussian compression, this might not end up reflected in the overall computing time; this is because, in general, the NMF variant with Gaussian compression requires more iterations to converge. The second observation is that the NMF variants that use structured compression yield very similar relative reconstruction errors than their uncompressed counterparts (higher in one example, lower in three). For multiplicative updates and ADMM, the gain in speed of using structured compression is huge; for active set, the speedup is not as dramatic. Lastly, Gaussian compression seems to come at the cost of higher reconstruction errors.

\begin{figure*}

	\begin{subfigure}{\textwidth}
		\tabulinesep=2pt
	
		\centering
		\begin{footnotesize}
		\begin{tabu} to \textwidth { @{\hspace{4pt}} *{3}{X[c,m] @{\hspace{4pt}}} }
			
			\includegraphics[width=\linewidth]{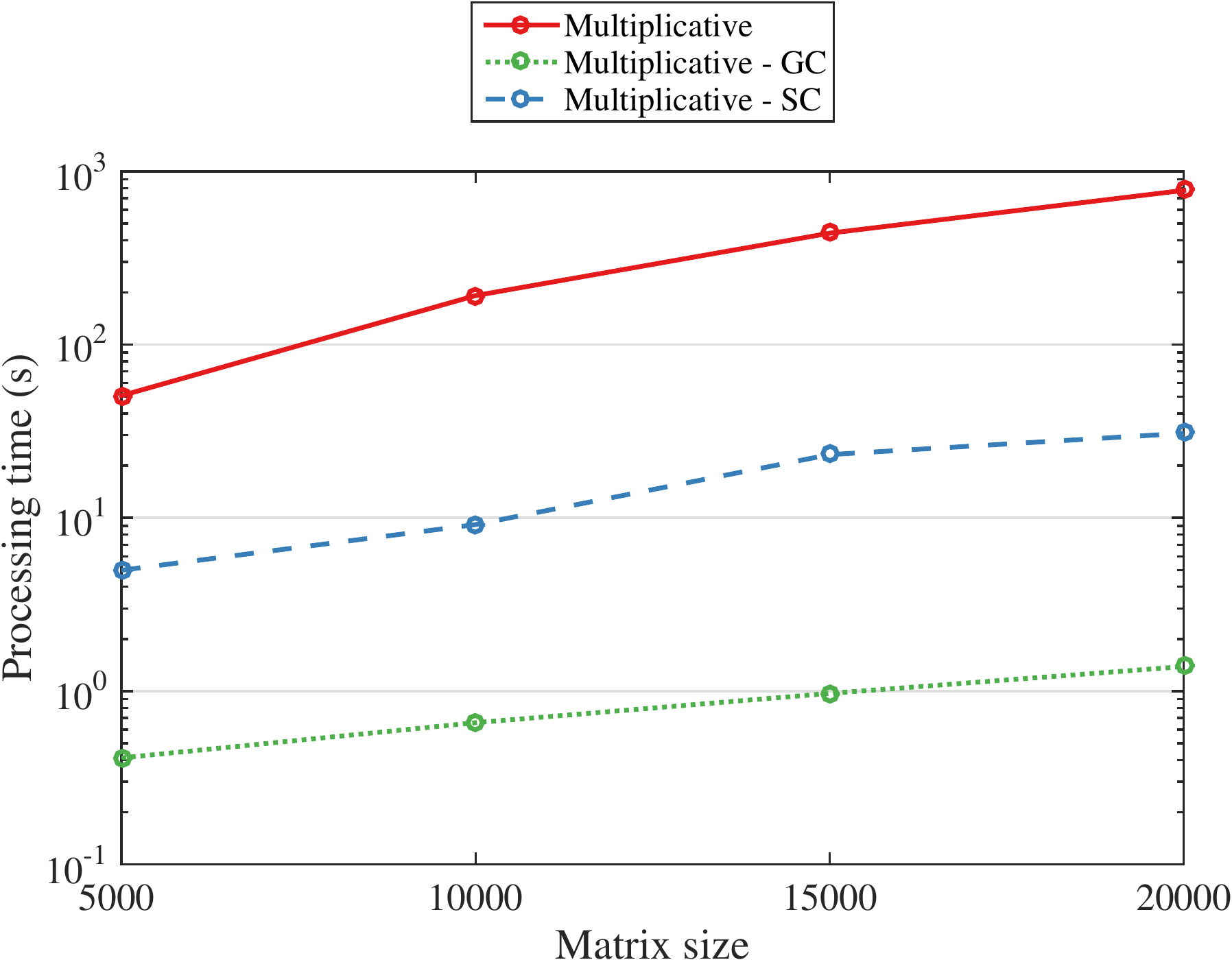} &
			\includegraphics[width=\linewidth]{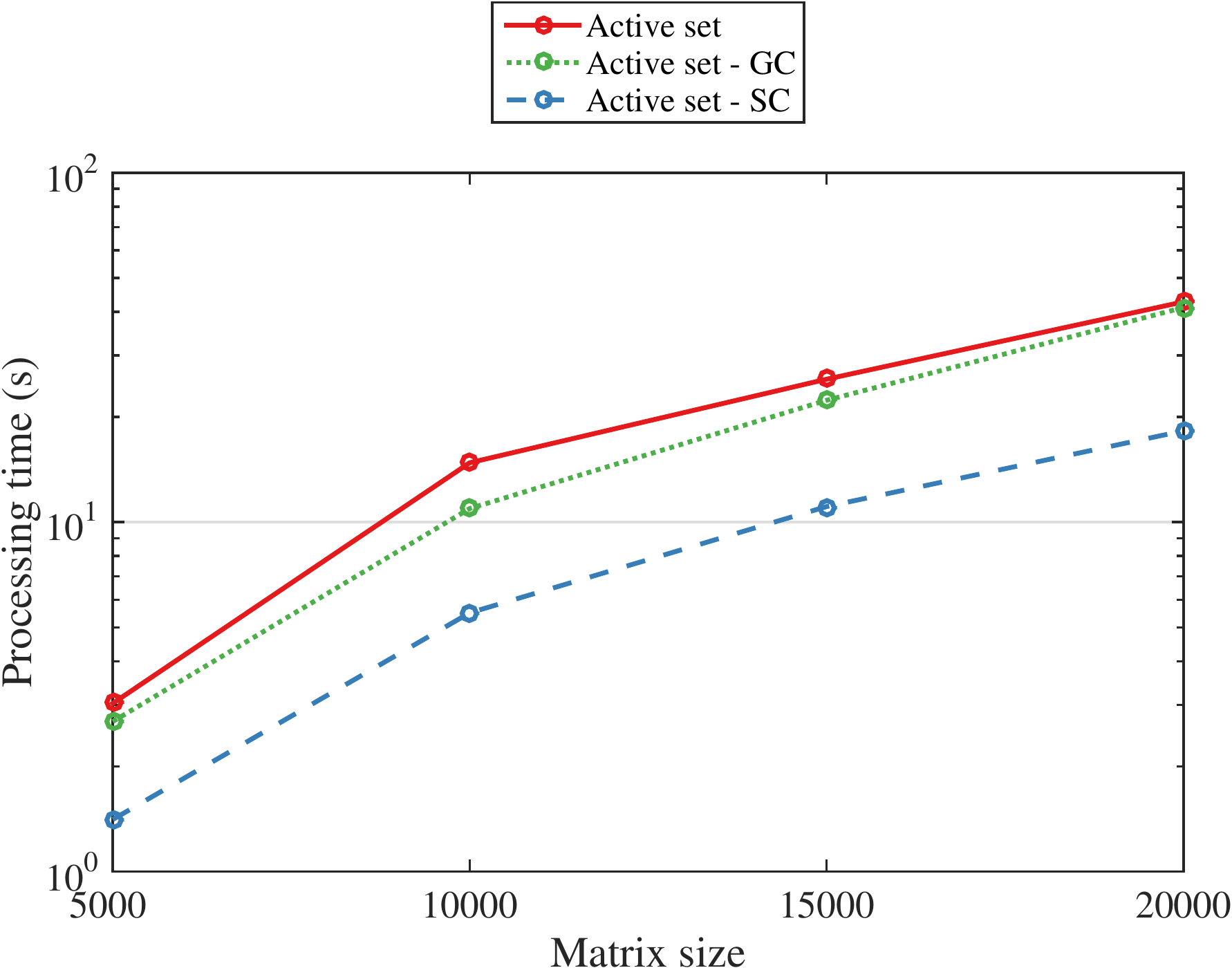} &
			\includegraphics[width=\linewidth]{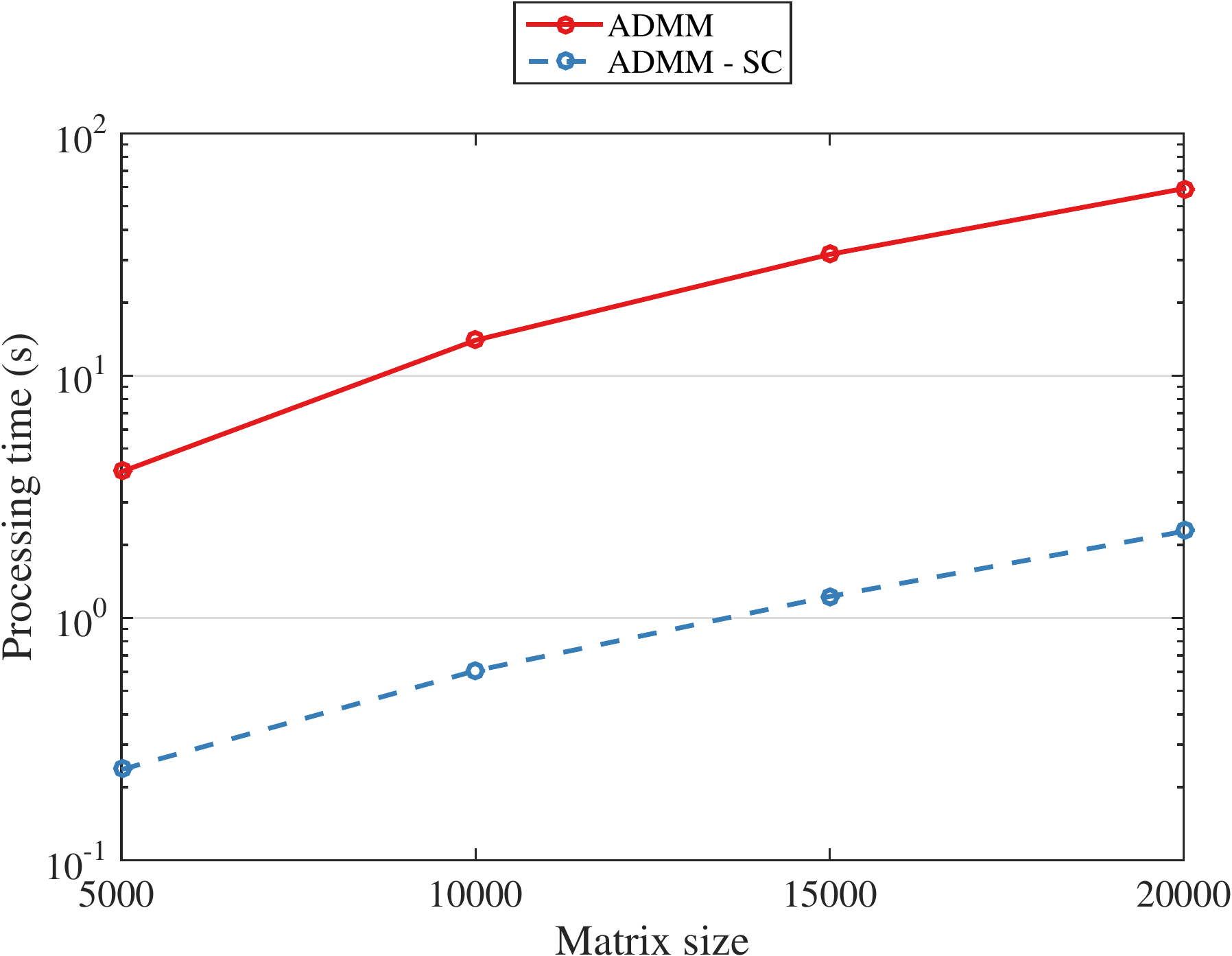} \\
	
			\includegraphics[width=\linewidth]{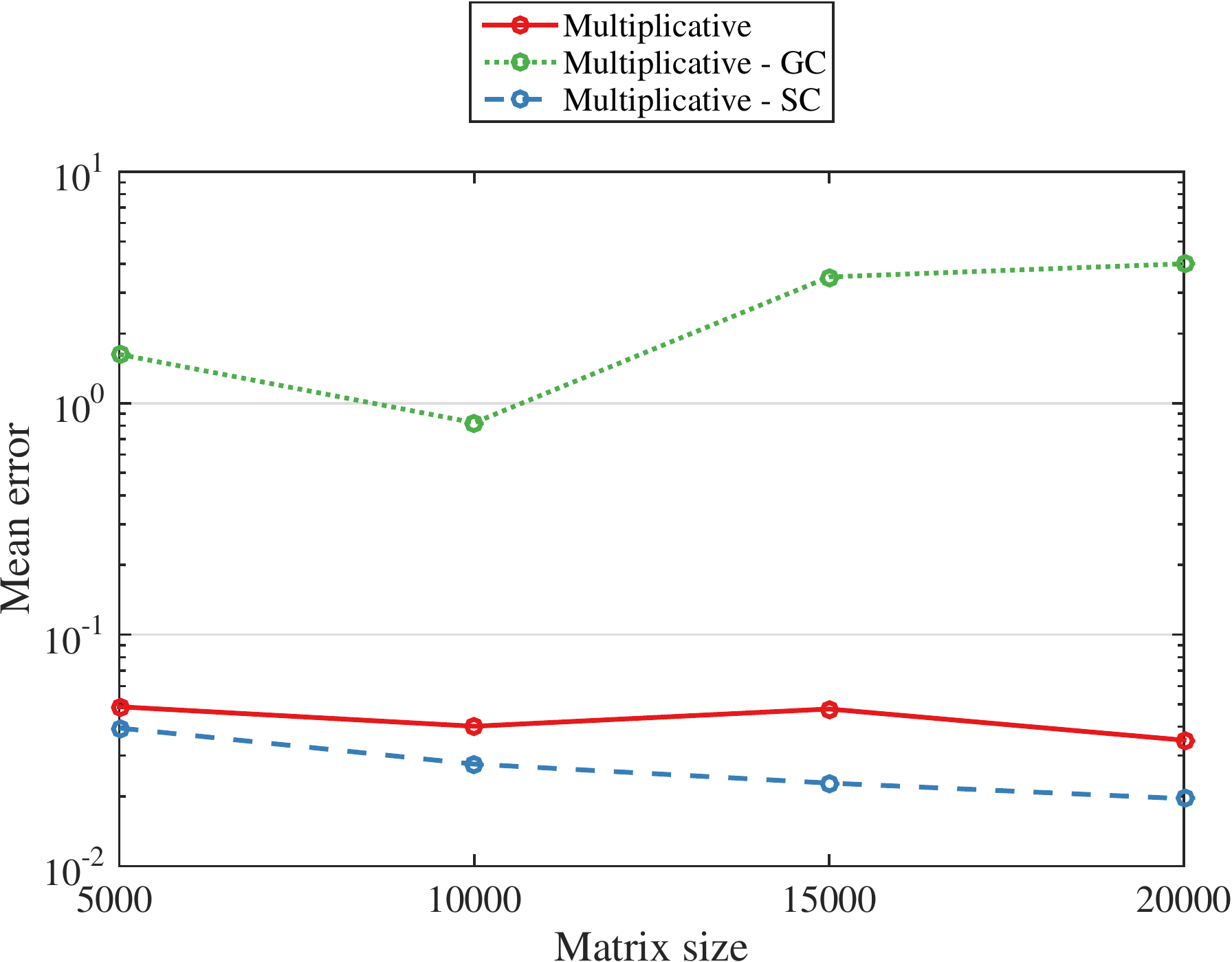} &
			\includegraphics[width=\linewidth]{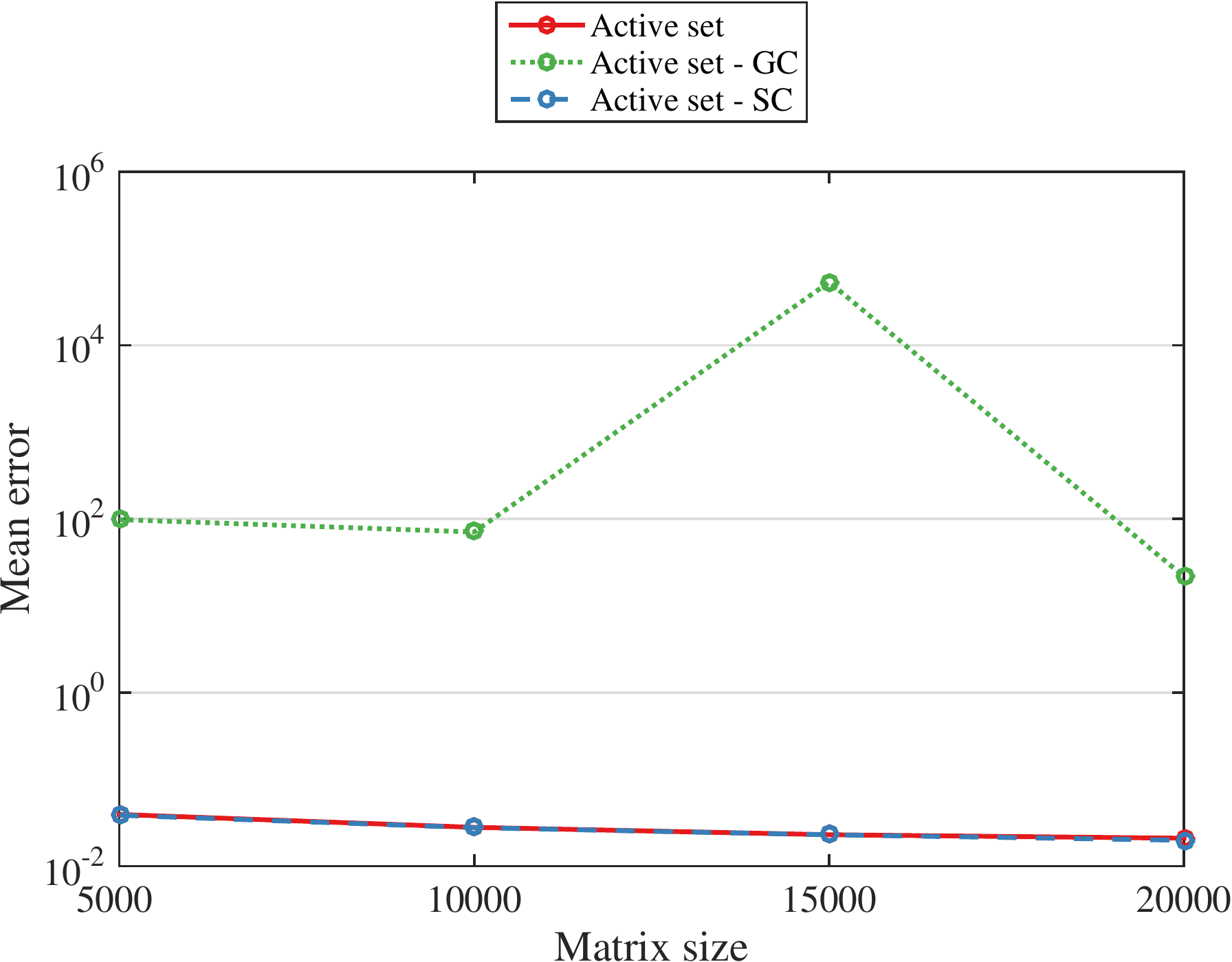} &
			\includegraphics[width=\linewidth]{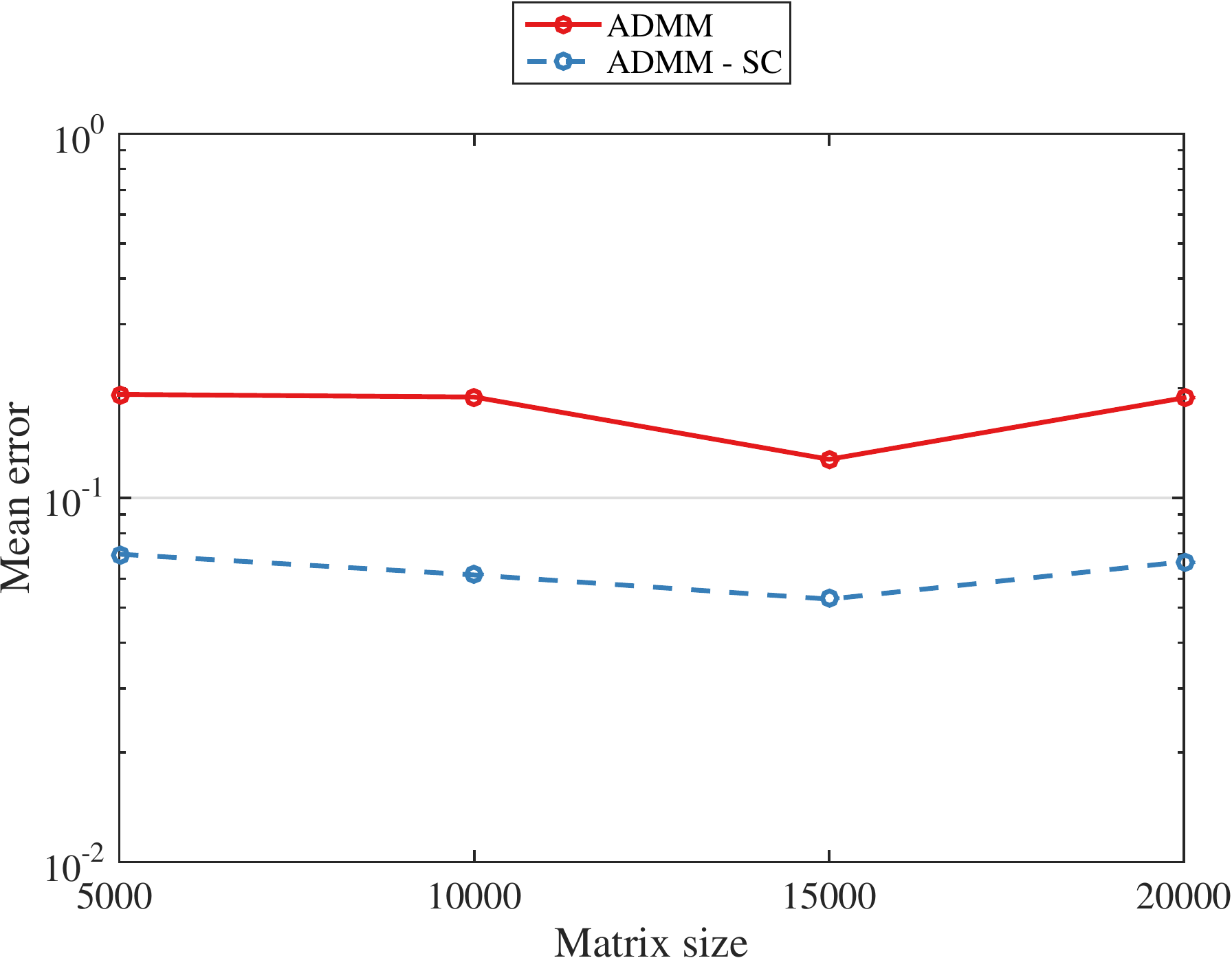} \\
		\end{tabu}
		\end{footnotesize}
		
		\caption{\textbf{Synthetic dense matrices.} The matrix size indicates the number of rows $m$; the number of columns $n$ is fixed to $n=0.75 m$ in all cases. Since the matrices are dense, $\delta = 1$.}
	\end{subfigure}

	\begin{subfigure}{\textwidth}
		\tabulinesep=2pt
	
		\centering
		\begin{footnotesize}
			\begin{tabu} to \textwidth { @{\hspace{4pt}} *{3}{X[c,m] @{\hspace{4pt}}} }
				
				\includegraphics[width=\linewidth]{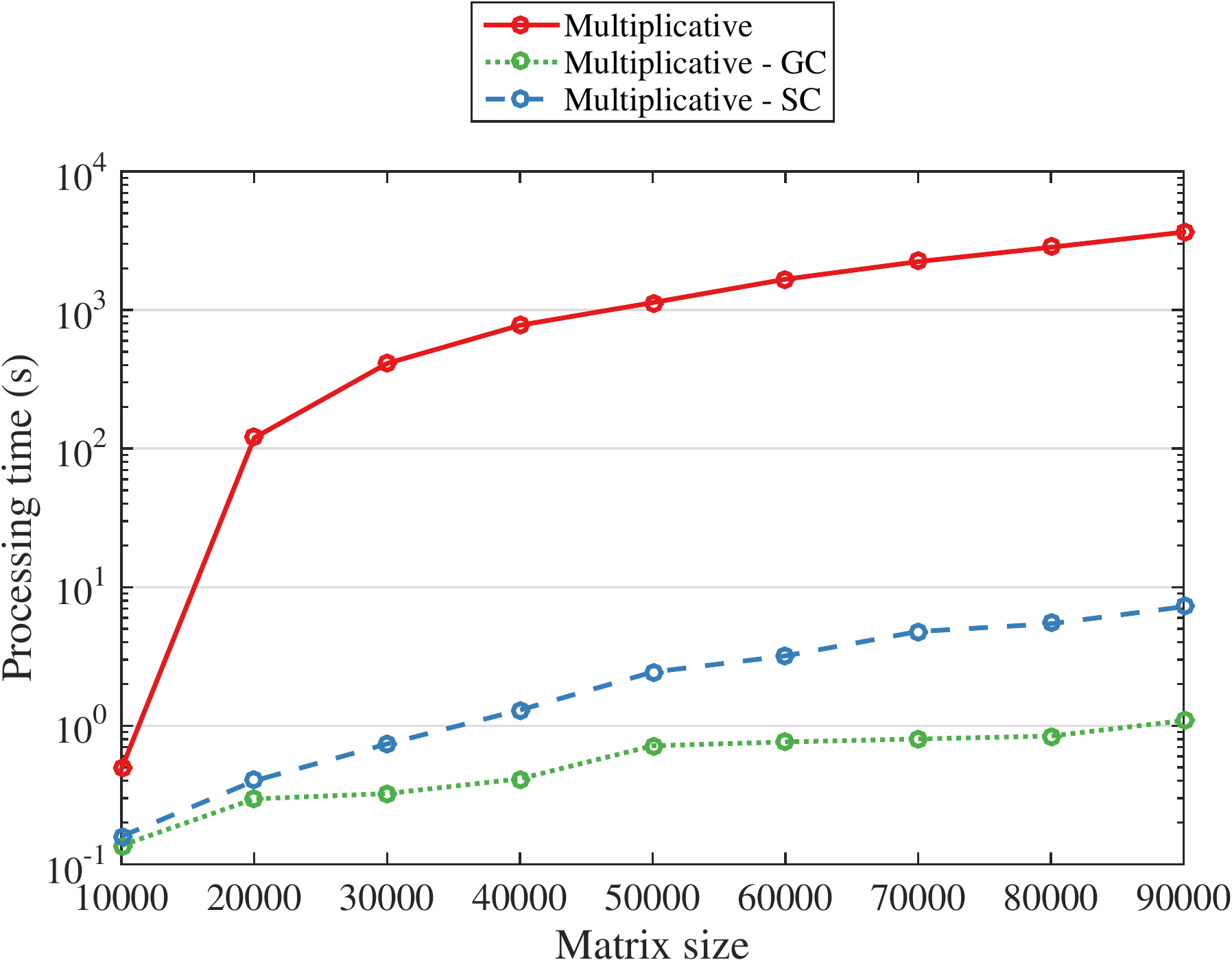} &
				\includegraphics[width=\linewidth]{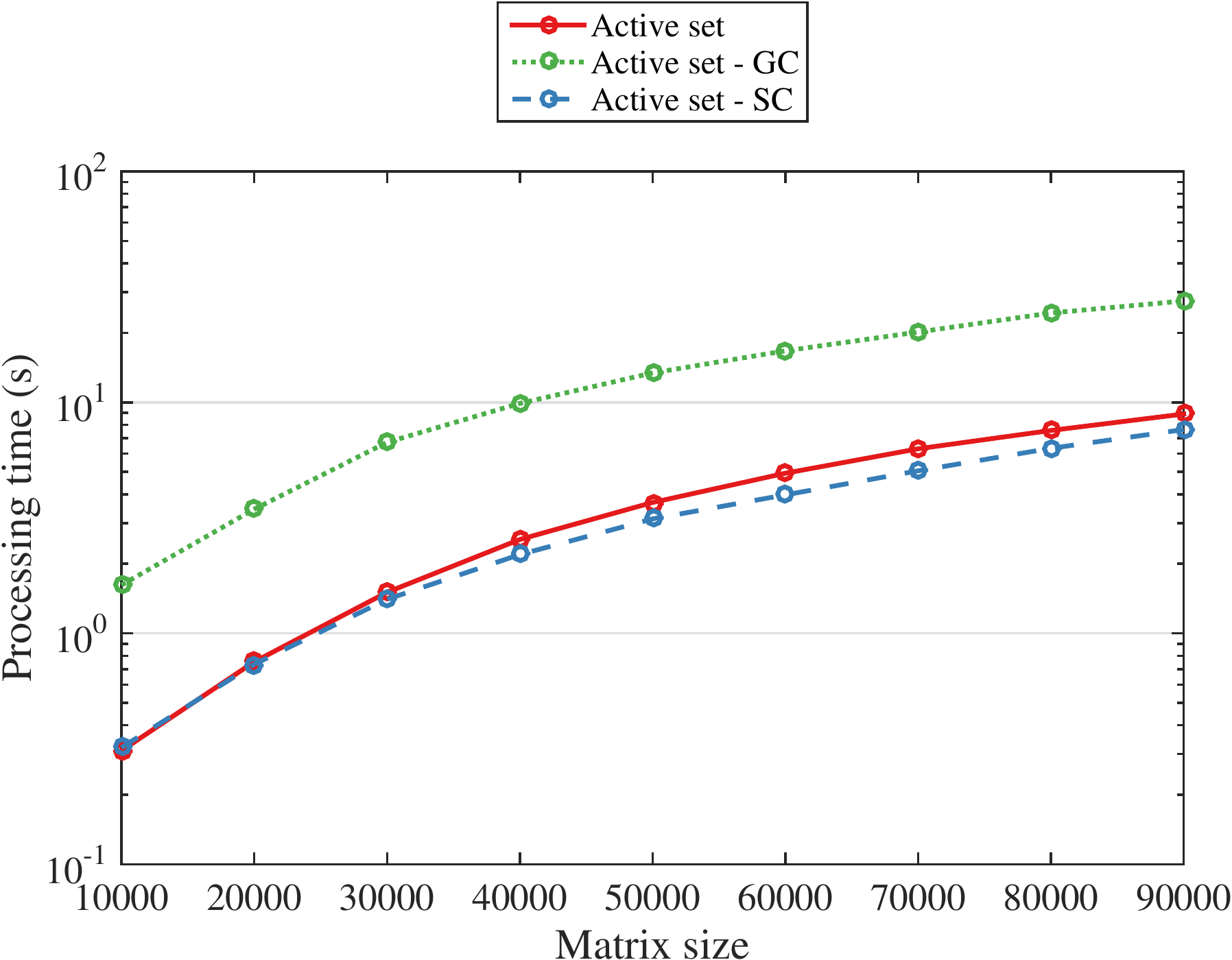} &
				\includegraphics[width=\linewidth]{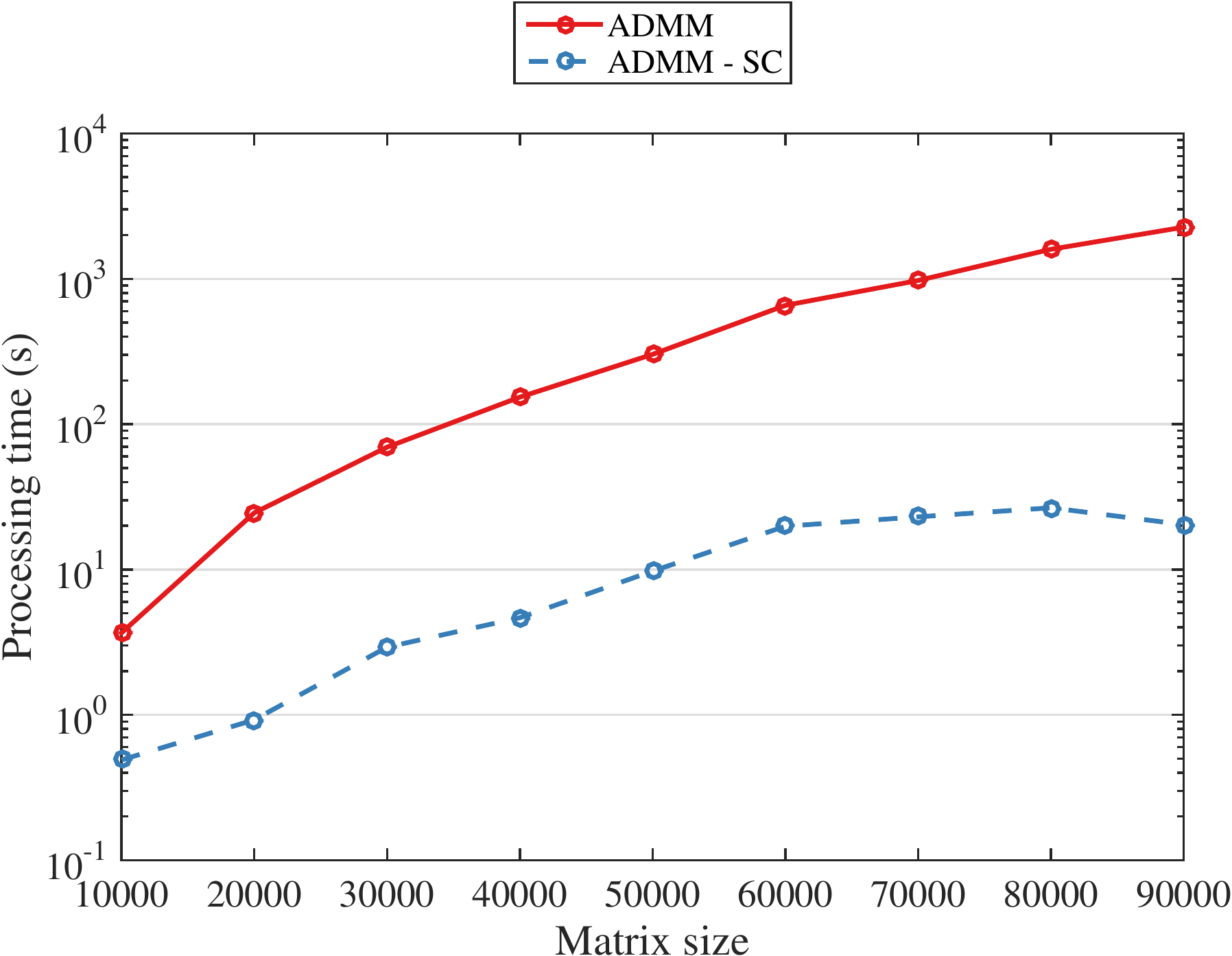} \\
				
				\includegraphics[width=\linewidth]{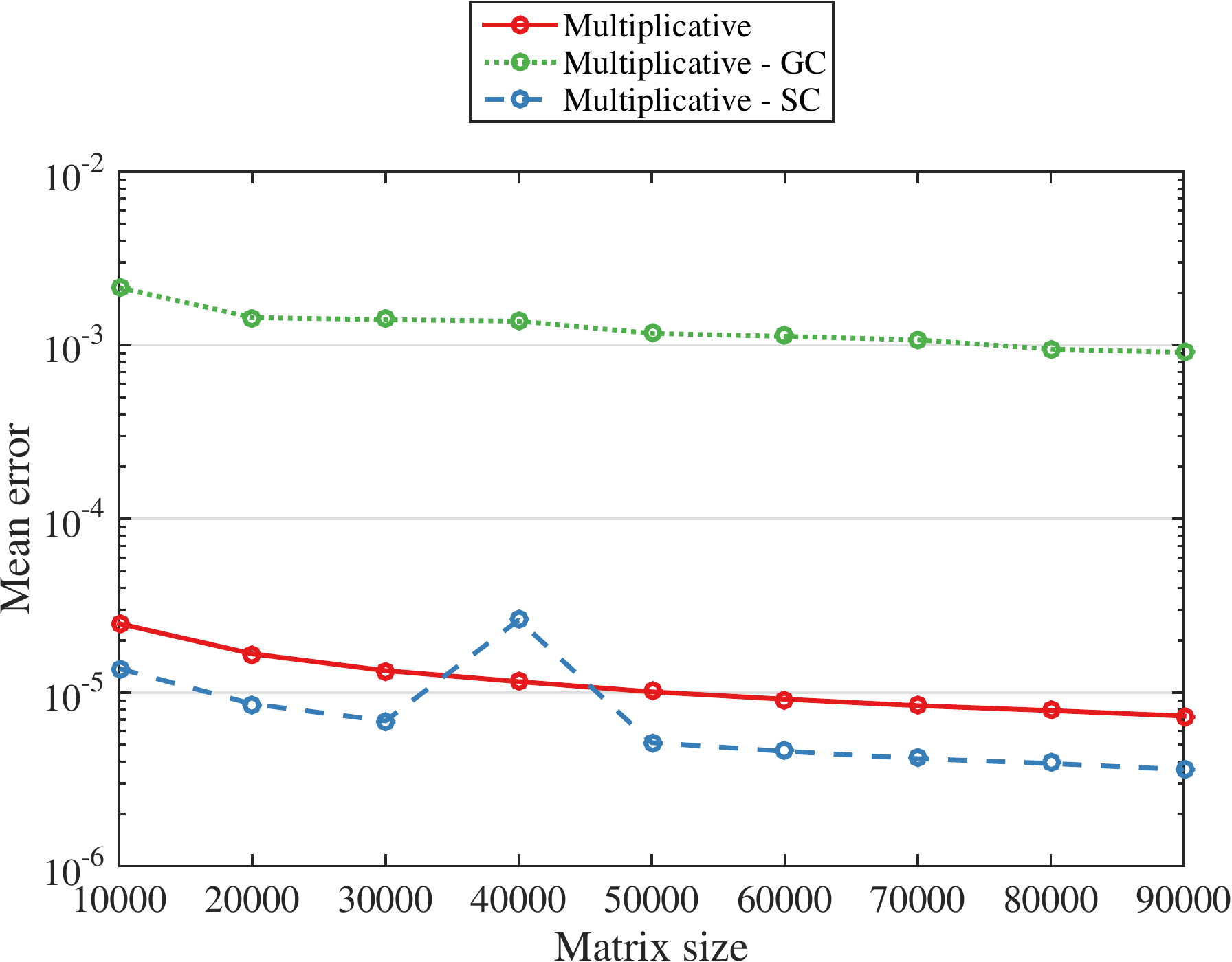} &
				\includegraphics[width=\linewidth]{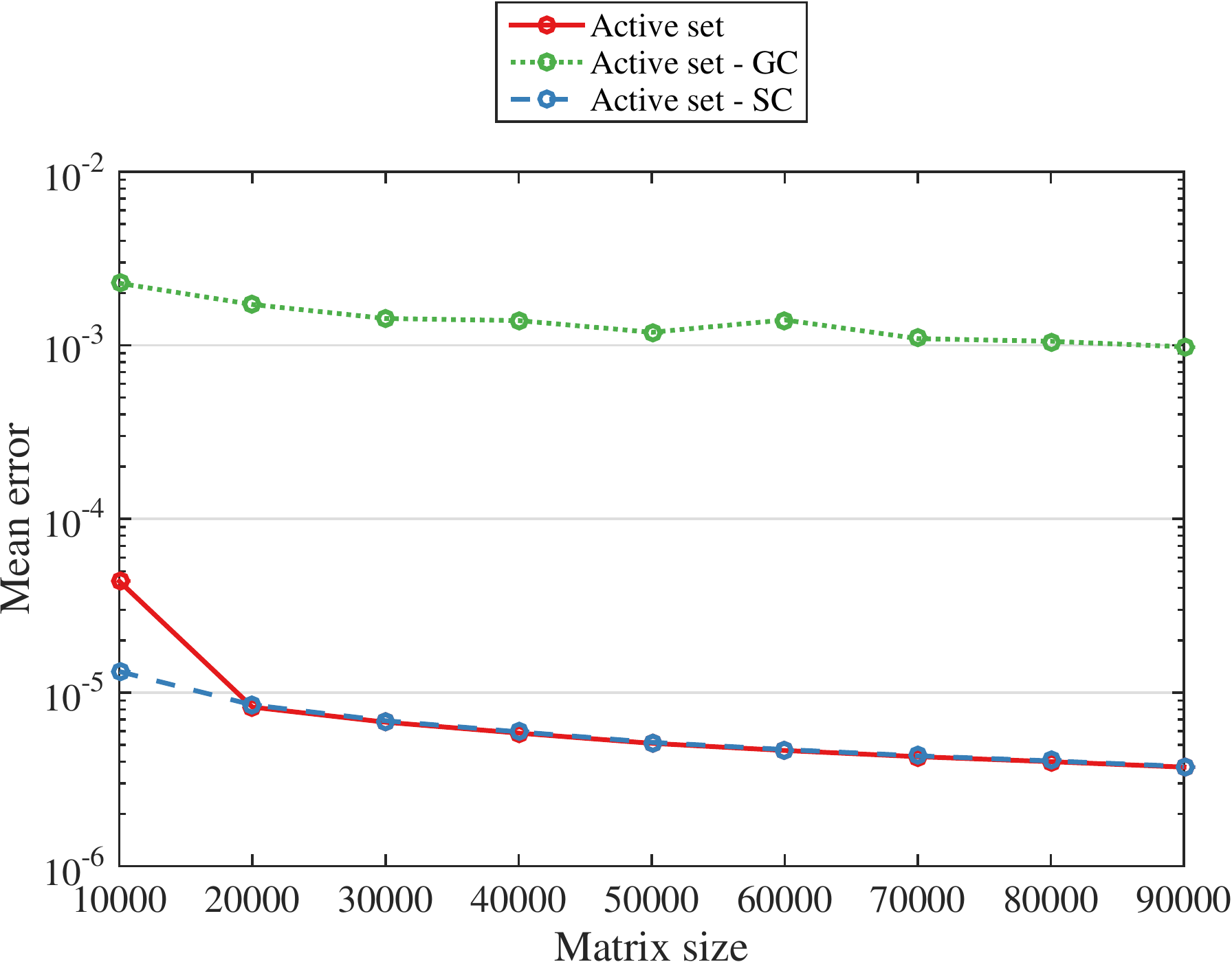} &
				\includegraphics[width=\linewidth]{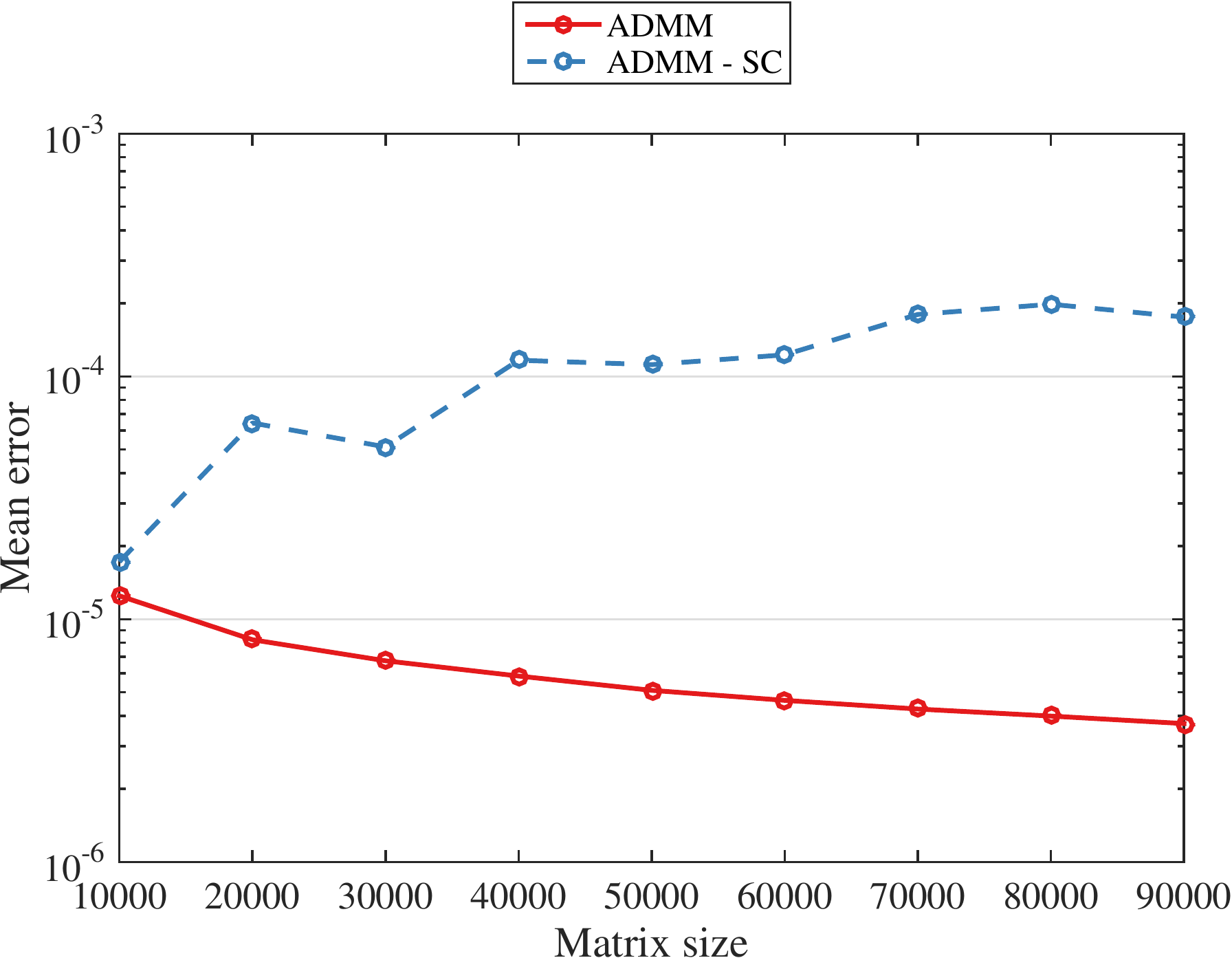} \\
			\end{tabu}
		\end{footnotesize}
		
		\caption{\textbf{Synthetic sparse matrices.} The matrix size indicates the number of rows $m$; the number of columns $n$ and the sparsity level $\delta$ are fixed to $n=0.75 m$ and $\delta = 10^{-2}$ in all cases.}
	\end{subfigure}	

	\caption{\textbf{Performance comparison on synthetic matrices}. We first generate two matrices $\mat{X}_\textsc{gt} \in \Real^{m \times r}$, $\mat{Y}_\textsc{gt} \in \Real^{r \times n}$, where their entries are uniformly distributed in $[0, 1]$ with probability $\delta$, or zero with probability $1-\delta$. We then build $\mat{A} = \mat{X}_\textsc{gt} \mat{Y}_\textsc{gt} + \mat{N}$, where the entries of $\mat{N}$ are normally distributed with probability $\delta^2$, or zero with probability $1-\delta^2$. GC and SC stand for Gaussian and structured compression, respectively. The reconstruction error is reported as the mean over 10 different runs. While both GC and SC are generally faster than the original uncompressed methods (top row), the accuracy levels of the latter are only matched (and sometimes even outmatched) by SC (bottom row).}
	\label{fig:syntheticNMF}
		
\end{figure*}

We also run different NMF algorithms on a hyperspectral positron emission tomography (PET) image, see~\cref{fig:pet}. This example allows to visually compare the errors produced by the different methods. The NMF methods with Gaussian compression create ``clusters'' of errors (particular areas in which the errors seem to concentrate). In \cref{tab:pet} we show several error statistics and the computing time for the different methods. The statistics also reflect the same behavior as our visual previous inspection. Structured compression has a positive effect on the computing time (it decreases), and no significant effect on the error statistics.

\begin{figure*}
	
	\centering
	\begin{footnotesize}
	
	\hfill
	\begin{minipage}{.18\textwidth}
	\centering
	Original image\\[2pt]
	\begin{tikzpicture}
	\node[anchor=south west,inner sep=0] (image) at (0,0) {
		\shortstack{
			\includegraphics[width=\textwidth]{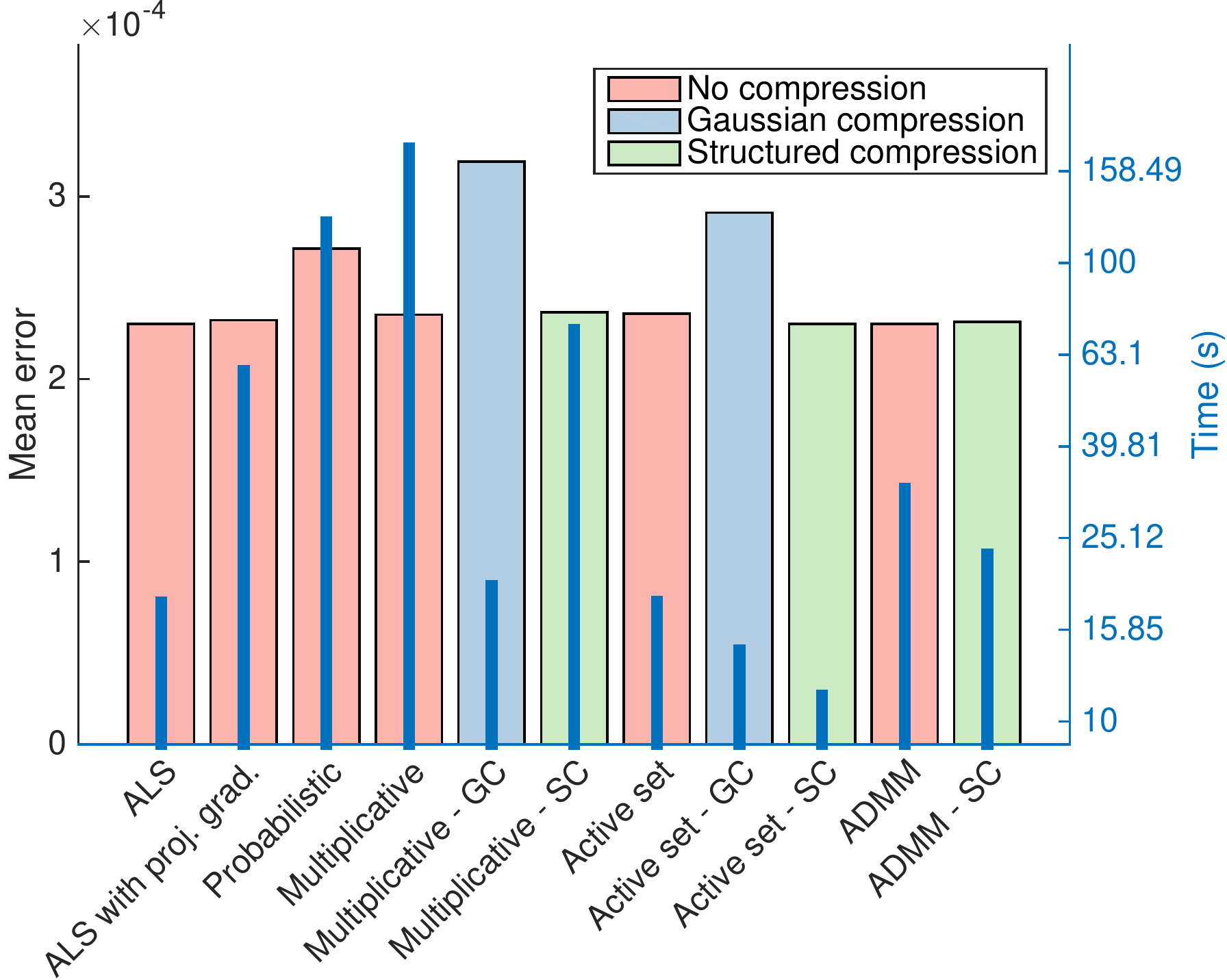}\\
			\includegraphics[width=\textwidth, trim=174px 348px 174px 0px, clip]{brain/brainNMF}
		}
	};
	\begin{scope}[x={(image.south east)},y={(image.north west)}]
	\draw [DarkRed, very thick] (0.17, 0.83) rectangle ++(0.33, 0.165);
	\draw [DarkRed, very thick] (0.17, 0.83) -- (0.00, 0.49);
	\draw [DarkRed, very thick] (0.50, 0.83) -- (1.00, 0.49);
	\draw [DarkRed, very thick] (0.00, 0.00) rectangle (1.00, 0.49);
	\end{scope}
	\end{tikzpicture}%
	\end{minipage}%
	\hfill
	\begin{minipage}{.7\textwidth}
	\centering
	\begin{tabu} to \textwidth { @{\hspace{4pt}} *{5}{X[c,m] @{\hspace{4pt}}} }
	
		Multiplicative &
		Multiplicative - GC &
		Multiplicative - SC &
		ADMM &
		ADMM - SC \\
	
		\begin{tikzpicture}
		\node[anchor=south west,inner sep=0] (image) at (0,0) {
			\shortstack{
				\includegraphics[width=\linewidth]{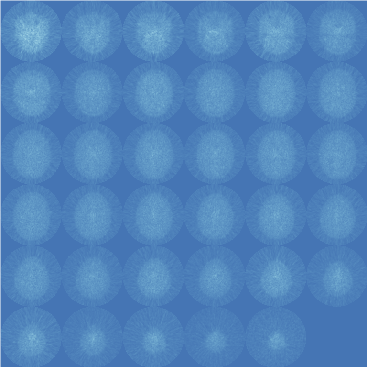} \\
				\includegraphics[width=\linewidth, trim=118px 236px 118px 0px, clip]{brain/nmf_multiplicative}
			}
		};
		\begin{scope}[x={(image.south east)},y={(image.north west)}]
		\draw [DarkRed, very thick] (0.17, 0.83) rectangle ++(0.33, 0.165);
		\draw [DarkRed, very thick] (0.17, 0.83) -- (0.00, 0.49);
		\draw [DarkRed, very thick] (0.50, 0.83) -- (1.00, 0.49);
		\draw [DarkRed, very thick] (0.00, 0.00) rectangle (1.00, 0.49);
		\end{scope}
		\end{tikzpicture}%
		&	
		\begin{tikzpicture}
		\node[anchor=south west,inner sep=0] (image) at (0,0) {
			\shortstack{
				\includegraphics[width=\linewidth]{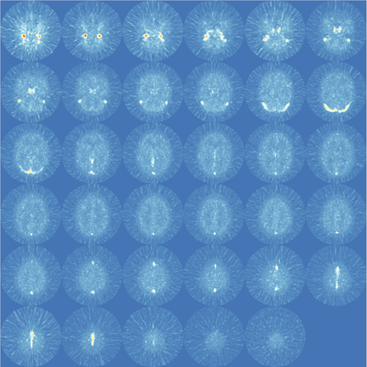} \\
				\includegraphics[width=\linewidth, trim=118px 236px 118px 0px, clip]{brain/nmf_multiplicative_compressed}
			}
		};
		\begin{scope}[x={(image.south east)},y={(image.north west)}]
		\draw [DarkRed, very thick] (0.17, 0.83) rectangle ++(0.33, 0.165);
		\draw [DarkRed, very thick] (0.17, 0.83) -- (0.00, 0.49);
		\draw [DarkRed, very thick] (0.50, 0.83) -- (1.00, 0.49);
		\draw [DarkRed, very thick] (0.00, 0.00) rectangle (1.00, 0.49);
		\end{scope}
		\end{tikzpicture}%
		&
		\begin{tikzpicture}
		\node[anchor=south west,inner sep=0] (image) at (0,0) {
			\shortstack{
				\includegraphics[width=\linewidth]{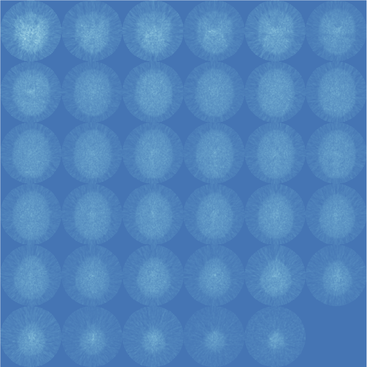} \\
				\includegraphics[width=\linewidth, trim=118px 236px 118px 0px, clip]{brain/nmf_multiplicative_structured_compressed}
			}
		};
		\begin{scope}[x={(image.south east)},y={(image.north west)}]
		\draw [DarkRed, very thick] (0.17, 0.83) rectangle ++(0.33, 0.165);
		\draw [DarkRed, very thick] (0.17, 0.83) -- (0.00, 0.49);
		\draw [DarkRed, very thick] (0.50, 0.83) -- (1.00, 0.49);
		\draw [DarkRed, very thick] (0.00, 0.00) rectangle (1.00, 0.49);
		\end{scope}
		\end{tikzpicture}%
		&
		\begin{tikzpicture}
			\node[anchor=south west,inner sep=0] (image) at (0,0) {
				\shortstack{
					\includegraphics[width=\linewidth]{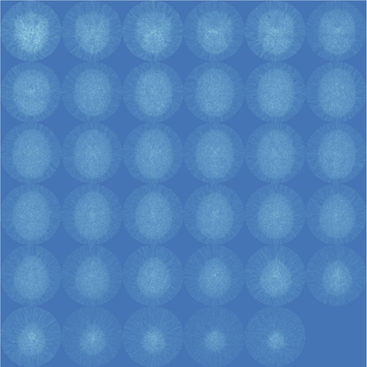} \\
					\includegraphics[width=\linewidth, trim=118px 236px 118px 0px, clip]{brain/nmf_admm}
				}
			};
			\begin{scope}[x={(image.south east)},y={(image.north west)}]
				\draw [DarkRed, very thick] (0.17, 0.83) rectangle ++(0.33, 0.165);
				\draw [DarkRed, very thick] (0.17, 0.83) -- (0.00, 0.49);
				\draw [DarkRed, very thick] (0.50, 0.83) -- (1.00, 0.49);
				\draw [DarkRed, very thick] (0.00, 0.00) rectangle (1.00, 0.49);
			\end{scope}
		\end{tikzpicture}%
		&
		\begin{tikzpicture}
			\node[anchor=south west,inner sep=0] (image) at (0,0) {
				\shortstack{
					\includegraphics[width=\linewidth]{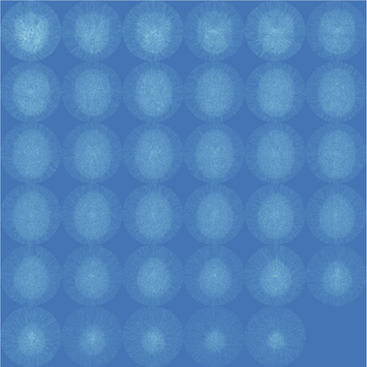} \\
					\includegraphics[width=\linewidth, trim=118px 236px 118px 0px, clip]{brain/nmf_admm_structured_compressed}
				}
			};
			\begin{scope}[x={(image.south east)},y={(image.north west)}]
				\draw [DarkRed, very thick] (0.17, 0.83) rectangle ++(0.33, 0.165);
				\draw [DarkRed, very thick] (0.17, 0.83) -- (0.00, 0.49);
				\draw [DarkRed, very thick] (0.50, 0.83) -- (1.00, 0.49);
				\draw [DarkRed, very thick] (0.00, 0.00) rectangle (1.00, 0.49);
			\end{scope}
		\end{tikzpicture}%
		\\
		\\[-6pt]
		
		Active set &
		Active set - GC &
		Active set - SC &
		ALS &
		ALS (proj. grad.) \\
		
		\begin{tikzpicture}
			\node[anchor=south west,inner sep=0] (image) at (0,0) {
				\shortstack{
					\includegraphics[width=\linewidth]{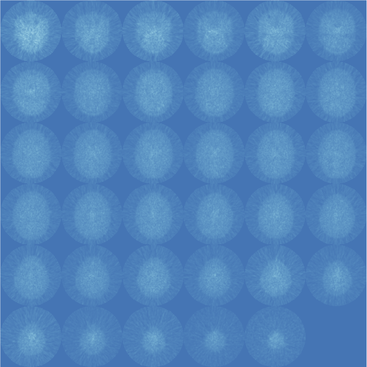} \\
					\includegraphics[width=\linewidth, trim=118px 236px 118px 0px, clip]{brain/nmf_activeSet}
				}
			};
			\begin{scope}[x={(image.south east)},y={(image.north west)}]
				\draw [DarkRed, very thick] (0.17, 0.83) rectangle ++(0.33, 0.165);
				\draw [DarkRed, very thick] (0.17, 0.83) -- (0.00, 0.49);
				\draw [DarkRed, very thick] (0.50, 0.83) -- (1.00, 0.49);
				\draw [DarkRed, very thick] (0.00, 0.00) rectangle (1.00, 0.49);
			\end{scope}
		\end{tikzpicture}%
		&
		\begin{tikzpicture}
			\node[anchor=south west,inner sep=0] (image) at (0,0) {
				\shortstack{
					\includegraphics[width=\linewidth]{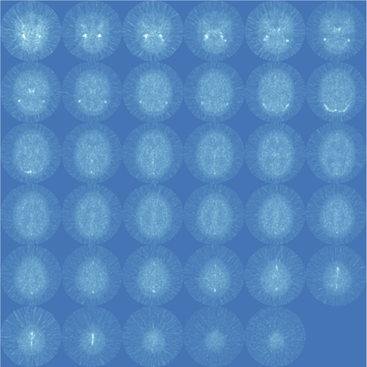} \\
					\includegraphics[width=\linewidth, trim=118px 236px 118px 0px, clip]{brain/nmf_activeSet_compressed}
				}
			};
			\begin{scope}[x={(image.south east)},y={(image.north west)}]
				\draw [DarkRed, very thick] (0.17, 0.83) rectangle ++(0.33, 0.165);
				\draw [DarkRed, very thick] (0.17, 0.83) -- (0.00, 0.49);
				\draw [DarkRed, very thick] (0.50, 0.83) -- (1.00, 0.49);
				\draw [DarkRed, very thick] (0.00, 0.00) rectangle (1.00, 0.49);
			\end{scope}
		\end{tikzpicture}%
		&
		\begin{tikzpicture}
			\node[anchor=south west,inner sep=0] (image) at (0,0) {
				\shortstack{
					\includegraphics[width=\linewidth]{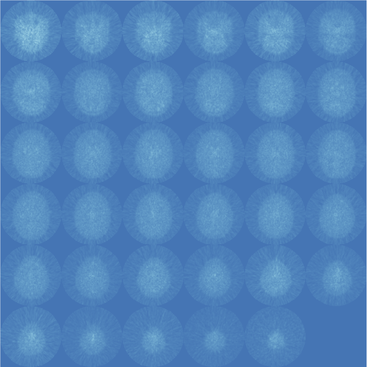} \\
					\includegraphics[width=\linewidth, trim=118px 236px 118px 0px, clip]{brain/nmf_activeSet_structured_compressed}
				}
			};
			\begin{scope}[x={(image.south east)},y={(image.north west)}]
				\draw [DarkRed, very thick] (0.17, 0.83) rectangle ++(0.33, 0.165);
				\draw [DarkRed, very thick] (0.17, 0.83) -- (0.00, 0.49);
				\draw [DarkRed, very thick] (0.50, 0.83) -- (1.00, 0.49);
				\draw [DarkRed, very thick] (0.00, 0.00) rectangle (1.00, 0.49);
			\end{scope}
		\end{tikzpicture}%
		&
		\begin{tikzpicture}
			\node[anchor=south west,inner sep=0] (image) at (0,0) {
				\shortstack{
					\includegraphics[width=\linewidth]{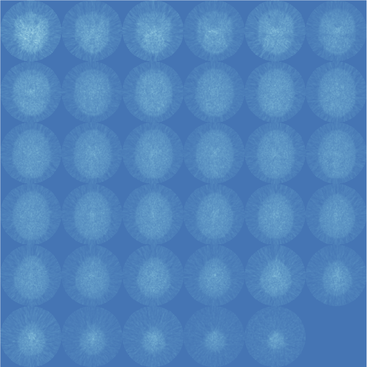} \\
					\includegraphics[width=\linewidth, trim=118px 236px 118px 0px, clip]{brain/nmf_als.png}
				}
			};
			\begin{scope}[x={(image.south east)},y={(image.north west)}]
				\draw [DarkRed, very thick] (0.17, 0.83) rectangle ++(0.33, 0.165);
				\draw [DarkRed, very thick] (0.17, 0.83) -- (0.00, 0.49);
				\draw [DarkRed, very thick] (0.50, 0.83) -- (1.00, 0.49);
				\draw [DarkRed, very thick] (0.00, 0.00) rectangle (1.00, 0.49);
			\end{scope}
		\end{tikzpicture}%
		&
		\begin{tikzpicture}
			\node[anchor=south west,inner sep=0] (image) at (0,0) {
				\shortstack{
					\includegraphics[width=\linewidth]{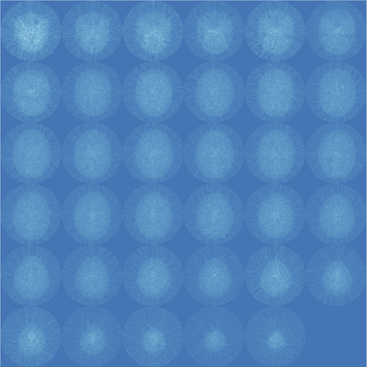} \\
					\includegraphics[width=\linewidth, trim=118px 236px 118px 0px, clip]{brain/nmf_cjlin}
				}
			};
			\begin{scope}[x={(image.south east)},y={(image.north west)}]
				\draw [DarkRed, very thick] (0.17, 0.83) rectangle ++(0.33, 0.165);
				\draw [DarkRed, very thick] (0.17, 0.83) -- (0.00, 0.49);
				\draw [DarkRed, very thick] (0.50, 0.83) -- (1.00, 0.49);
				\draw [DarkRed, very thick] (0.00, 0.00) rectangle (1.00, 0.49);
			\end{scope}
		\end{tikzpicture}%
		\\

	\end{tabu}
	\end{minipage}%
	\begin{minipage}{.05\textwidth}
		\centering
		\begin{tikzpicture}
			\node[anchor=south west,inner sep=0] (image) at (0,0) {
				\frame{\includegraphics[width=.15\textwidth]{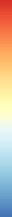}}
			};
			\begin{scope}[x={(image.south east)},y={(image.north west)}]
				\draw [->, very thick] (-1, 0.2) -- (-1, 0.8);
				\node[rotate=90] at (-2.5,.5) {Higher error};
			\end{scope}
		\end{tikzpicture}%
	\end{minipage}%
	\hfill
	\end{footnotesize}
	
	\caption{\textbf{Reconstruction errors when compressing a positron emission tomography (PET) image} (\url{http://cogsys.imm.dtu.dk/toolbox/nmf/}). The image is composed of $40$ temporal frames, where each frame is a $128 \times 128 \times 35$ 3D image (35 is the number of $128 \times 128$ slices). The matrix size is then $573440 \times 40$ and we perform NMF with $r = 5$. As we can observe in each slice (a few of them are highlighted with zoom-ins), the use of Gaussian compression (GC) increases the reconstruction errors, while structured random compression (SC) has no identifiable effect. See \cref{tab:pet} for additional numerical results.}
	\label{fig:pet}
\end{figure*}

\begin{table}[t]
	\caption{\textbf{Performance when compressing a PET image}. See \cref{fig:pet} for a detailed explanation of the setup. As we can observe, the use of Gaussian compression (GC) is detrimental to the reconstruction error (higher values indicate a lower error), while the proposed structured random compression (SC) has no significant impact on it. As a counterpart, the use of SC significantly decreases the computing time with respect to the original method. The best values for each column are highlighted in green.}
	\label{tab:pet}

	\centering
	\begin{threeparttable}[b]
		\begin{tabular}{l *{3}{d{1.3}} d{3.3}}
			\toprule
			& \multicolumn{3}{c}{Error ($-\log_{10}$)} & \multicolumn{1}{c}{\multirow{2}{*}{Time (s)}} \\
			\cmidrule(lr){2-4}
			& \multicolumn{1}{c}{Mean} & \multicolumn{1}{c}{STD} & \multicolumn{1}{c}{Median} \\
			
			\midrule

			Multiplicative & 3.628 & 3.435 & 3.980 & 177.813 \\
			Multiplicative - GC & 3.496 & 3.304 & 3.834 & 19.742 \\
			Multiplicative - SC & 3.626 & 3.433 & 3.979 & 71.437 \\
			ADMM  & \cellcolor{PaleGreen} 3.638 & \cellcolor{PaleGreen} 3.436 & 4.027 & 32.195 \\
			ADMM - SC & 3.636 & 3.433 & \cellcolor{PaleGreen} 4.028 & 23.168\\			
			Active set & 3.628 & \cellcolor{PaleGreen} 3.436 & 3.976 & 18.251 \\
			Active set - GC & 3.536 & 3.350 & 3.882 & 14.277 \\
			Active set - SC & \cellcolor{PaleGreen} 3.638 & \cellcolor{PaleGreen} 3.436 & 4.024 & \cellcolor{PaleGreen} 11.371 \\
			ALS\tnote{1}   & \cellcolor{PaleGreen} 3.638 & \cellcolor{PaleGreen} 3.436 & 4.023 & 18.162 \\
			ALS with proj. grad.\tnote{1} & 3.634 & \cellcolor{PaleGreen} 3.436 & 4.004 & 58.208 \\
			
			\bottomrule
		\end{tabular}%
		\begin{tablenotes}
			\item [1] Obtained from \url{http://cogsys.imm.dtu.dk/toolbox/nmf/}.
		\end{tablenotes}
	\end{threeparttable}
	
\end{table}

Climate datasets are very interesting to analyze using NMF. We believe that the evidence of a low rank model within climate data is of interest by itself. Nonnegativiy is a useful addition since, under this model, the effects of different factors cannot cancel each other. The technical details and results of an experiment using climate data are shown in \cref{fig:climate}. In this case, we only use the active set method for our comparisons. We found that two factors explain the data with enough accuracy. Both factors seem to correspond to two very different seasons across the globe, and they exhibit inversely correlated periodic patterns. While the left and right factors obtained using structured compression are very similar to their uncompressed counterparts, Gaussian compression introduces visible artifacts in the resulting factorization. Structured random compression also is the fastest of the three methods.

\begin{figure*}
	\centering
	
	\begin{subfigure}{\textwidth}
		\centering
		\begin{small}
		
		\centerline{
			\hfill
			\shortstack{
				Active set - GC\\[1pt]
				\includegraphics[width=.25\textwidth]{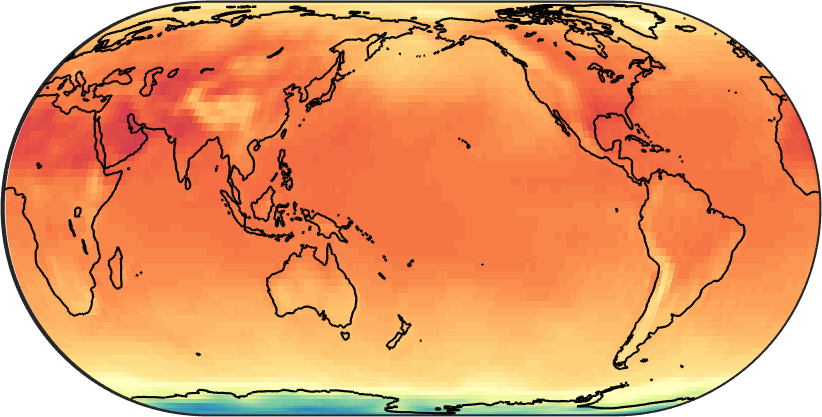} \\
				\includegraphics[width=.25\textwidth]{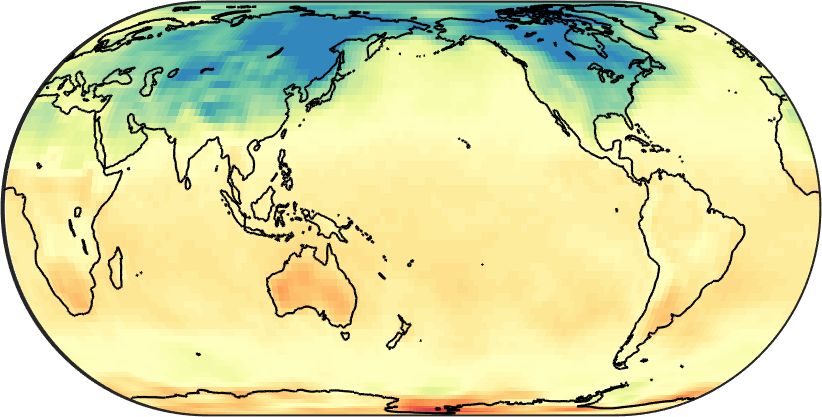}
			}
			\hfill
			\shortstack{
				Active set \\[1pt]
				\includegraphics[width=.25\textwidth]{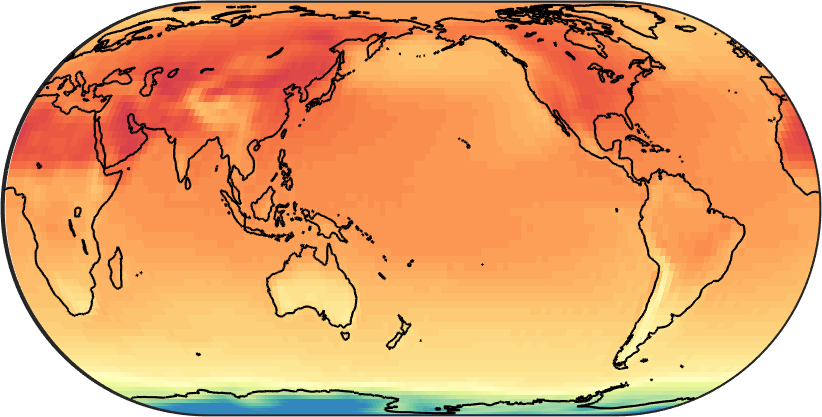} \\
				\includegraphics[width=.25\textwidth]{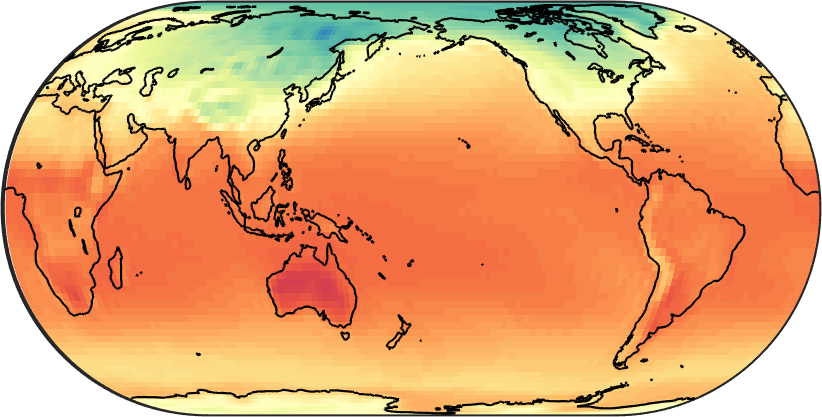}
			}
			\hfill
			\shortstack{
				Active set - SC\\[1pt]
				\includegraphics[width=.25\textwidth]{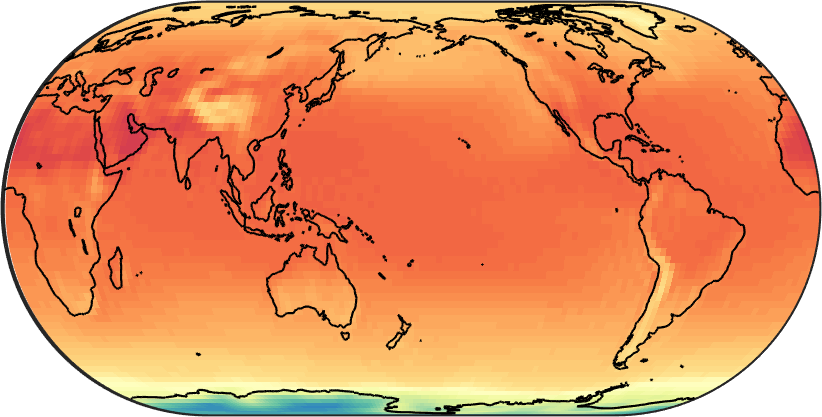} \\
				\includegraphics[width=.25\textwidth]{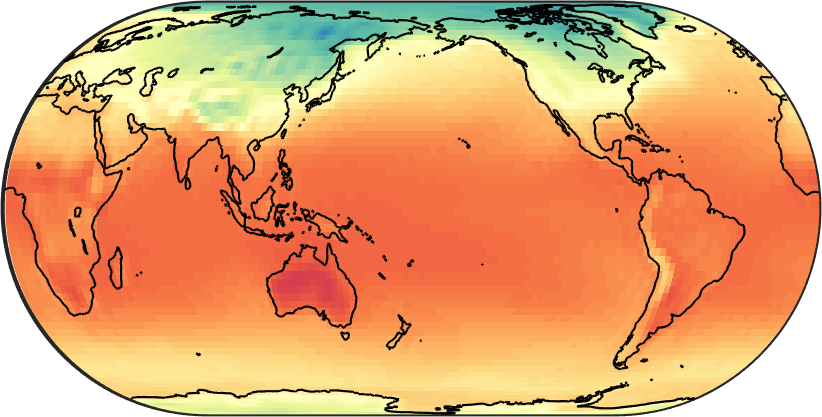}
			}
			\hfill
		}
		\end{small}
		
		\caption{\textbf{Left factors analysis}. Interestingly, the first factor (top row) corresponds to summer and winter in the north and south hemispheres, respectively, while the second factor (bottom row) corresponds to summer in the south and winter in the north. Visually, it is very clear that SC introduces much less artifacts than GC compared to the vanilla method (center).}
		\label{fig:climate_U}
	\end{subfigure}
	
	\begin{subfigure}{\textwidth}
		\centering
		\begin{small}

		\includegraphics[width=.85\textwidth]{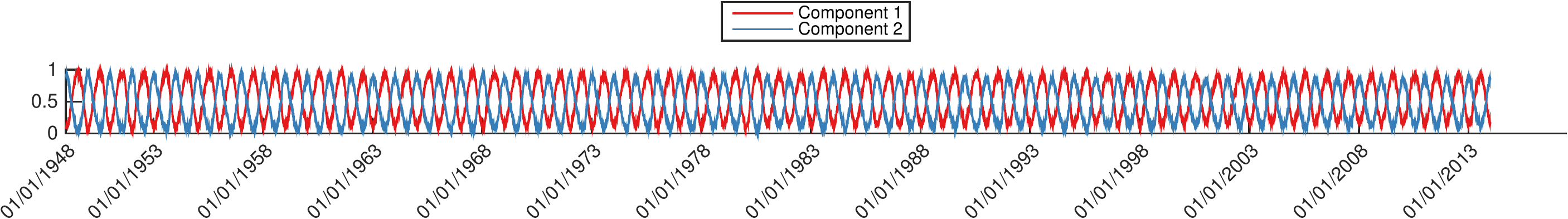}
		
		\vskip10pt
		
		\centerline{
			\hfill
			\shortstack{
				Active set - GC\\
				\includegraphics[width=.27\textwidth]{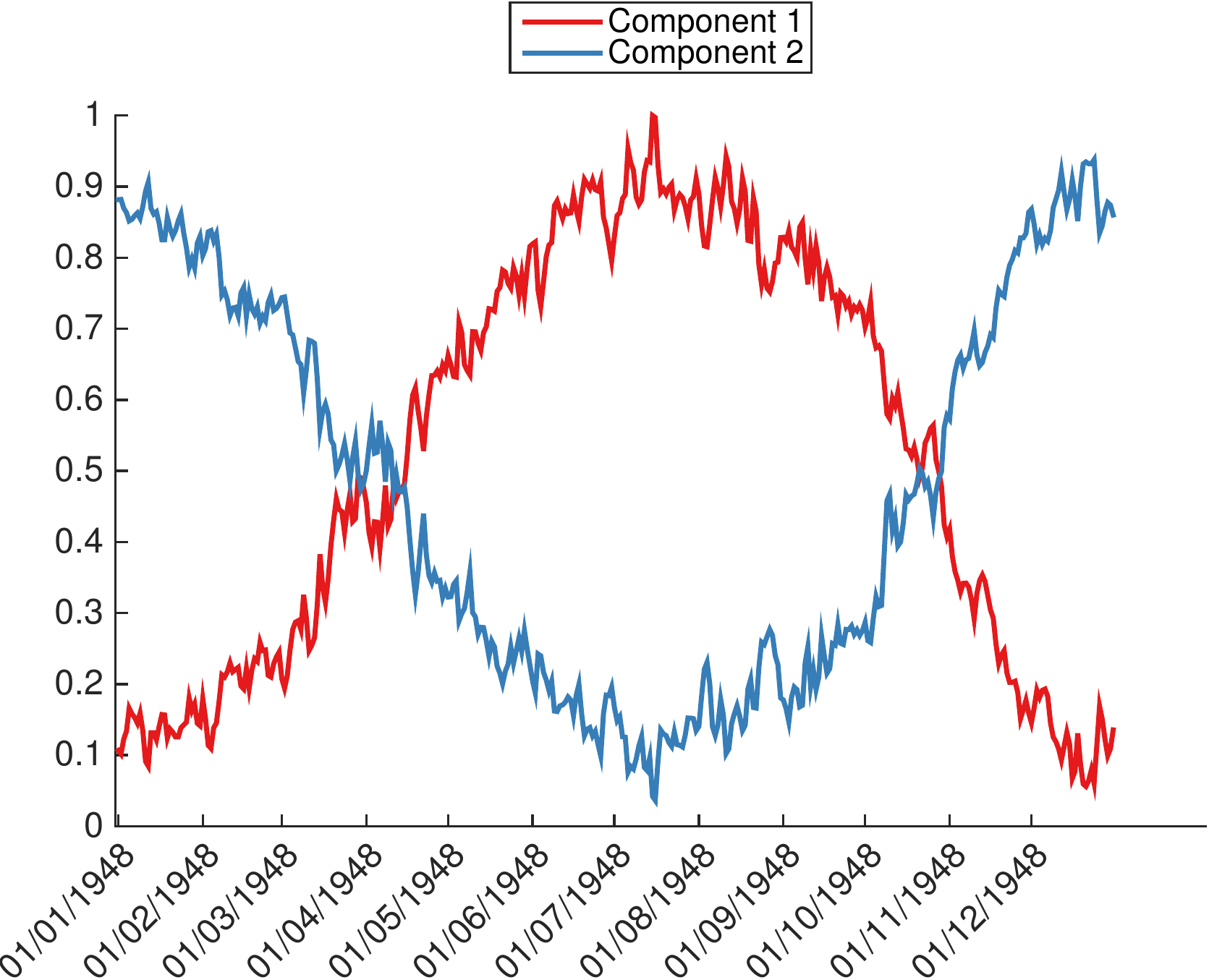}
			}
			\hfill
			\shortstack{
				Active set\\
				\includegraphics[width=.27\textwidth]{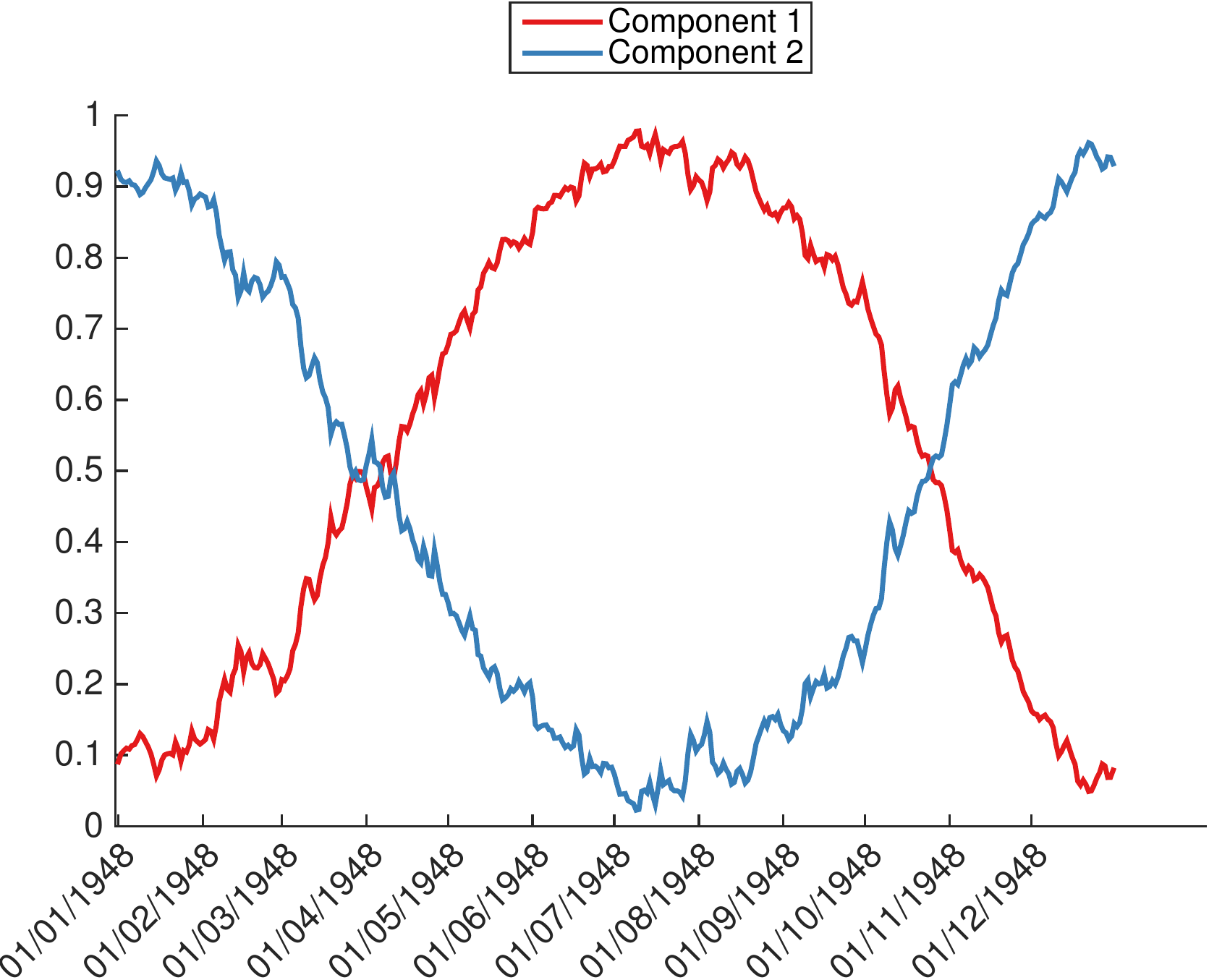}
			}
			\hfill
			\shortstack{
				Active set - SC\\
				\includegraphics[width=.27\textwidth]{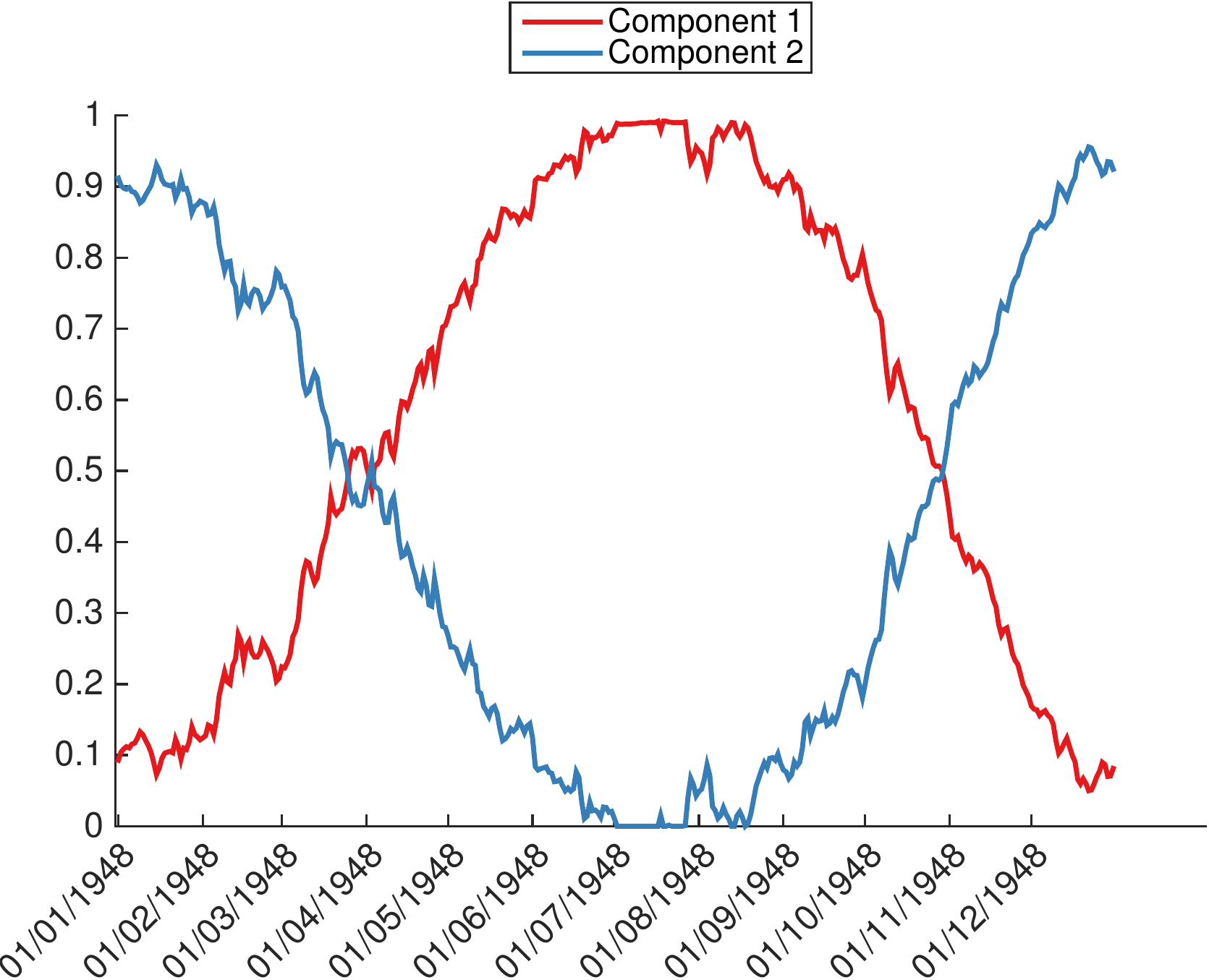}
			}
			\hfill
		}
		\end{small}
		\caption{\textbf{Right factors analysis}. In the top row, we can easily observe that the two factors are periodic and inversely correlated, corroborating the winter/summer duality between both components. In the bottom row, we observe that the GC factors are much more noisy (about an order of magnitude larger), compared to the original and SC methods.}
		\label{fig:climate_V}
	\end{subfigure}

	\caption{\textbf{NMF on gridded climate data} (\url{http://www.esrl.noaa.gov/psd/repository/}). The data contains daily mean surface temperatures arranged in a $144 \times 73$ grid since 1948 (23742 days in total), forming a $10512 \times 23742$ matrix. We perform NMF using the active set method with $r=2$. The computing times for the method in its vanilla version (center), with Gaussian compression (GC) and, with structured random compression (GC) were $50$, $70$, and $20$ seconds, respectively; the respective relative reconstruction errors were $0.0459$, $0.0537$, and $0.0458$, confirming the conclusions reached through visual inspection.}
	\label{fig:climate}
\end{figure*}

Our last classical NMF example consists of a popular application: biclustering. In this case, we bicluster a bipartite social network, i.e., that contains two different types of nodes. In our particular example, these two types correspond to characters from Marvel comic books and to the comic books in which they appear. We performed NMF with $r=10$ (recall that $r$ is the number of factors). We then thresholded each column of $\mat{X}$ and each row of $\mat{Y}$ to obtain sparse components that we define as a bicluster (we could have also added a sparsity term to the formulation, but opted for a simpler approach that does not introduce additional complexity). For each column (row) of $\mat{X}$ ($\mat{Y}$), we set to zero the entries smaller than the column (row) mean plus three standard deviations. Then, for display purposes, we only keep the largest 25 entries in each column of $\mat{X}$ if there are more than that number of nonzero entries. In \cref{fig:marvel} we show two of the biclusters obtained in such a way. It becomes quickly apparent that structured compression does not introduce significant artifacts in the biclusters, whereas the clusters found with Gaussian compression are heavily intertwined (all ten factors seem to be mixed together). For example, Mary Jane Parker-Watson, Spider-Man's wife, is not a recurring character of the Fantastic Four comic books.

\begin{figure*}
	
	\centering

	\newlength{\marvellength}
	\setlength{\marvellength}{0.2\textwidth}
	\tabulinesep=2pt
	\begin{footnotesize}
	\begin{tabu} to .96\textwidth{ @{\hspace{0pt}} m{0.01\textwidth} @{\hspace{2pt}} *{4}{ X[c,m] @{\hspace{0pt}}} }
		
		&
		\multicolumn{2}{c}{Fantastic Four} &
		\multicolumn{2}{c}{Spider-Man} \\
		\cmidrule(l{4pt}r{8pt}){2-3}
		\cmidrule(l{2pt}r{8pt}){4-5}
		&
		Left factor &
		Right factor &
		Left factor &
		Right factor \\
		
		\begin{sideways}Active set - GC\end{sideways}&
		\includegraphics[height=\marvellength]{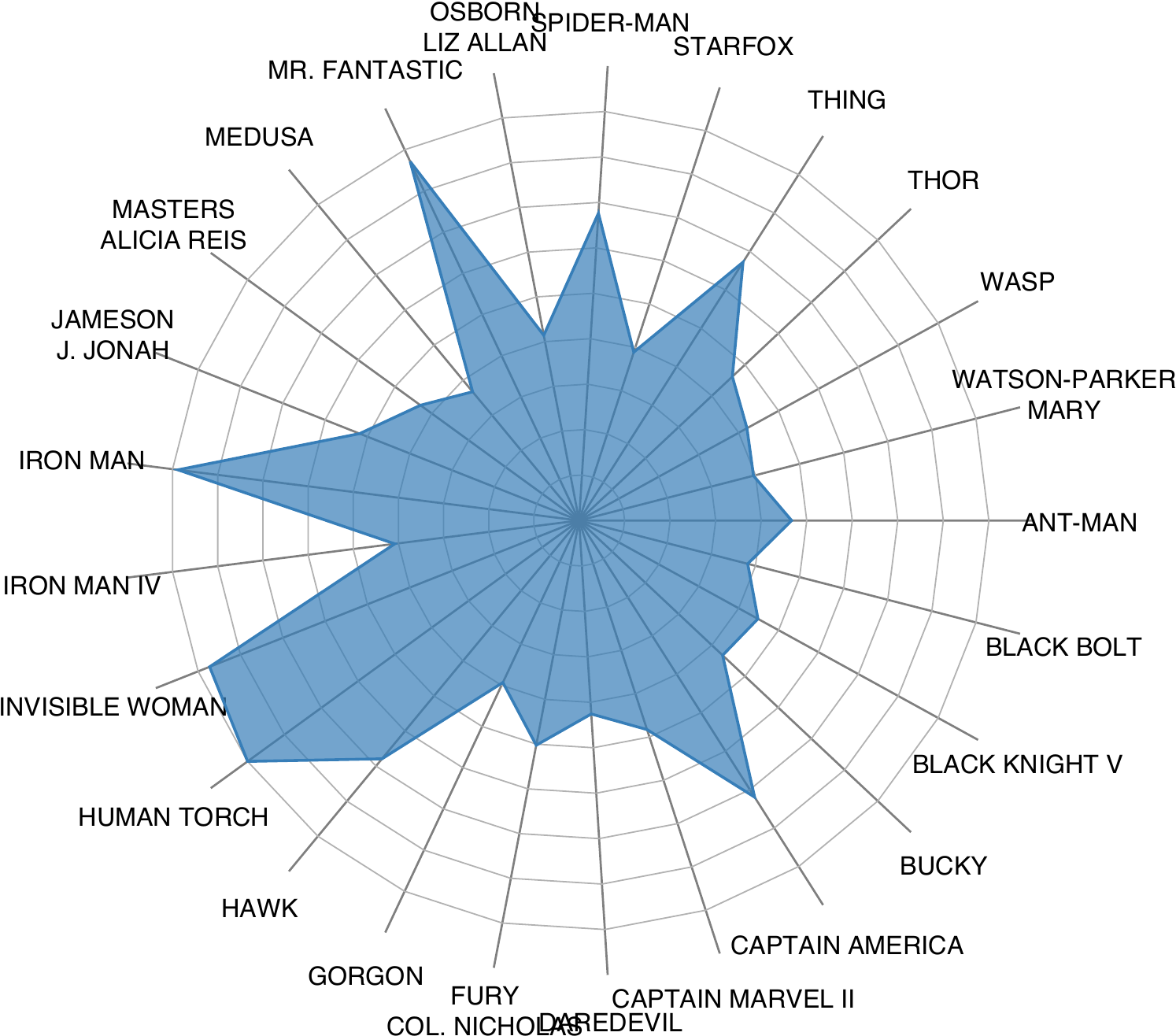}&
		\includegraphics[height=\marvellength]{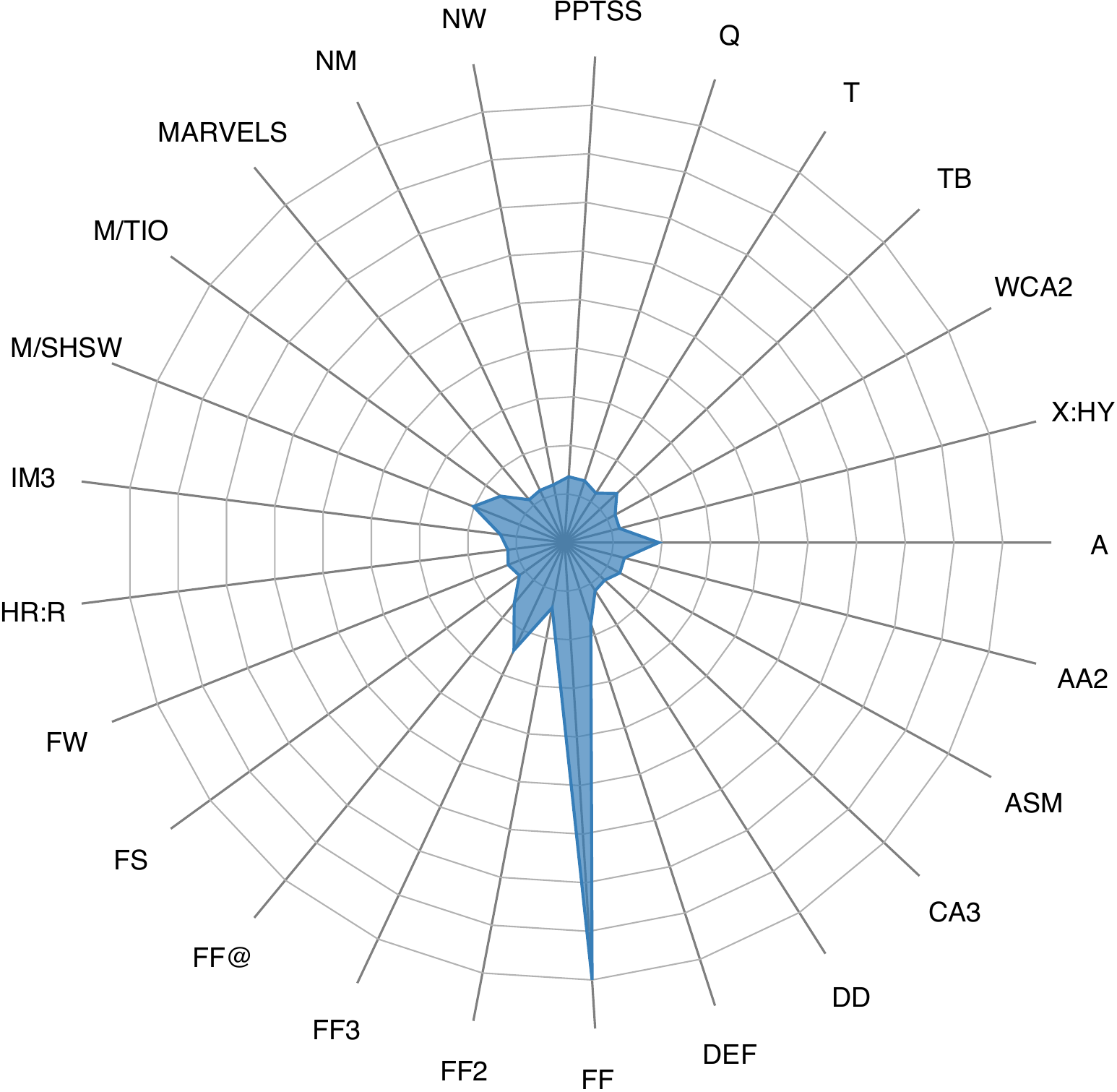}&
		\includegraphics[height=\marvellength]{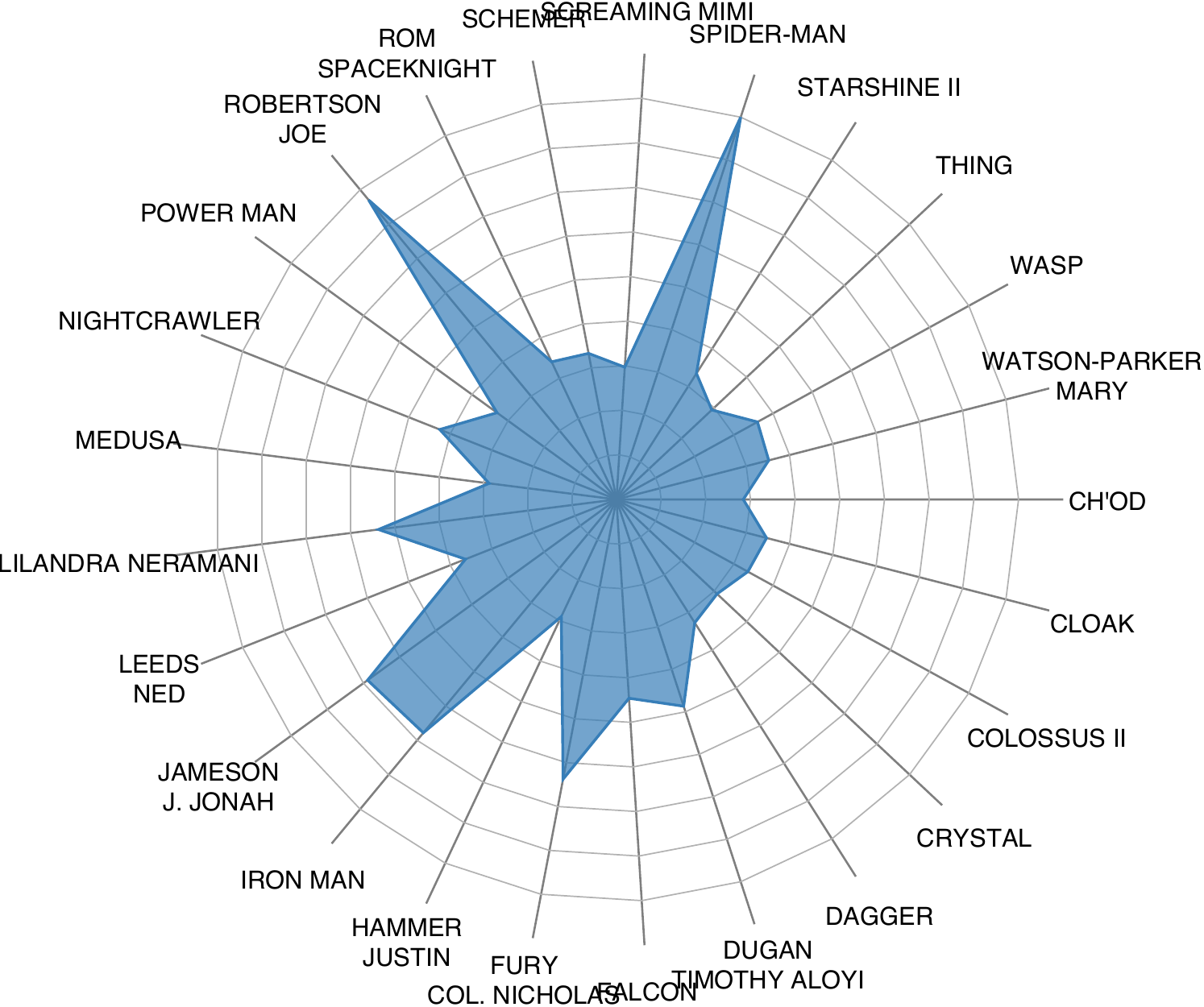}&
		\includegraphics[height=\marvellength]{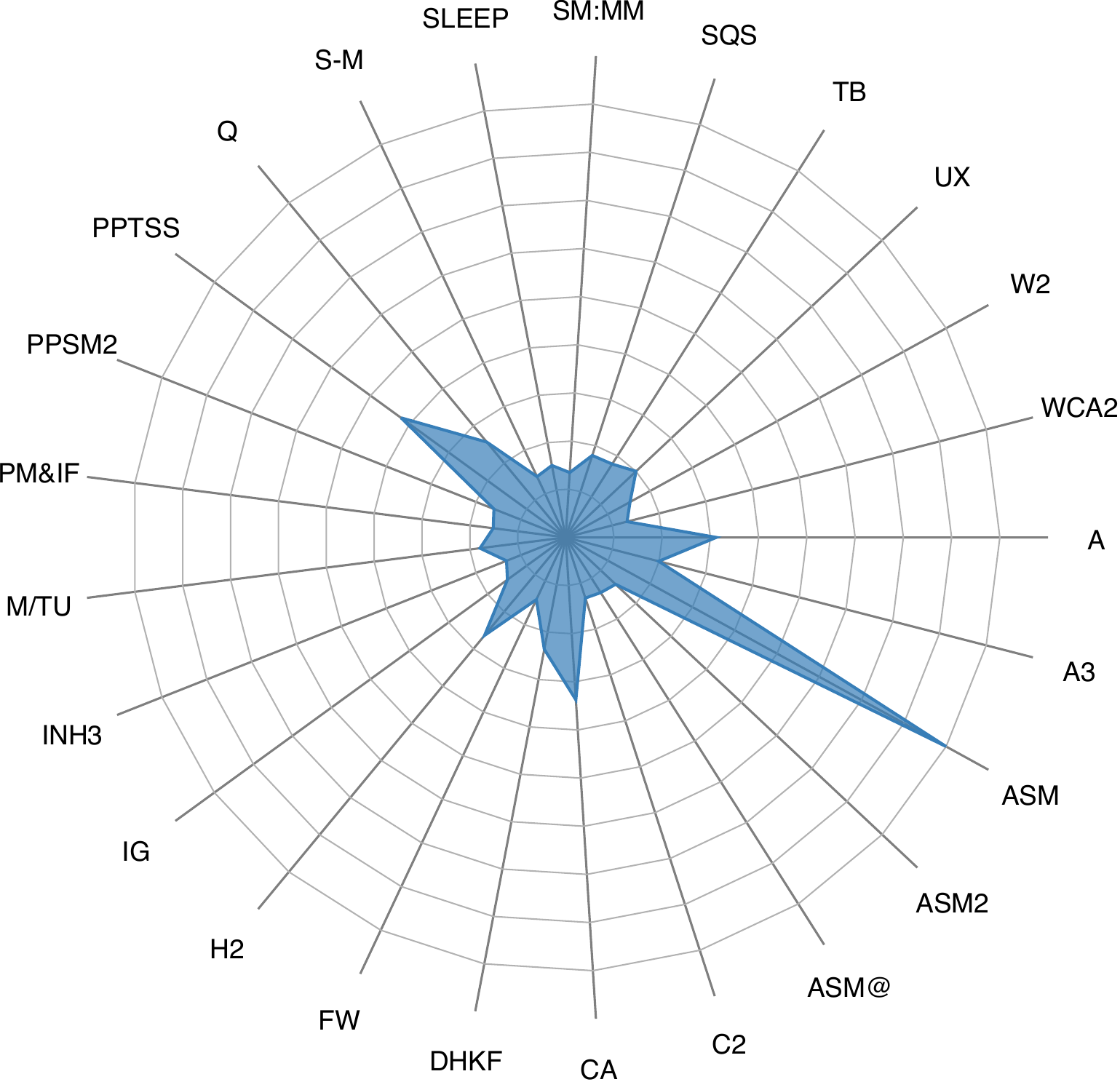} \\
		
		\begin{sideways}Active set\end{sideways}&
		\includegraphics[height=\marvellength]{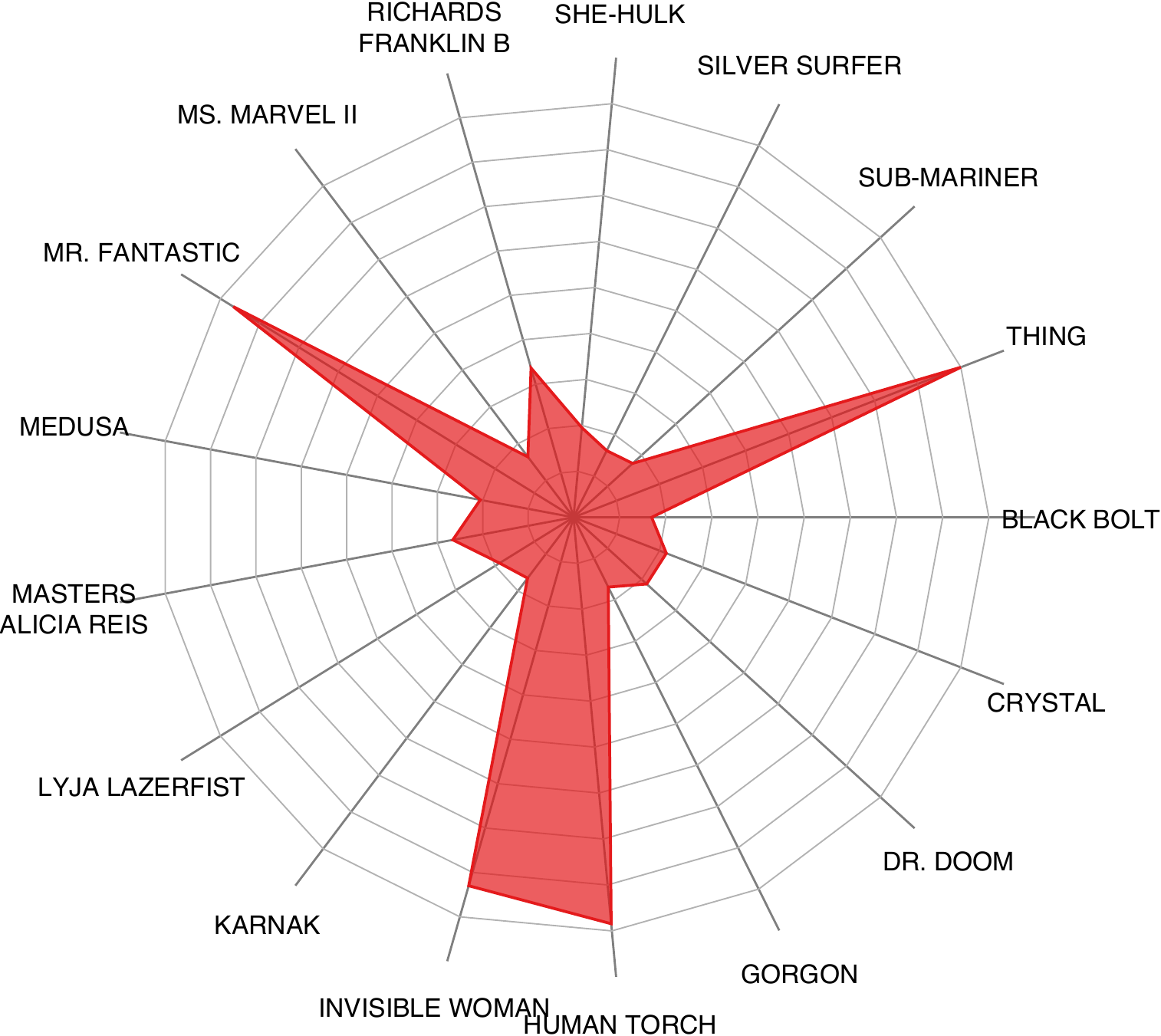}&
		\includegraphics[height=\marvellength]{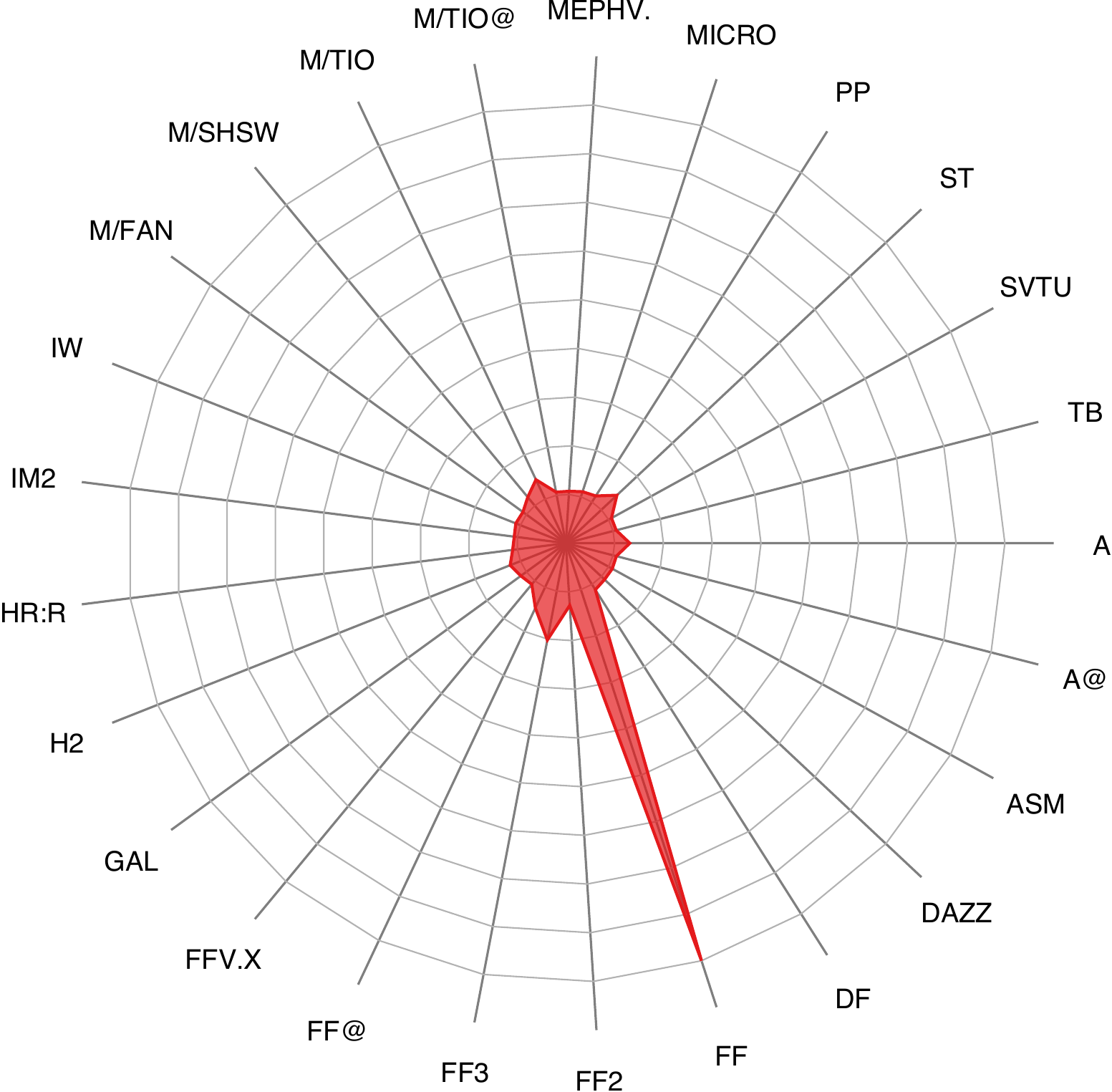}&
		\includegraphics[height=\marvellength]{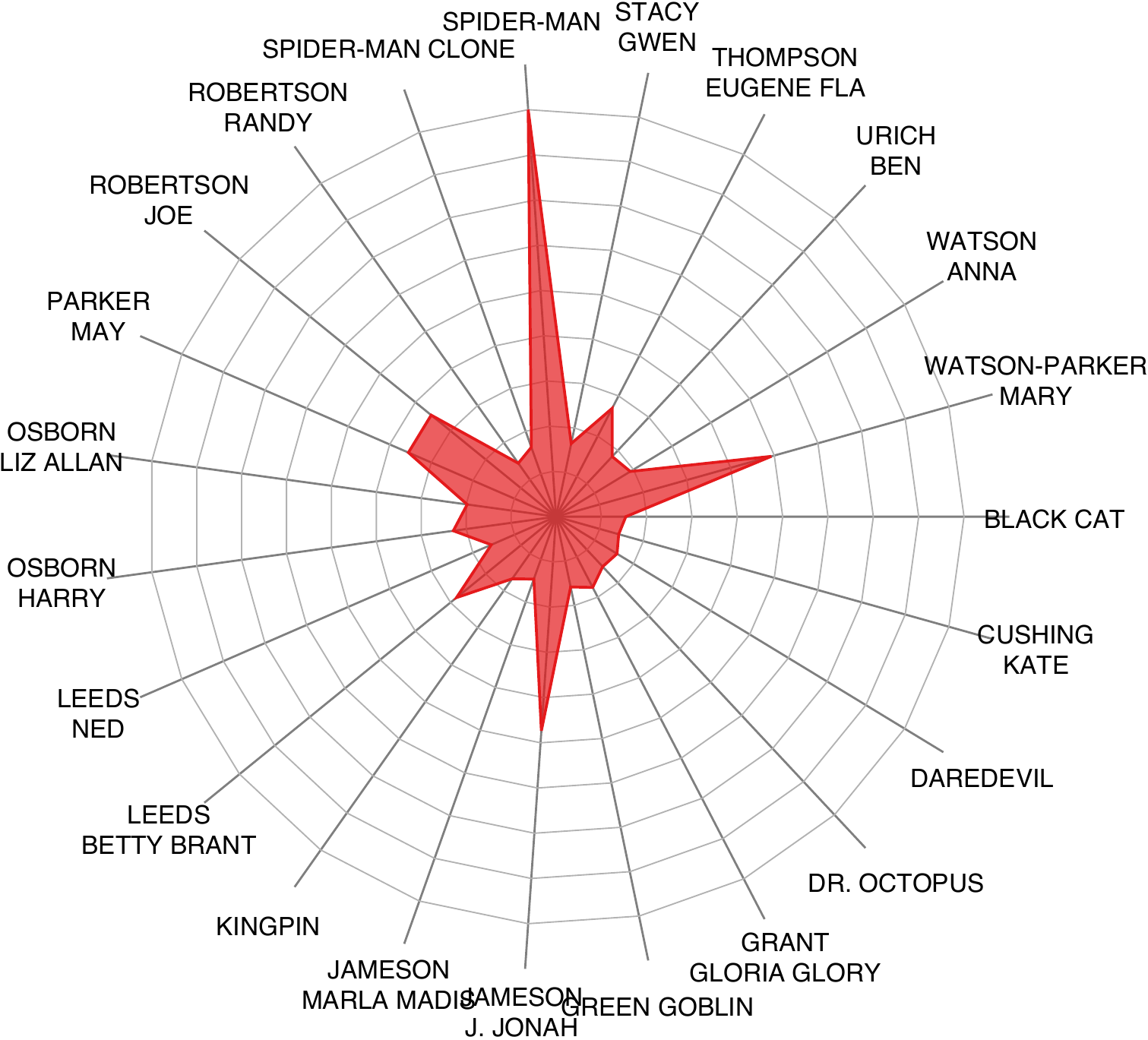}&
		\includegraphics[height=\marvellength]{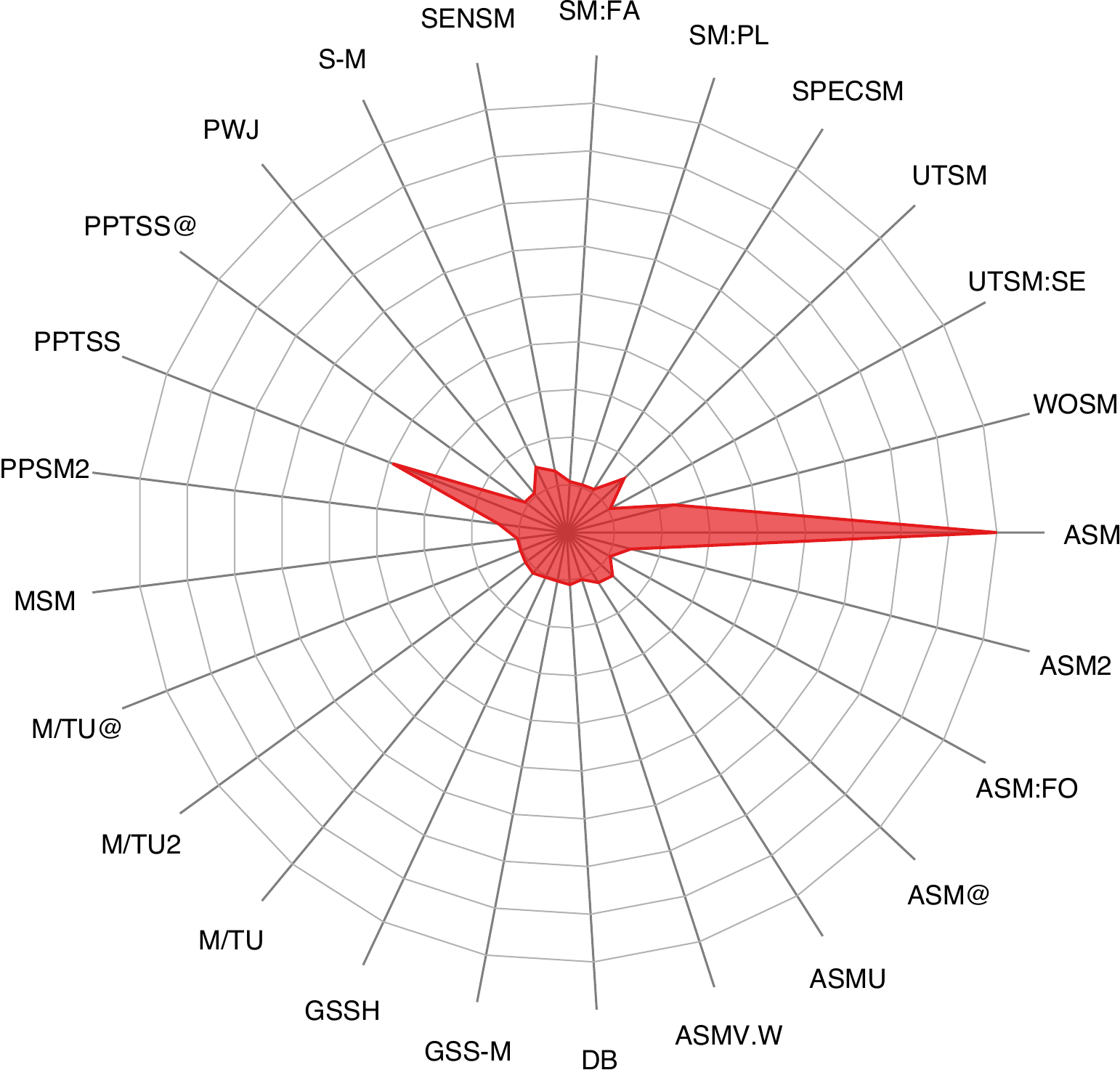}\\
		
		\begin{sideways}Active set - SC\end{sideways}&
		\includegraphics[height=\marvellength]{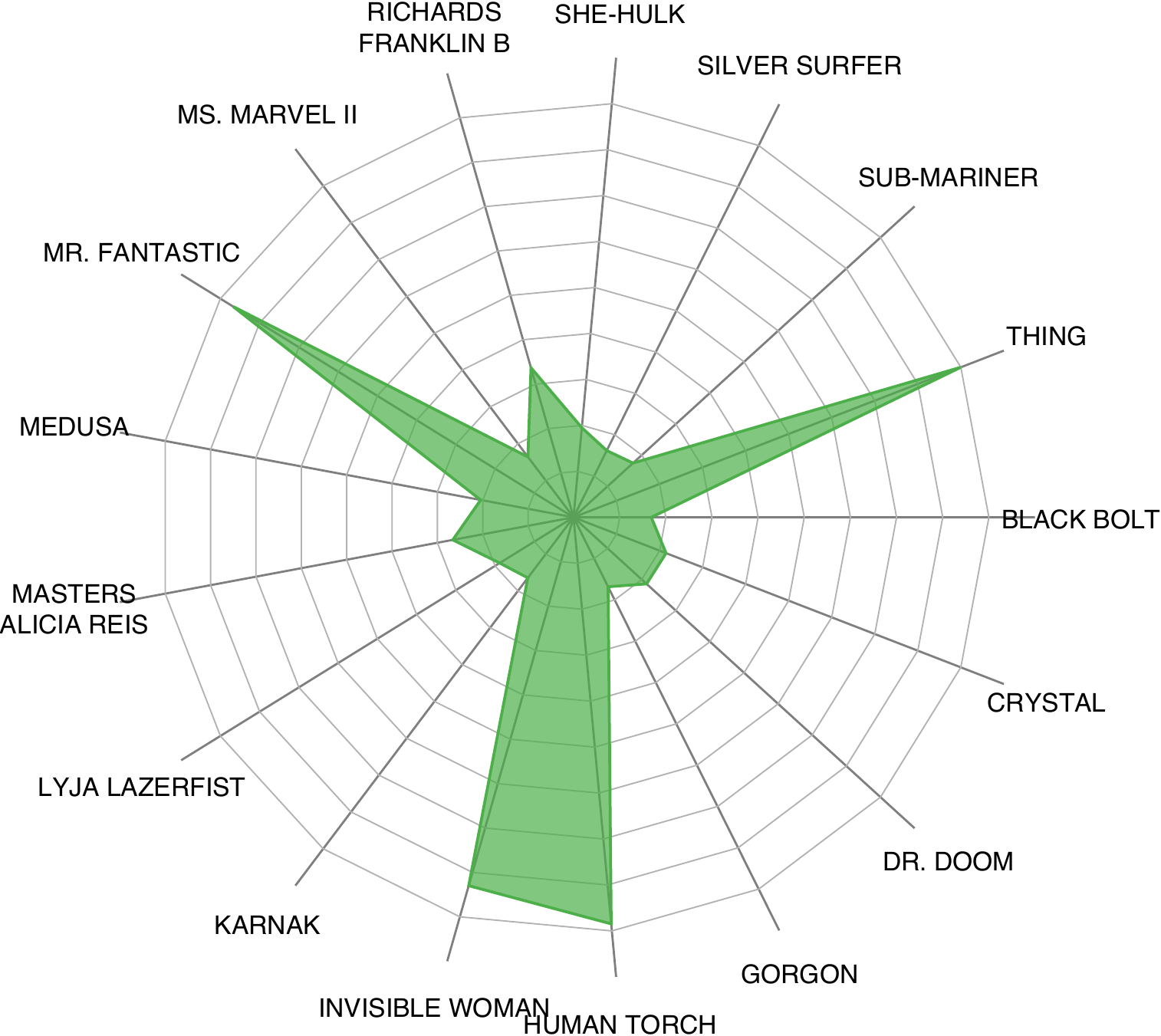}&
		\includegraphics[height=\marvellength]{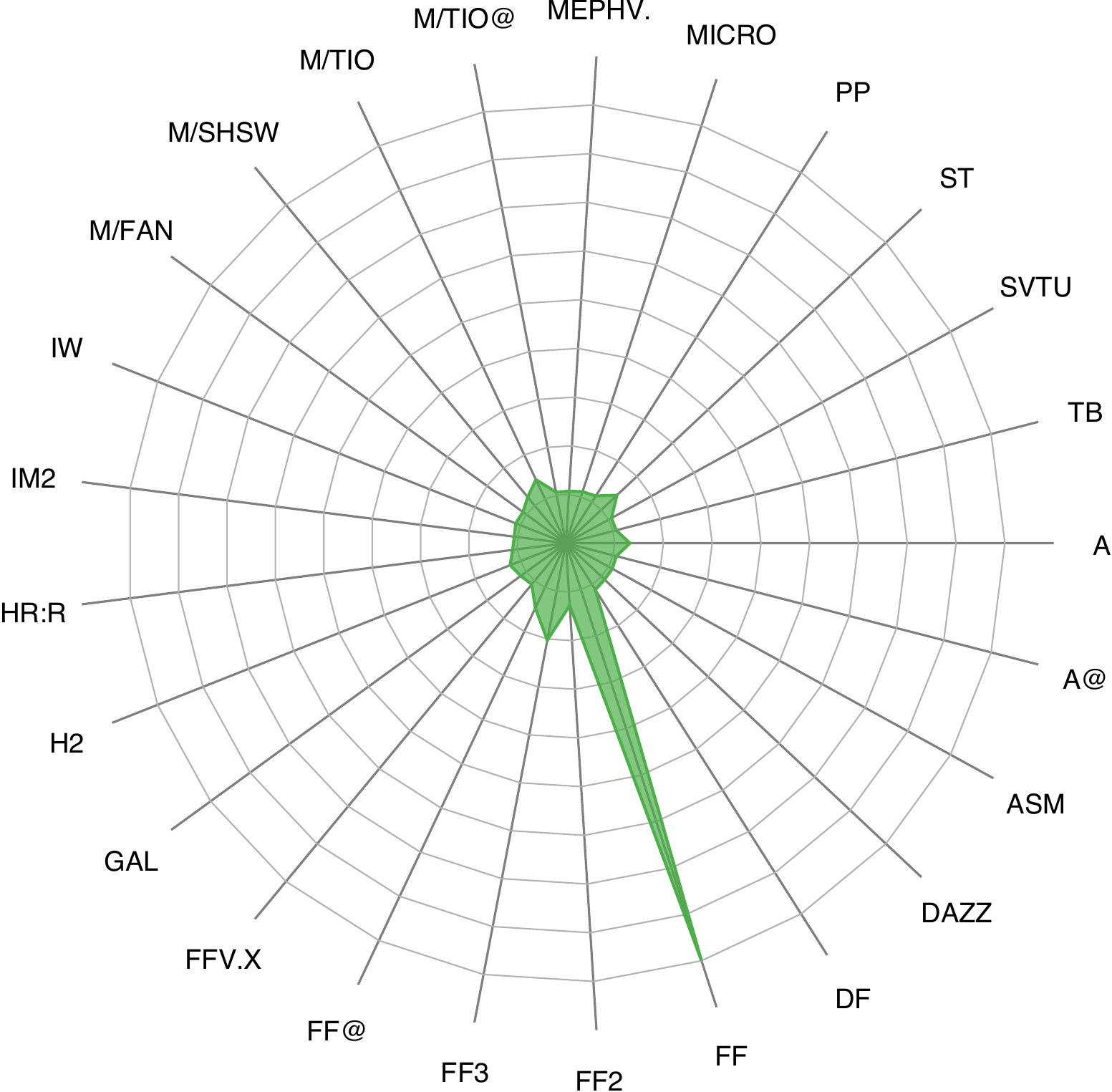}&
		\includegraphics[height=\marvellength]{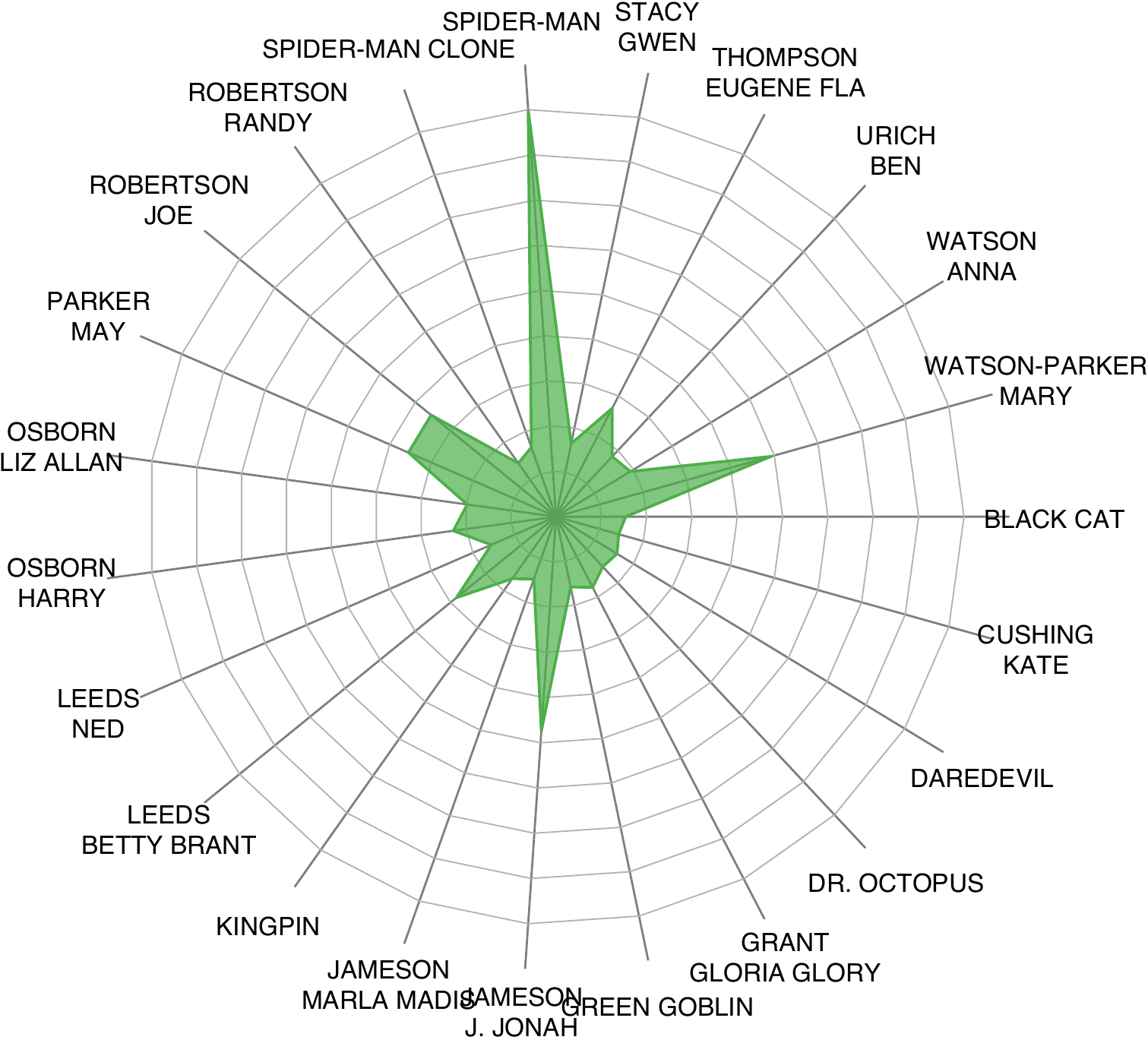}&
		\includegraphics[height=\marvellength]{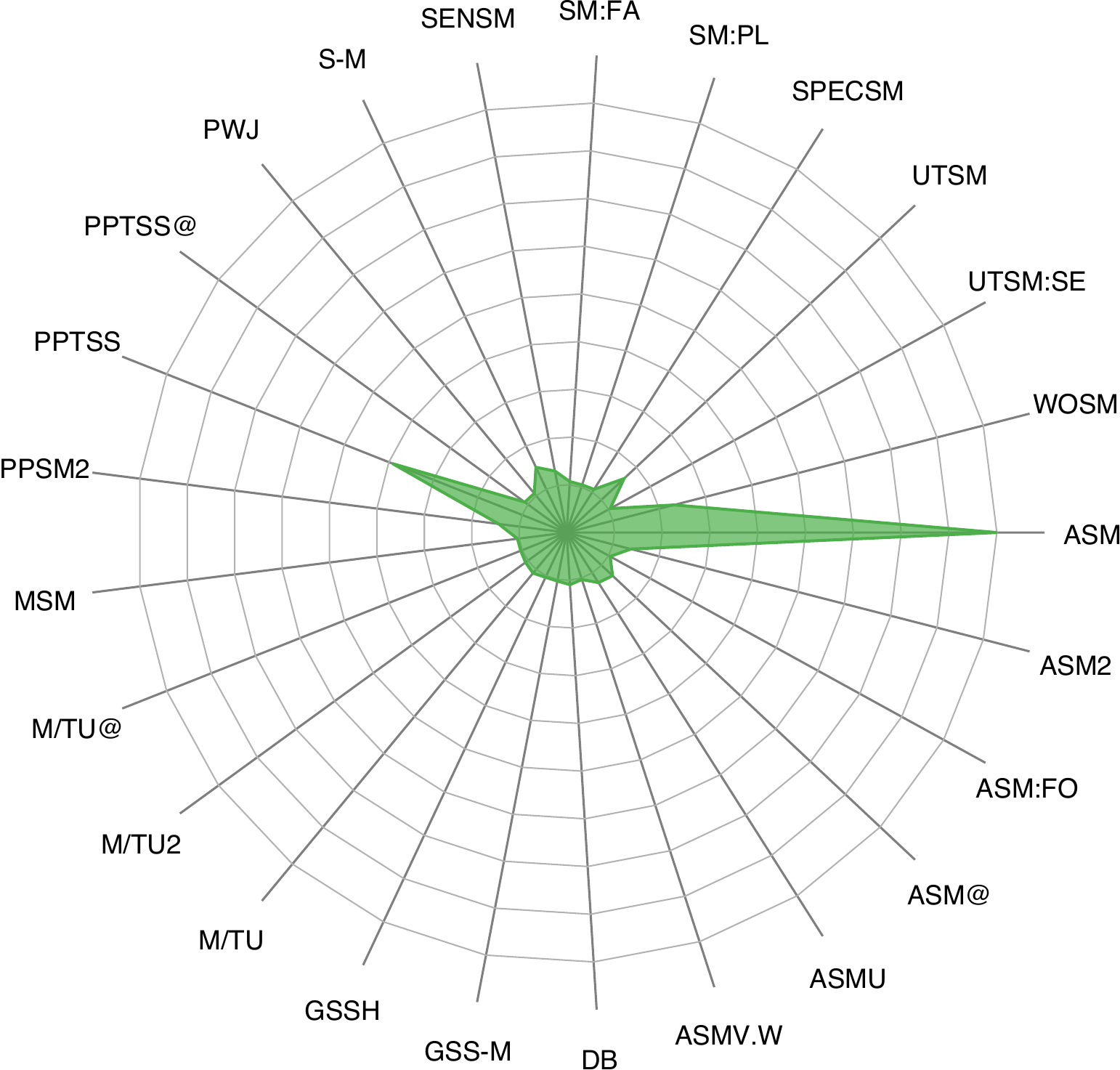}\\
						
	\end{tabu}
	\end{footnotesize}	
		
	\caption{\textbf{Biclustering the Marvel Universe collaboration network} (\url{http://www.chronologyproject.com/}). The network links Marvel characters and the Marvel comic books in which they appear, and exhibits most characteristics of “real-life” collaboration networks~\cite{Alberich2002}. It can be represented as an $m \times n$ matrix, where $m=6445$ and $n=12850 $ are the number of characters and comics, respectively. We bicluster this matrix using NMF with $r=10$, aiming at obtaining 10 very representative groups of characters appearing jointly in different comic books.
	The $i$th bicluster ($i = 1 \dots 10$) is formed by the $i$th column of $\mat{X}$ and the $i$th row of $\mat{Y}$ (small entries were set to zero, as explained in \cref{sec:results_nmf}). The radar plots represent the coefficients of these vectors. We show two biclusters that we identify with characters from the Fantastic Four (first two columns of the figure) and the Spider-Man (last two columns of the figure) comics. Active set NMF correctly identifies that Mr Fantastic, The Thing, the Invisible Woman, and the Human Torch are the four most recurring characters in the ``Fantastic Four'' (FF) series. Similarly, active set NMF correctly identifies that Spider-Man/Peter Parker, Mary Jane Watson-Parker (Peter Parker's wife), and Jonah Jameson (Peter Parker's boss) are the most recurring characters in the ``Amazing Spider-Man'' (ASM) and ``Peter Parker, The Spectacular Spider-Man'' (PPTSS) series. It is clear that the biclusters recovered using structured random compression (SC) are very close to the biclusters found with no compression; contrarily, Gaussian compression (GC) significantly affects the biclustering result. All $10$ biclusters can be found at \url{http://www.marianotepper.com.ar/research/cnmf}.}
	\label{fig:marvel}
\end{figure*}

To summarize, the overall observation is that structured compression brings additional speed to NMF methods without introducing significant errors. On the other hand, Gaussian compression seems to come at the cost of higher reconstruction errors and is not consistently faster than structured compression.

\subsection{Separable NMF}
\label{sec:results_snmf}

We implemented our SNMF algorithms in Python, using the dask and into libraries\footnote{\url{http://dask.readthedocs.org/}, \url{http://into.readthedocs.org/}} to perform out-of-core matrix computations (i.e., without fully loading the involved matrices in main memory). A byproduct of this implementation choice is that we can compute SNMF on very large matrices on a regular laptop, without having to resort to a cluster. To the best of our knowledge, our TSQR implementation is the first publicly available one that runs on any regular laptop using out-of-core computations. 

We perform all of our comparisons with the SNMF algorithm using the QR decomposition~\cite{Benson2014}, analyzed in \cref{sec:snmf}. We use SPA~\cite{Araujo2001,Gillis2014a}, and XRAY~\cite{Kumar2013} as the column selection algorithms. Throughout this section, we simply use compression to refer to structured compression.
In all tests, we set $w = 0$ and $r_\textsc{ov} = 10$ in the compression algorithm in \cref{algo:compression}; we further adjust the value of $r_\textsc{ov}$ so that $r + r_\textsc{ov} = \min(\max(20, r + r_\textsc{ov}), n)$.

We first present results on synthetic matrices in \cref{fig:snmf_synthetic}. We produced different matrices of fixed size by varying their rank, see \cref{fig:snmf_synthetic_rank}. In general, we aim at explaining the data matrix with a small fraction of its columns. The proposed compression method for SNMF is faster when fewer factors are needed to explain the data. On the other hand, QR-based methods have always the same (high) computing time, no matter how simple is the structure of the data. We also investigated how much faster is the proposed method with respect to QR-based approaches. We generated $m \times n$ input matrices, where $m$ is fixed and $n$ varies; we then extract $n/10$ columns. Remember that QR-based approaches solve an $n \times n$ version of \cref{eq:snmf_right_factor_QR}, while the proposed compressed approach solves an $(r + r_\textsc{ov}) \times n$ version. This difference is reflected almost exactly in the speedup that we observe in \cref{fig:snmf_synthetic_cols}: about an order of magnitude is gained with the proposed scheme.

\begin{figure*}
	\begin{subfigure}[t]{\columnwidth}
		\includegraphics[width=\columnwidth]{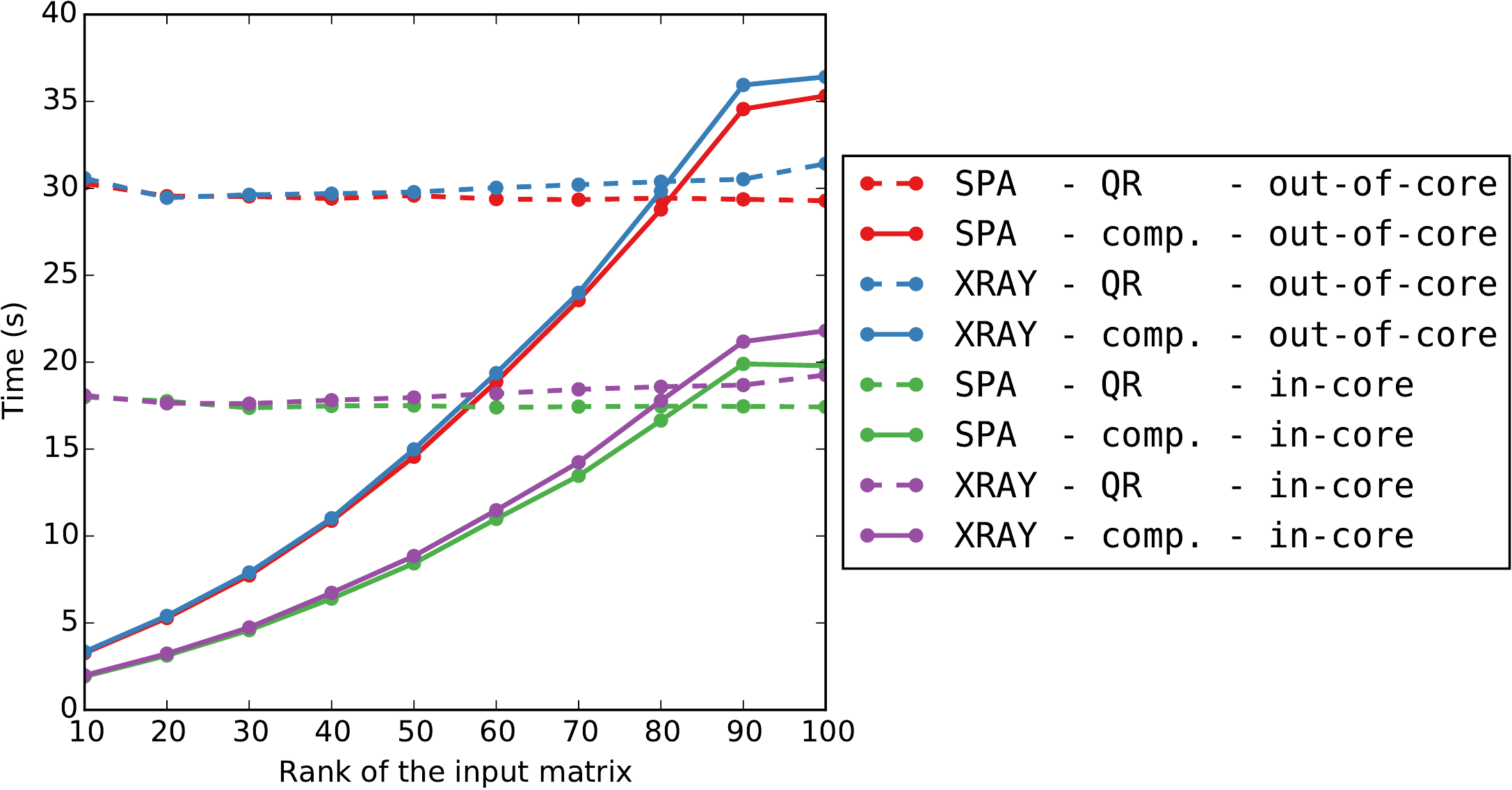}
		
		\caption{We extract $r$ columns, where $r$ is the rank of the $10^6 \times 100$ input matrix. As expected, the computing time of the QR-based methods does not change with the number of extracted columns. On the other hand, compressed methods are faster when the rank of the input matrix is low compared to its size. In this case, out-of-core methods appear slower than in-core ones (slightly above $2 \times$). We use an oversampling factor $r_\textsc{ov} = 10$ for compression, which explains the flattening of the compressed curves towards their end.}
		\label{fig:snmf_synthetic_rank}
	\end{subfigure}
	\hfill	
	\begin{subfigure}[t]{\columnwidth}
		\includegraphics[width=\columnwidth]{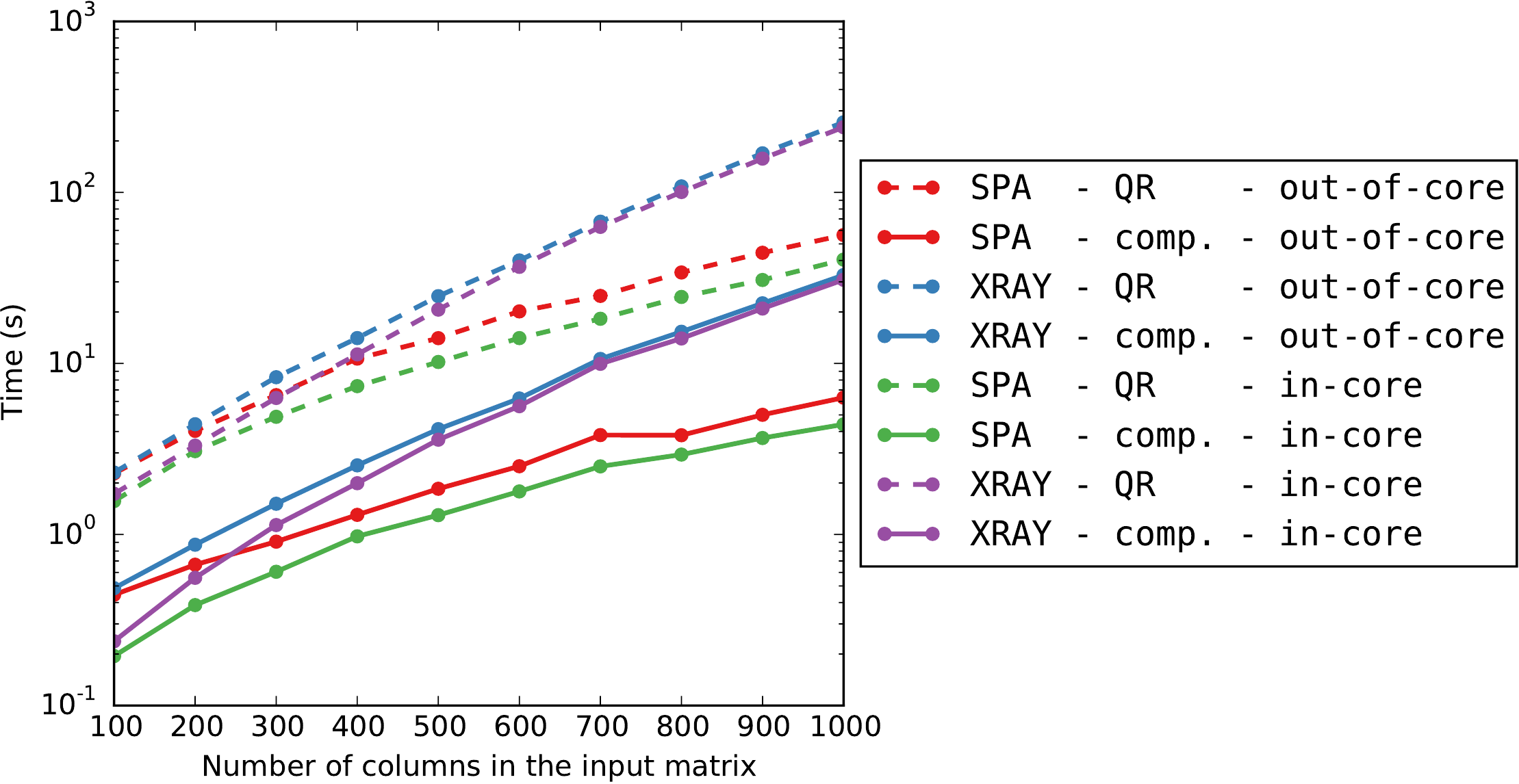}
		
		\caption{We extract $n/10$ columns from a $10^5 \times n$ input matrix. Note that the speedup of compressed versus QR-based methods (approx.~$10 \times$) is straightforwardly explained by the fixed ratio between the rank and the number of columns. In this case, no significant speed difference is noticeable when comparing in-core with out-of-core methods.}
		\label{fig:snmf_synthetic_cols}
	\end{subfigure}
		
	\caption{\textbf{Performance of different SNMF algorithms on synthetic matrices}. We generate the input matrix $\mat{A} = \mat{X}_\textsc{gt} \mat{Y}_\textsc{gt}$, where $\mat{X}_\textsc{gt} \in \Real^{m \times r}$ and $\mat{Y}_\textsc{gt} \in \Real^{r \times n}$ have normally distributed entries ($r$ and $n$ take different values in subfigures~\subref{fig:snmf_synthetic_rank} and~\subref{fig:snmf_synthetic_cols}). All algorithms select the same set of columns, thus producing equal errors.}
	\label{fig:snmf_synthetic}
\end{figure*}

In \cref{fig:climate_snmf} we analyze the same dataset as in \cref{fig:climate}. Interestingly, a similar conclusion is reached using SNMF and NMF. The data is well explained by the same two factors (in this case, two extreme columns). Notice that the analyzed matrix is fat and the QR-based approach provides no speedup, i.e., $\mat{R} \in \Real^{m \times n}$ in \cref{eq:snmf_right_factor_QR}. On the other hand, the proposed approach produces a smaller problem independently of the input matrix's shape. Quantitatively, in this example, compressed SNMF is two orders of magnitude faster than the QR-based SNMF.

\begin{figure*}
	\centering
	
	\begin{subfigure}{.25\textwidth}
		\centering
		\includegraphics[width=\textwidth]{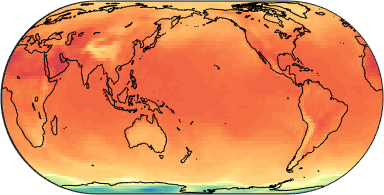}\\[4pt]
		\includegraphics[width=\textwidth]{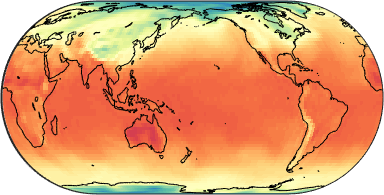}
		
		\caption{Columns extracted with SPA-comp. When $r=2$, the extreme columns look similar to the ones found with traditional NMF, see \cref{fig:climate}.}
	\end{subfigure}
	\hfill
	\begin{subfigure}{.735\textwidth}
		\centering
		\includegraphics[height=.32\textwidth]{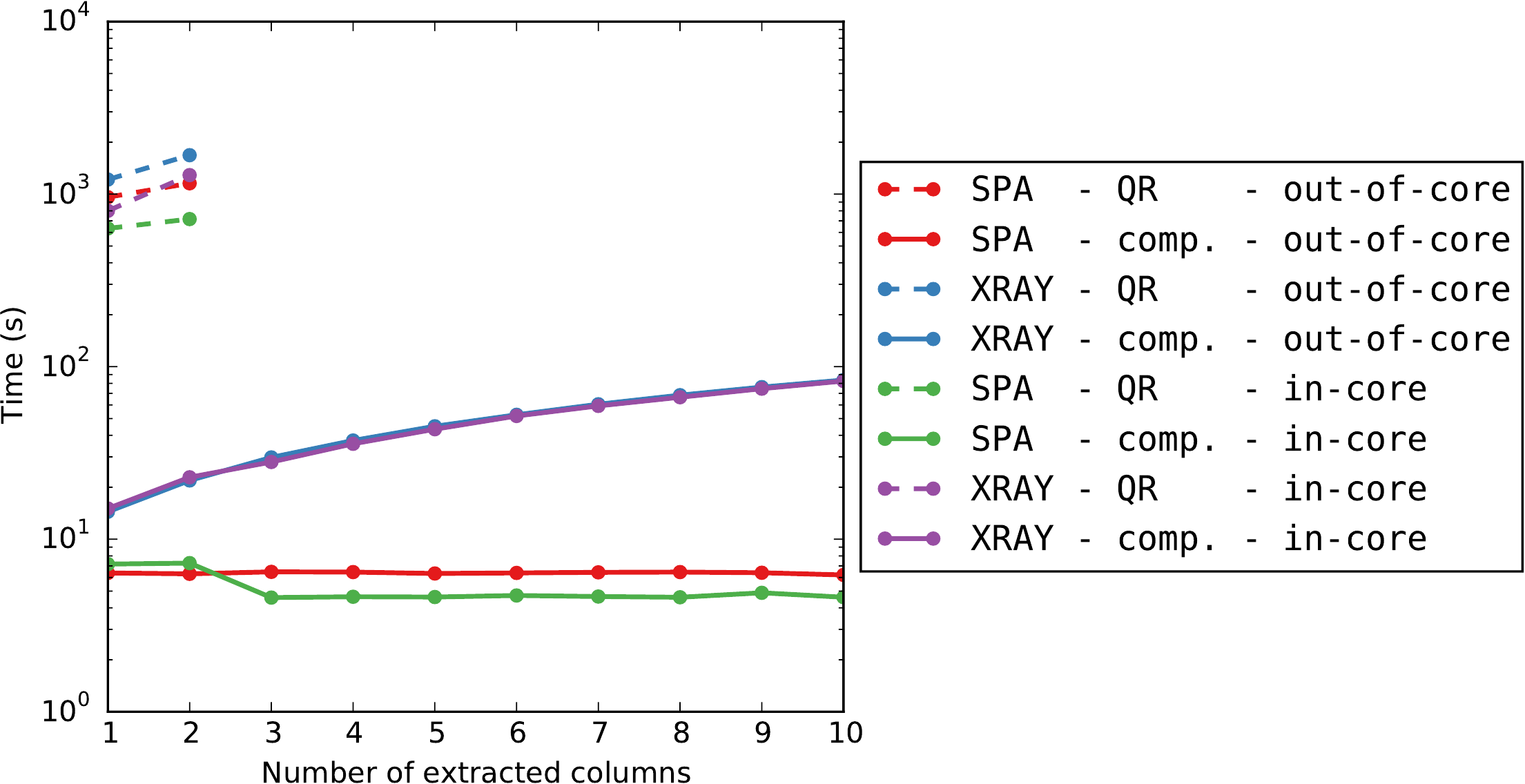}%
		\hspace{.1cm}%
		\includegraphics[height=.32\textwidth]{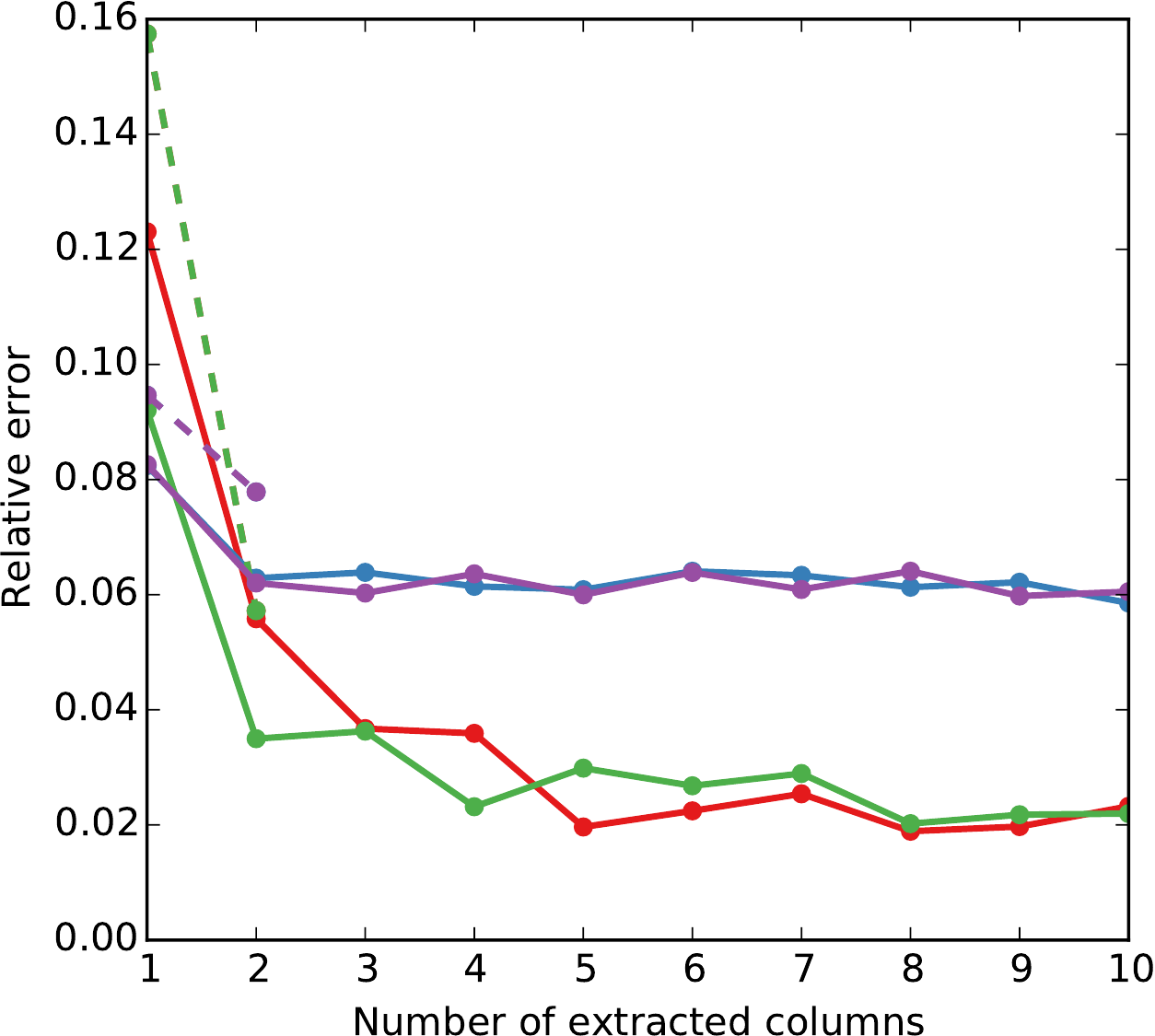}
		
		\caption{With the uncompressed methods, extracting columns becomes extremely slow (about two orders of magnitude slower) than with the compressed methods. Since compressed SNMF is faster with 10 columns than QR-based SNMF with two columns, we simply stopped the computation of the latter after $2$ columns. Notice that these QR-based methods are explicitly designed to be faster for tall-and-skinny matrices, but end-up being extremely slow for fat matrices. The proposed compressed SNMF is also very fast for fat matrices.}
	\end{subfigure}
	
	\caption{\textbf{SNMF on gridded climate data}. Same dataset as in \cref{fig:climate}. The data form a fat $10512 \times 23742$ matrix. We study the performance of SNMF in terms of computing speed and relative error as the number $r$ of columns changes. As with NMF, the data is well explained with only two factors by observing the decay in the reconstruction error. SPA with compression seems not to increase its computing time as the number of extracted columns increases; this is due to forcing the compression algorithm to produce at least $20$ rows, the subsequent column extraction in SPA is extremely efficient. Notice that SPA with compression is about four times faster than NMF using the active set method with compression, see \cref{fig:climate}.}
	\label{fig:climate_snmf}
\end{figure*}

Our last example consists on an application for selecting representative frames from videos. We first examine a short clip ($5$ seconds long, $120$ frames) of the open-source movie ``Elephants Dream''  at a resolution of 360p ($640 \times 360$). In \cref{tab:elephantsDreams} we show a summary of the comparisons performed with this video. An example of the frames extracted by SPA with compression is shown in \cref{fig:elephantsDreams}.

Our first observation is that the proposed compressed SNMF is at least an order of magnitude faster than the QR-based variant. Second, since the matrix built from video is not truly low-rank, projecting the matrix into a low-rank subspace by means of compression seems to yield better results than when using the QR decomposition. Intuitively, compression eliminates some variability in the data in such a way that it can be better approximated by SNMF.

Although not strictly comparable, because it does not impose nonnegativity constraints, we included in our comparisons the method for extracting representative elements from~\cite{Elhamifar2012}. As discussed in \cref{sec:snmf}, this method's formulation does not scale gracefully with large input matrices. A fact that is easily reflected in the slow running time, even for a relatively small example.

\begin{table}
	\caption{\textbf{Extracting representative frames from a video} For details about the experiment setup, see \cref{fig:elephantsDreams}. We are considering a (relatively small) $691,200 \times 120$ matrix to be able to compare the performance of in-core and out-of-core methods and with ESV~\cite{Elhamifar2012}, which is not fit for large scale matrices. The proposed compression scheme for SNMF (SPA-COMP) greatly improves speed with no detriment for the reconstruction error. Notice that since the matrix is not actually low-rank (it is a video), enforcing the projection onto a subspace helps in finding a better solution (SPA-COMP versus SPA-QR).}
	\label{tab:elephantsDreams}
	
	\centering
	\begin{threeparttable}[b]
	\begin{tabular}{lld{2.0}d{2.2}d{1.4}}
		\toprule
		\multicolumn{1}{c}{Methods} & \multicolumn{1}{c}{Comp. model} & \multicolumn{1}{c}{$r$} & \multicolumn{1}{c}{Time (s)} & \multicolumn{1}{c}{Rel. error}\\
		\midrule
		SPA-COMP & in-core & 6 & 2.28 & 0.4240\\
		SPA-COMP & out-of-core & 6 & 4.75 & 0.4293\\
		SPA-QR & in-core & 6 & 18.76 & 0.5446\\[3pt]
		
		SPA-COMP & in-core & 9 & 2.31 & 0.3626\\
		SPA-COMP & out-of-core & 9 & 4.59 & 0.3610\\
		SPA-QR & in-core & 9 & 19.08 & 0.4453\\
		ESV~\cite{Elhamifar2012} ($\alpha=2$)\tnote{1} & in-core & 9 & 57.38 & 0.3751\tnote{2}\\[3pt]
		
		SPA-COMP & in-core & 15 & 2.65 & 0.3068\\
		SPA-COMP & out-of-core & 15 & 5.50 & 0.3047\\
		SPA-QR & in-core & 15 & 19.93 & 0.4011\\
		ESV~\cite{Elhamifar2012} ($\alpha=50$)\tnote{1} & in-core & 15 & 68.05 & 0.1358\tnote{2}\\
		\bottomrule
	\end{tabular}
	\begin{tablenotes}
		\item [1] $\alpha$ is a regularization parameter that (indirectly) controls the number of representatives $r$.
		\item [2] The errors are not directly comparable since this formulation does not impose nonnegativity.
	\end{tablenotes}
	\end{threeparttable}
\end{table}

\begin{figure*}
	\centering
	
	\tabulinesep=1pt
	\begin{tabu} to \textwidth{ @{\hspace{1pt}} *{10}{ X[c,m] @{\hspace{1pt}}} }
		\includegraphics[width=\linewidth]{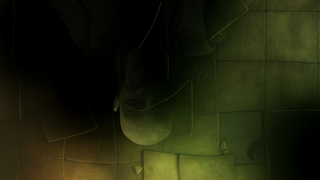} &
		\includegraphics[width=\linewidth]{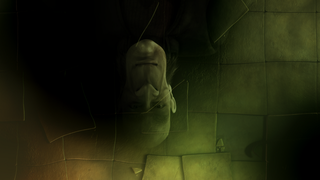} &
		\includegraphics[width=\linewidth]{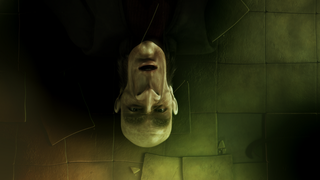} &
		\includegraphics[width=\linewidth]{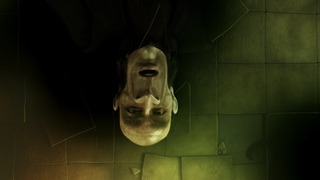} &
		\includegraphics[width=\linewidth]{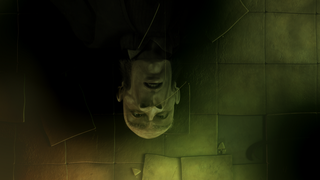} &
		\includegraphics[width=\linewidth]{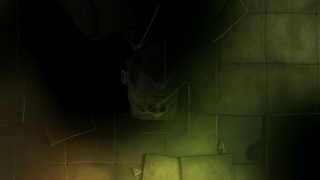} &
		\includegraphics[width=\linewidth]{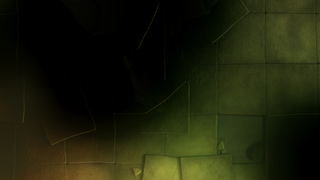} &
		\includegraphics[width=\linewidth]{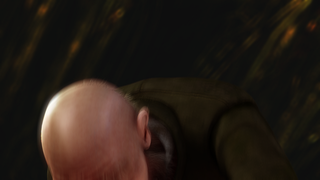} &
		\includegraphics[width=\linewidth]{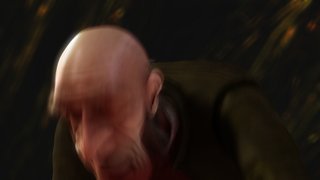} &
		\includegraphics[width=\linewidth]{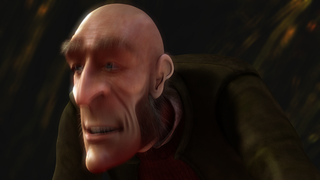} \\

		\includegraphics[width=\linewidth]{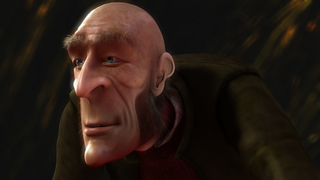} &
		\includegraphics[width=\linewidth]{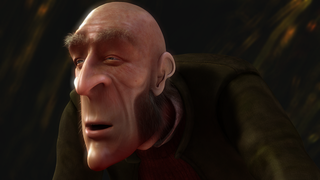} &
		\includegraphics[width=\linewidth]{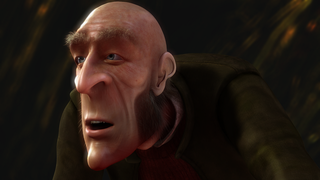} &
		\includegraphics[width=\linewidth]{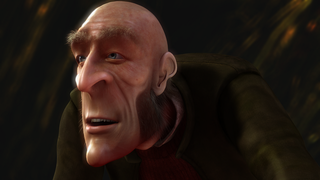} &
		\includegraphics[width=\linewidth]{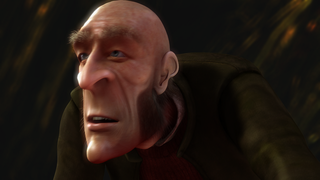} &
		\includegraphics[width=\linewidth]{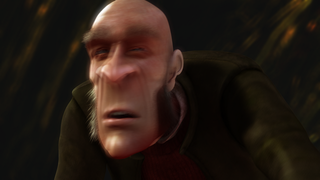} &
		\includegraphics[width=\linewidth]{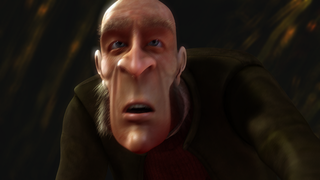} &
		\includegraphics[width=\linewidth]{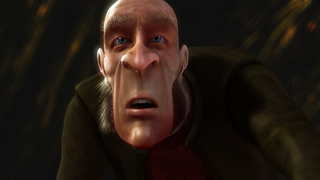} &
		\includegraphics[width=\linewidth]{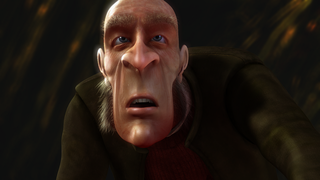} &
		\includegraphics[width=\linewidth]{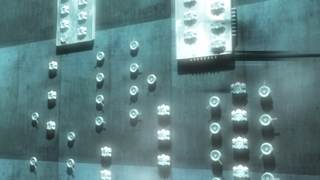} \\

		\includegraphics[width=\linewidth]{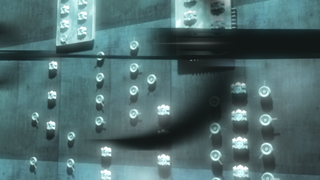} &
		\includegraphics[width=\linewidth]{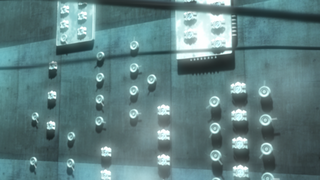} &
		\includegraphics[width=\linewidth]{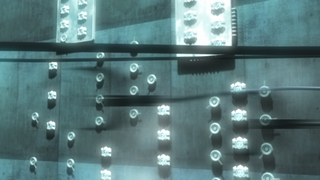} &
		\includegraphics[width=\linewidth]{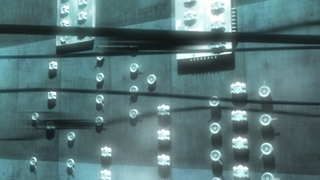} &
		\includegraphics[width=\linewidth]{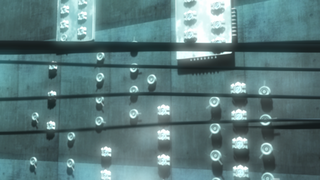} &
		\includegraphics[width=\linewidth]{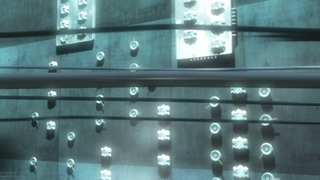} &
		\includegraphics[width=\linewidth]{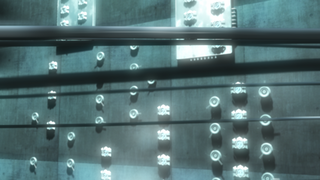} &
		\includegraphics[width=\linewidth]{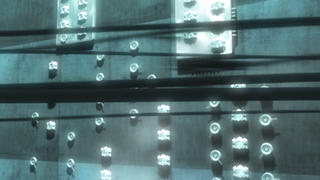} &
		\includegraphics[width=\linewidth]{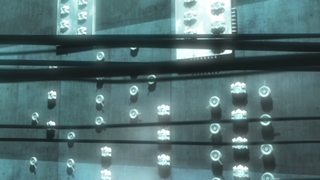} &
		\includegraphics[width=\linewidth]{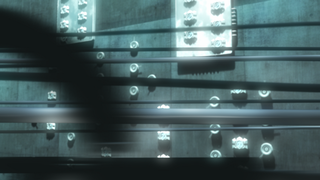} \\

		\includegraphics[width=\linewidth]{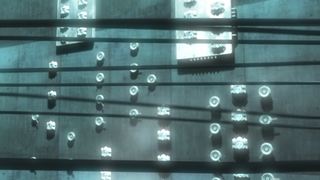} &
		\includegraphics[width=\linewidth]{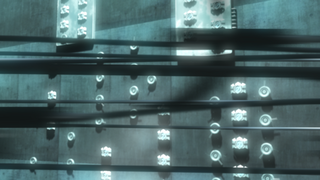} &
		\includegraphics[width=\linewidth]{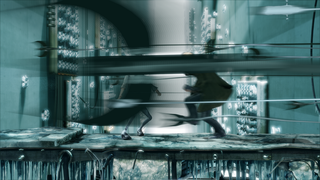} &
		\includegraphics[width=\linewidth]{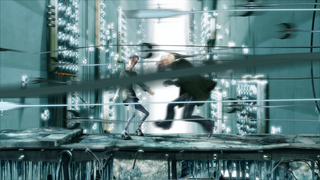} &
		\includegraphics[width=\linewidth]{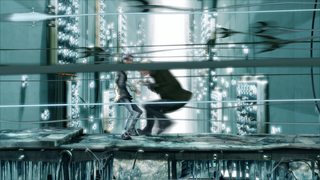} &
		\includegraphics[width=\linewidth]{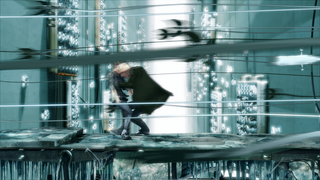} &
		\includegraphics[width=\linewidth]{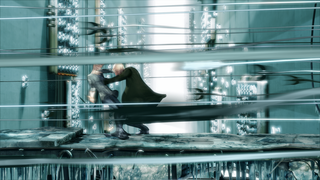} &
		\includegraphics[width=\linewidth]{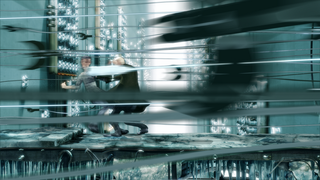} &
		\includegraphics[width=\linewidth]{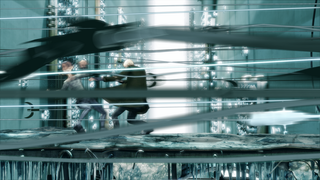} &
		\includegraphics[width=\linewidth]{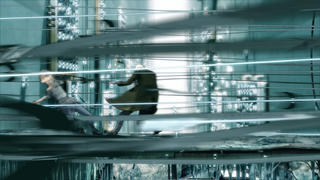} \\

	\end{tabu}	
	
	\begin{tabu} to \textwidth{ @{\hspace{1pt}} m{.5\textwidth} @{\hspace{20pt}} m{.33\textwidth} @{\hspace{1pt}} }
		\includegraphics[width=.5\textwidth]{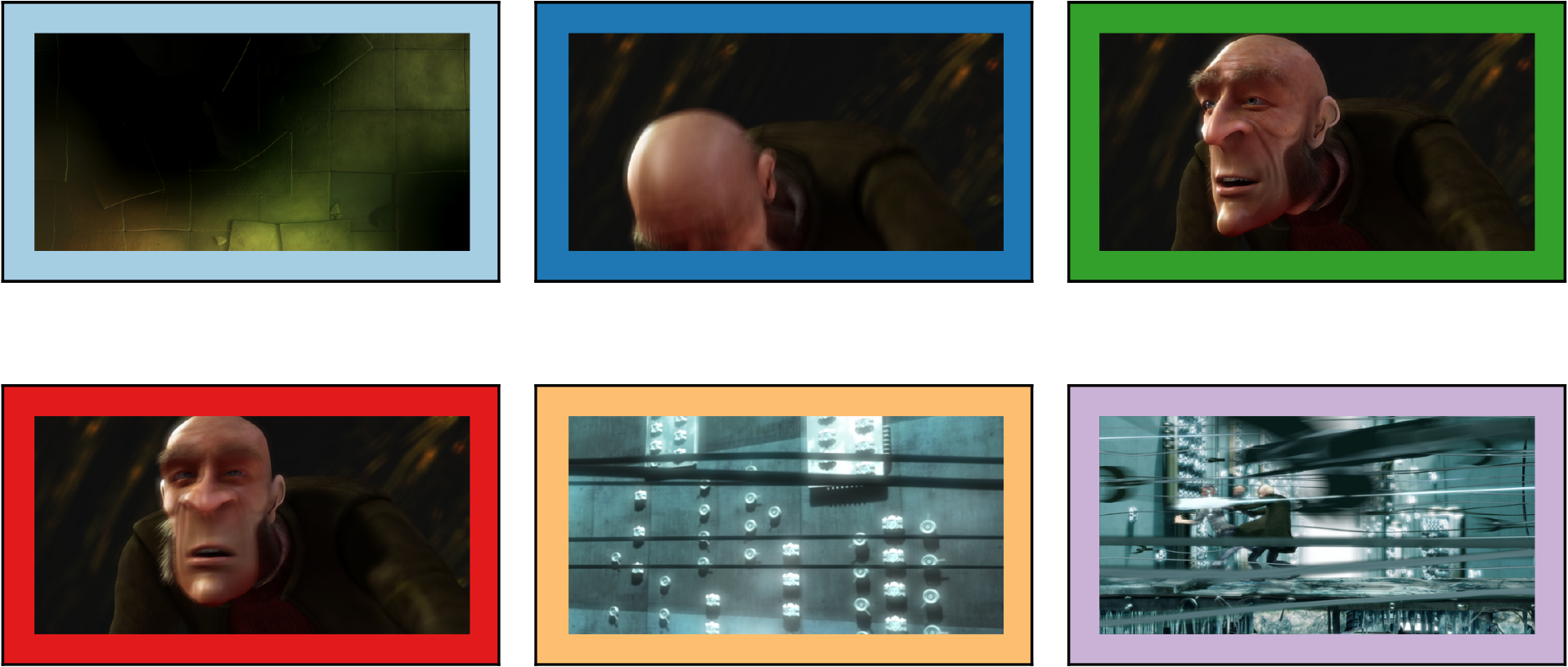} &
		\includegraphics[width=.33\textwidth]{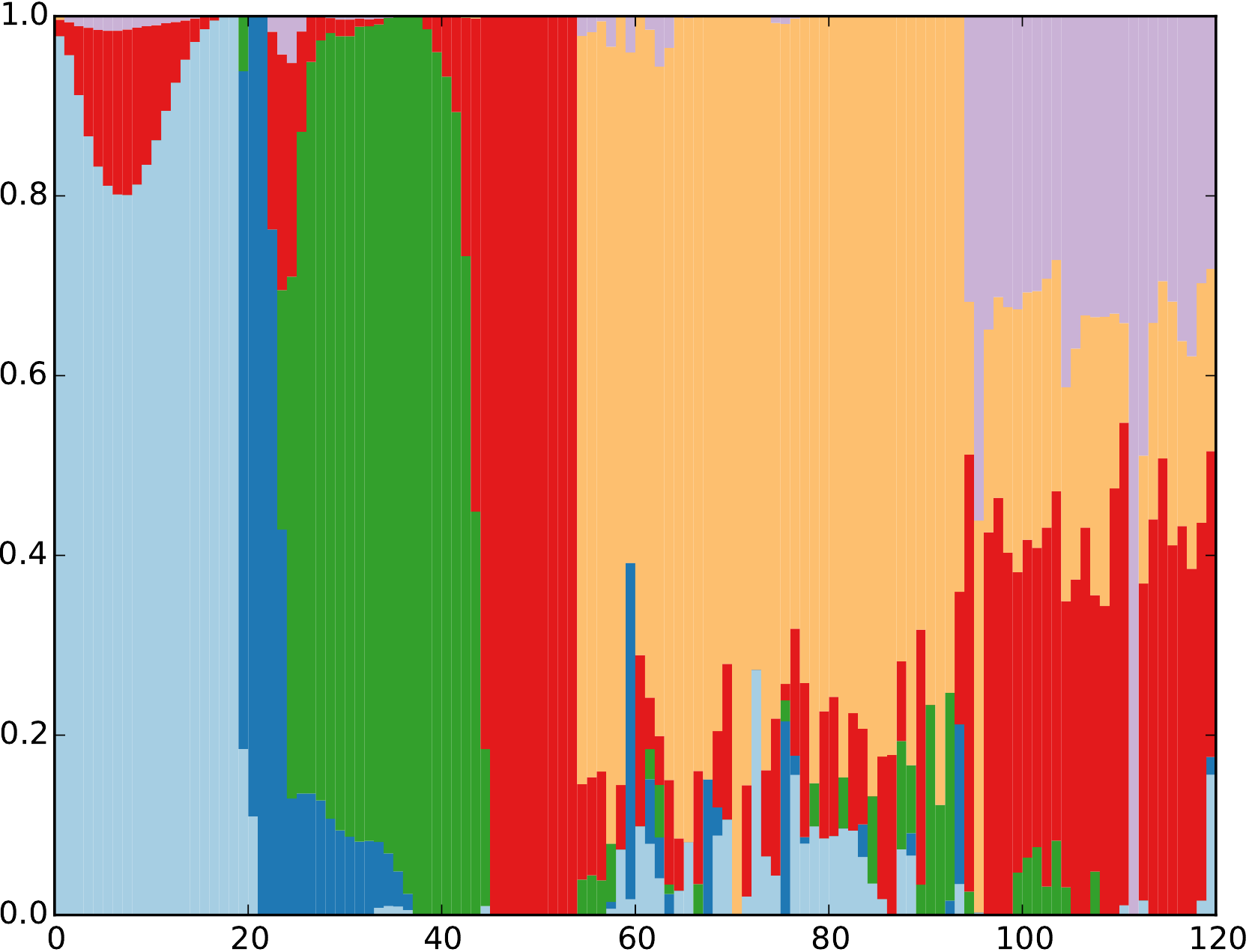}
	\end{tabu}			
	
	\caption{\textbf{Extracting representative frames from a video} (\url{http://www.elephantsdream.org}). The video resolution is $640 \times 360$ pixels and contains $120$ frames ($5$ seconds). On the top block, we display $40$ uniformly sampled frames. We build a $691,200 \times 120$ matrix by vectorizing one frame per column (each frame has 3 color channels), and then use SNMF with compression to extract six representative frames (bottom left). On the bottom right we show the (normalized) columns of the matrix $\mat{H}$ in \cref{snmf_H}, i.e., the reconstruction coefficients. It took $2.18$ seconds to compute the result with relative errors of $0.2714$ and of $0.4240$ with respect to the compressed and the original matrices, respectively.}
	\label{fig:elephantsDreams}
\end{figure*}

\paragraph{Scaling to Big Data}
We also run tests on the complete open-source movie ``Elephants Dream.''\footnote{\url{http://www.elephantsdream.org}} The movie is approximately 11 minutes long (15691 frames). We processed the video at two resolutions, 360p ($640 \times 360$), and 1080p ($1920 \times 1080$), resulting in $691200 \times 15691$ and $6220800 \times 15691$ matrices, respectively. The HDF5 files occupy $43.55$ GB and $391.13$ GB, respectively, not fitting in main memory. Using compressed SPA, we extract 130 representatives (extreme columns) from the video, one every 120 frames (5 seconds). At 360p we obtained a relative error of $0.2941$ in $1891$ seconds (about 32 minutes). At 1080p, we obtained a relative error of $0.2676$ in $20776$ seconds (about 5:46 hours) processing both matrices on a laptop with 16GB of memory.

\section{Conclusions}
\label{sec:conclusions}

In this work we proposed to use structured random projections for NMF and SNMF. For NMF, we presented formulations for three popular techniques, namely, multiplicative updates~\cite{Lee2000}, active set method for nonnegative least squares~\cite{Kim2008}, and ADMM~\cite{Xu2012}. For SNMF, we presented a general technique that can be used with any algorithm. In all cases, we showed that the resulting compressed techniques are faster than their uncompressed variants and, at the same time, do not introduce significant errors in the final result.

There are in the literature very efficient SNMF algorithms for tall-and-skinny matrices. Interestingly, the use of structured random projections allows to compute SNMF for arbitrarily large matrices, granting access to very efficient computations in the general setting.

As a byproduct, we also propose an algorithmic solution for computing structured random projections of extremely large matrices (i.e., matrices so large that even after compression they do not fit in main memory). This is useful as a general tool for computing many different matrix decompositions, such as the singular value decomposition, for example.

We are currently investigating the problem of replacing the Frobenius norm with an $\ell_p$ norm in our compressed variants of NMF and SNMF. In this setting, the fast Cauchy transform~\cite{Clarkson2013} is a suitable alternative to structured random projections. Compression consists of sampling and rescaling rows of $\mat{A}$, thus identifying the so-called \emph{coreset} of the problem. This formulation is of particular interest for network analysis, where we need to deal with sparse structures.

\section*{Acknowledgments}
The authors would like to thank Mauricio Delbracio for many useful scientific discussions and Matthew Rocklin for his help and technical support with the dask and into libraries.

\appendix

\subsection{QR decompositions for tall-and-skinny matrices}
\label{sec:tsqr}

The direct TSQR algorithm uses a simple but highly efficient approach for computing that QR decomposition of a tall and skinny matrix. Let $\mat{A}$ be the $m \times n$ matrix to decompose ($m \gg n$).
The direct TSQR algorithm starts by splitting $\mat{A}$ into a stack of $b$ blocks 
\begin{equation}
	\mat{A}
	=
	\begin{bmatrix}
		\mat{A}_{\set{K}_{1:}} \\
		\vdots \\
		\mat{A}_{\set{K}_{b:}}
	\end{bmatrix} ,
\end{equation}
where $\set{K}_{i:}$ denotes the set of rows selected in the $i$th block.
Each block $\mat{A}_{\set{K}_{i:}}$ is factorized into its components $\mat{Q}_{\set{K}_{1:}}^{(1)}$, $\mat{R}_{\set{K}_{1:}}$ using any standard QR decomposition algorithm. This can be written in matrix form as
\begin{equation}
	\underbrace{
	\begin{bmatrix}
		\mat{A}_{\set{K}_{1:}} \\
		\vdots \\
		\mat{A}_{\set{K}_{b:}}
	\end{bmatrix}	}_{m \times n}
	=
	\underbrace{
		\begin{bmatrix}
			\mat{Q}_{\set{K}_{1:}}^{(1)} \\
			& \ddots \\
			&& \mat{Q}_{\set{K}_{b:}}^{(1)}
		\end{bmatrix}
	}_{m \times bn}
	\underbrace{
		\begin{bmatrix}
			\mat{R}_{\set{K}_{1:}} \\
			\vdots \\
			\mat{R}_{\set{K}_{b:}}
		\end{bmatrix}
	}_{bn \times n} .
\end{equation}
The second step is to gather the matrix composed by vertically stacking the factors $\mat{R}_{\set{K}_{i:}}$ and computing an additional QR decomposition, i.e.,
\begin{equation}
	\underbrace{
		\begin{bmatrix}
			\mat{R}_{\set{K}_{1:}} \\
			\vdots \\
			\mat{R}_{\set{K}_{b:}}
		\end{bmatrix}
	}_{bn \times n}
	=
	\underbrace{
		\begin{bmatrix}
			\mat{Q}_{\set{K}_{1:}}^{(2)} \\
			\vdots \\
			\mat{Q}_{\set{K}_{b:}}^{(2)}
		\end{bmatrix}
	}_{bn \times n}
	\underbrace{
		\mat{R}
	}_{n \times n} .
	\label{eq:tsqr_centralizedQR}
\end{equation}
This is the only centralized step in TSQR.
We then multiply the intermediate Q factors to get the matrix 
\begin{equation}
	\mat{Q}
	= 
	\underbrace{
		\begin{bmatrix}
			\mat{Q}_{\set{K}_{1:}}^{(1)} \\
			& \ddots \\
			&& \mat{Q}_{\set{K}_{b:}}^{(1)}
		\end{bmatrix}
	}_{m \times bn}
	\underbrace{
		\begin{bmatrix}
			\mat{Q}_{\set{K}_{1:}}^{(2)} \\
			\vdots \\
			\mat{Q}_{\set{K}_{b:}}^{(2)}
		\end{bmatrix}
	}_{bn \times n}
	=
	\begin{bmatrix}
		\mat{Q}_{\set{K}_{1:}}^{(1)} \mat{Q}_{\set{K}_{1:}}^{(2)} \\
		\vdots \\
		\mat{Q}_{\set{K}_{b:}}^{(1)} \mat{Q}_{\set{K}_{b:}}^{(2)}
	\end{bmatrix} .
\end{equation}
Finally note that $\mat{A} = \mat{Q} \mat{R}$, where $\mat{Q}$ is an orthonormal matrix (obtained from the multiplication of two orthonormal matrices) and $\mat{R}$ is by algorithmic design, upper triangular. Thus, these matrices form a QR decomposition of $\mat{A}$.

\subsubsection{TSQR for structured random compression}

When using TSQR for compressing a matrix $\mat{A}$, \cref{algo:compression}, the input matrix to decompose is
\begin{equation}
	\mat{B} = \left( \mat{A} \transpose{\mat{A}} \right)^w \mat{A} \mat{\Omega} ,
\end{equation}
where $w \in \Natural$. Let us assume, for simplicity, that $w = 0$. The input of TSQR is not the matrix $\mat{B}$ as a whole, but blocks extracted from it. We can thus avoid storing the entire matrix $\mat{B}$ in main memory, and compute its blocks as needed, i.e.,
\begin{equation}
	\mat{B}_{\set{K}_{i:}} = \mat{A}_{\set{K}_{i:}} \mat{\Omega} .
\end{equation}
A similar (but more complex) indexing holds for $w > 0$.

\subsection{An ADMM algorithm for solving \cref{eq:nmfEquiv_extended_compressed}}
\label{sec:nmf_admm_compressed}

We consider the augmented Lagrangian of \cref{eq:nmfEquiv_extended_compressed},
\begin{multline}
\mathscr{L} \left( \widetilde{\mat{X}}, \widetilde{\mat{Y}}, \mat{U}, \mat{V}, \mat{\Lambda}, \mat{\Phi} \right)
= \norm{ \widetilde{\mat{A}} - \widetilde{\mat{X}} \widetilde{\mat{Y}} }{F}^2  + \\
+ \mat{\Lambda} \bullet \left( \mat{L} \widetilde{\mat{X}} - \mat{U} \right) + \tfrac{\lambda}{2} \norm{ \mat{L} \widetilde{\mat{X}} - \mat{U}}{F}^2 +  \\
+ \mat{\Phi} \bullet \left( \widetilde{\mat{Y}} \mat{R} - \mat{V} \right) + \tfrac{\phi}{2} \norm{ \widetilde{\mat{Y}} \mat{R} - \mat{V} }{F}^2  ,
\end{multline}
where $\Lambda \in \Real^{m \times r}, \mat{\Phi}  \in \Real^{r \times n}$ are Lagrange multipliers, $\lambda, \phi \in \Real^+$ are penalty parameters, and $\mat{B} \bullet \mat{C} = \sum_{i,j} (\mat{B})_{ij} (\mat{C})_{ij}$ for matrices $\mat{B}, \mat{C}$ of the same size.

We use the Alternating Direction Method of Multipliers (ADMM) for solving \cref{eq:nmfEquiv_extended_compressed}. The algorithm works in a coordinate descent fashion, successively minimizing $\mathscr{L}$ with respect to $\widetilde{\mat{X}}, \widetilde{\mat{Y}}, \mat{U}, \mat{V}$, one at a time while fixing the others at their most recent values, i.e.,
\begin{subequations}
	\begin{align}
	\widetilde{\mat{X}}_{k+1} &= \argmin_{\widetilde{\mat{X}}} \mathscr{L} \left( \widetilde{\mat{X}}, \widetilde{\mat{Y}}_{k}, \mat{U}_{k}, \mat{V}_{k}, \mat{\Lambda}_{k}, \mat{\Phi}_{k} \right) , \\
	\widetilde{\mat{Y}}_{k+1} &= \argmin_{\widetilde{\mat{Y}}} \mathscr{L} \left( \widetilde{\mat{X}}_{k+1}, \widetilde{\mat{Y}}, \mat{U}_{k}, \mat{V}_{k}, \mat{\Lambda}_{k}, \mat{\Phi}_{k} \right) , \\
	\mat{U}_{k+1} &= \argmin_{\mat{U} \geq 0} \mathscr{L} \left( \widetilde{\mat{X}}_{k+1}, \widetilde{\mat{Y}}_{k+1}, \mat{U}, \mat{V}_{k}, \mat{\Lambda}_{k}, \mat{\Phi}_{k} \right) , \\
	\mat{V}_{k+1} &= \argmin_{\mat{V} \geq 0} \mathscr{L} \left( \widetilde{\mat{X}}_{k+1}, \widetilde{\mat{Y}}_{k+1}, \mat{U}_{k+1}, \mat{V}, \mat{\Lambda}_{k}, \mat{\Phi}_{k} \right) ,
	\end{align}
	\label[algorithm]{eq:admm_nmf_compressed}
\end{subequations}
and then updating the multipliers $\mat{\Lambda}, \mat{\Phi}$.
Each of these steps can be written in closed form and define our algorithm, see \cref{algo:compressedNMF_admm}. In practice, we set $\alpha, \beta, \gamma, \xi$ to 1.

\begin{figure}[t]
	\removelatexerror
	\begin{algorithm2e}[H]	
		\SetInd{0.5em}{0.5em}

		\begin{small}
			
			\SetKwInOut{Input}{input}\SetKwInOut{Output}{output}
			
			\Input{a matrix $\mat{A} \in \Real^{m \times n}$, a target rank $r  \in \Natural^+$, an oversampling parameter $r_\textsc{ov} \in \Natural^+$ ($r + r_\textsc{ov} \leq \min \{ m, n\}$), an exponent $w \in \Natural$.}
			\Output{nonnegative matrices $\mat{U}_{k} \in \Real^{m \times r}, \mat{V}_{k} \in \Real^{r \times n}$.}
			
			Compute compression matrices $\mat{L} \in \Real^{m \times (r + r_\textsc{ov})}$, $\mat{R} \in \Real^{(r + r_\textsc{ov}) \times n}$\;
			
			$k \gets 1$\;
			Initialize $\mat{U}_{k}, \mat{V}_{k}$\;
			
			$\displaystyle \widetilde{\mat{A}} \gets \transpose{\mat{L}} \mat{A} \transpose{\mat{R}}$; \quad
			$\displaystyle \widetilde{\mat{Y}} \gets \mat{V}_{k} \transpose{\mat{R}}$\;

			$\mat{\Lambda}_{k} \gets \mat{0}$; \quad
			$\mat{\Phi}_{k} \gets \mat{0}$\;
			$\mat{I} \gets$ the $r \times r$ identity matrix
			
			\Repeat{convergence}{
				$\widetilde{\mat{X}}_{k+1} \gets ( \widetilde{\mat{A}} \transpose{\widetilde{\mat{Y}}_k} + \lambda \transpose{\mat{L}} \mat{U}_k - \transpose{\mat{L}} \mat{\Lambda}_k ) ( \widetilde{\mat{Y}}_k \transpose{\widetilde{\mat{Y}}_k} + \lambda \mat{I} )^{-1}$\;
				$\widetilde{\mat{Y}}_{k+1} \gets ( \transpose{\widetilde{\mat{X}}_{k+1}} \widetilde{\mat{X}}_{k+1} + \phi \mat{I} )^{-1} ( \transpose{ \widetilde{\mat{X}}_{k+1}} \widetilde{\mat{A}} + \phi \mat{V}_k \transpose{\mat{R}} - \mat{\Phi}_k \transpose{\mat{R}} )$\;
				\tcp{$(\mathscr{P}_+ (\mat{B}))_{ij} = \max\left\{ (\mat{B})_{ij}, 0 \right\}$}
				$\mat{U}_{k+1} \gets \mathscr{P}_+ ( \mat{L} \widetilde{\mat{X}}_{k+1} + \lambda^{-1} \mat{\Lambda}_k )$\;
				$\mat{V}_{k+1} \gets \mathscr{P}_+ ( \widetilde{\mat{Y}}_{k+1} \mat{R} + \phi^{-1} \mat{\Phi}_k )$\;
				
				$\mat{\Lambda}_{k+1} \gets \mat{\Lambda}_{k} + \xi \lambda ( \mat{L} \widetilde{\mat{X}}_{k+1} - \mat{U}_{k+1} )$\;
				$\mat{\Phi}_{k+1} \gets \mat{\Phi}_{k} + \xi \phi ( \widetilde{\mat{Y}}_{k+1} \mat{R} - \mat{V}_{k+1} )$\;
				$k \gets k + 1$\;
			}
			
		\end{small}
		
	\end{algorithm2e}
	
	\caption{ADMM algorithm for NMF with structured random compression.}
	\label{algo:compressedNMF_admm}
\end{figure}

We now provide a preliminary convergence property of the proposed ADMM algorithm. Our analysis follows closely the one in~\cite[Section 2.3]{Xu2012}.

To simplify notation, we consolidate all the variables as
\begin{equation*}
Z = \left( \widetilde{\mat{X}}, \widetilde{\mat{Y}}, \mat{U}, \mat{V}, \mat{\Lambda}, \mat{\Phi} \right) .
\end{equation*}
A point $Z$ is a Karush-Kuhn-Tucker (KKT) condition of \cref{eq:nmfEquiv_extended_compressed} if
\begin{subequations}
	\begin{align}
	\left( \widetilde{\mat{X}} \widetilde{\mat{Y}} - \widetilde{\mat{A}} \right) \transpose{\widetilde{\mat{Y}}} + \mat{\Lambda} &= 0 , \label{eq:nmfKKT_a} \\
	\transpose{\widetilde{\mat{X}}} \left( \widetilde{\mat{X}} \widetilde{\mat{Y}} - \widetilde{\mat{A}} \right) + \mat{\Phi} &= 0 , \label{eq:nmfKKT_b} \\
	\mat{L} \widetilde{\mat{X}} - \mat{U} &= 0 , \label{eq:nmfKKT_c} \\
	\widetilde{\mat{Y}} \mat{R} - \mat{V} &= 0 , \label{eq:nmfKKT_d} \\
	\mat{\Lambda} \leq 0 \leq \mat{U},\ \mat{\Lambda} \circ \mat{U} &= 0 , \label{eq:nmfKKT_e} \\
	\mat{\Phi} \leq 0 \leq \mat{V},\ \mat{\Phi} \circ \mat{V} &= 0 , \label{eq:nmfKKT_f}
	\end{align}
\end{subequations}
where $\circ$ denotes the Hadamard (entrywise) matrix product.

\begin{proposition}
	Let $\{ Z_k \}_{k=1}^{\infty}$ be a sequence generated by the algorithm in \cref{algo:compressedNMF_admm} that satisfies the condition
	\begin{equation}
	\lim_{k \rightarrow \infty} \left( Z_{k+1} - Z_{k} \right) = 0.
	\label[assumption]{eq:nmf_convergenceCondition}
	\end{equation}
	Then any accumulation point of $\{ Z_k \}_{k=1}^{\infty}$ is a KKT point of \cref{eq:nmfEquiv_extended_compressed}.
	
	\label{thm:convergence}
\end{proposition}

\begin{proof}
	From \Cref{eq:nmf_convergenceCondition}, we have
	\begin{subequations}
		\begin{align}
		\widetilde{\mat{X}}_{k+1} - \widetilde{\mat{X}}_k &\rightarrow 0 ,\\
		\widetilde{\mat{Y}}_{k+1} - \widetilde{\mat{Y}}_k &\rightarrow 0 ,\\
		\mat{\Lambda}_{k+1} - \mat{\Lambda}_{k} &\rightarrow 0 ,\\
		\mat{\Phi}_{k+1} - \mat{\Phi}_{k} &\rightarrow 0 ,\\
		\mat{U}_{k+1} - \mat{U}_{k} &\rightarrow 0 ,\\
		\mat{V}_{k+1} - \mat{V}_{k} &\rightarrow 0 .
		\end{align}
	\end{subequations}
	Plugging these subtractions in the variable updates in \cref{algo:compressedNMF_admm}, we get
	\begin{subequations}
		\begin{align}
		\left( \widetilde{\mat{A}} - \widetilde{\mat{X}}_k \widetilde{\mat{Y}}_k \right) \transpose{\widetilde{\mat{Y}}_k} - \mat{L} \mat{\Lambda}_k  &\rightarrow 0 , \label{eq:nmf_limit_a} \\
		\transpose{ \widetilde{\mat{X}}_{k+1}} \left( \widetilde{\mat{A}} - \widetilde{\mat{X}}_{k+1} \widetilde{\mat{Y}}_{k} \right) - \mat{\Phi}_k \mat{R} &\rightarrow 0 , \label{eq:nmf_limit_b} \\
		\mat{L} \widetilde{\mat{X}}_{k+1} - \mat{U}_{k+1} &\rightarrow 0 , \label{eq:nmf_limit_c} \\
		\widetilde{\mat{Y}}_{k+1} \mat{R} - \mat{V}_{k+1} &\rightarrow 0 , \label{eq:nmf_limit_d} \\
		\mathscr{P}_+ \left( \mat{L} \widetilde{\mat{X}}_{k+1} + \lambda^{-1} \mat{\Lambda}_k \right)  - \mat{U}_{k} &\rightarrow 0 , \label{eq:nmf_limit_e} \\
		\mathscr{P}_+ \left( \widetilde{\mat{Y}}_{k+1} \mat{R} + \phi^{-1} \mat{\Phi}_k \right)  - \mat{V}_{k} &\rightarrow 0 \label{eq:nmf_limit_f} .
		\end{align}
	\end{subequations}
	Notice that the terms $\lambda \left( \transpose{\mat{L}} \mat{U}_k - \widetilde{\mat{X}}_k \right)$ and $\phi \left( \mat{V}_k \transpose{\mat{R}} - \widetilde{\mat{Y}}_{k} \right)$ have been eliminated from \cref{eq:nmf_limit_a,eq:nmf_limit_b} by invoking \cref{eq:nmf_limit_c,eq:nmf_limit_d}, respectively.
	\Cref{eq:nmfKKT_a,eq:nmfKKT_b,eq:nmfKKT_c,eq:nmfKKT_d} are clearly satisfied by \cref{eq:nmf_limit_a,eq:nmf_limit_b,eq:nmf_limit_c,eq:nmf_limit_d} at any limit point
	\begin{equation*}
	Z_\infty = \left( \widetilde{\mat{X}}_\infty, \widetilde{\mat{Y}}_\infty, \mat{U}_\infty, \mat{V}_\infty, \mat{\Lambda}_\infty, \mat{\Phi}_\infty \right) .
	\end{equation*}
	We are then left to prove that \cref{eq:nmfKKT_e,eq:nmfKKT_f} hold. \Cref{eq:admm_nmf_compressed} guarantees the non-negativity of $\mat{U}_\infty, \mat{V}_\infty$.
	Let us focus on \cref{eq:nmfKKT_e} first.
	\Cref{eq:nmf_limit_e}, when combined with \cref{eq:nmf_limit_c}, yields
	\begin{equation}
	\mat{U}_\infty = \mathscr{P}_+ \left( \mat{U}_\infty + \lambda^{-1} \mat{\Lambda}_\infty \right) ,
	\end{equation}
	If $\left( \mat{U}_\infty \right)_{ij} = 0$, we get $\left( \mathscr{P}_+ \left( \lambda^{-1} \mat{\Lambda}_\infty \right) \right)_{ij} = 0$ and then $\left( \mat{\Lambda}_\infty \right)_{ij} \leq 0$.
	If $\left( \mat{U}_\infty \right)_{ij} > 0$, we get $\left( \mat{U}_\infty \right)_{ij} = \mathscr{P}_+ \left( \left( \mat{U}_\infty \right)_{ij} \right)$ and $\left( \mat{\Lambda}_\infty \right)_{ij} = 0$. 
	From this, we obtain that \cref{eq:nmfKKT_e} holds.
	An identical argument applies for \cref{eq:nmfKKT_f,eq:nmf_limit_f}.
	
	With this, we have proven that any accumulation point of $\{ Z_k \}_{k=1}^{\infty}$ is a KKT point of \cref{eq:nmfEquiv}.
	From the equivalence of \cref{eq:nmf,eq:nmfEquiv}, any accumulation point of $\left\{ (\mat{X}_k, \mat{Y}_k) \right\}_{k=1}^{\infty}$ is a KKT point of \cref{eq:nmf}.
\end{proof}

\begin{corollary}
	Whenever $\{ Z_k \}_{k=1}^{\infty}$ converges, it converges to a KKT point of \cref{eq:nmfEquiv}.
\end{corollary}

Ideally, we would like to guarantee that \cref{eq:admm_nmf_compressed} will always converge to a KKT point of \cref{eq:nmfEquiv_extended_compressed}.
The above simple result is an initial step in this direction, providing some assurance on the behavior of \cref{eq:admm_nmf_compressed}.

\bibliographystyle{IEEEtran}
\bibliography{biblio}

\end{document}